\documentclass{article}
\usepackage[final]{neurips_2025}

\usepackage{float}

\usepackage[utf8]{inputenc} 
\usepackage[T1]{fontenc}    
\usepackage{url}            
\usepackage{booktabs}       
\usepackage{amsfonts}       
\usepackage{tabularx}
\usepackage[dvipsnames]{xcolor}
\usepackage{nicefrac}       
\usepackage{microtype}      
\usepackage{lipsum}
\usepackage{graphicx}    
\usepackage{colortbl}
\usepackage[most]{tcolorbox}
\usepackage{fontawesome5}

\definecolor{DarkBlue}{rgb}{0.0,0.0,0.55}
\definecolor{lightgray}{gray}{0.95}
\definecolor{lightgreen}{RGB}{220,255,220}
\definecolor{lightred}{RGB}{255,220,220}
\definecolor{lightblue}{RGB}{220,240,255}
\definecolor{lightorange}{RGB}{255,235,200}

\newtcolorbox{problembox}[1][]{%
  enhanced,
  colback=gray!5!white,
  colframe=cyan!40!black,
  fonttitle=\bfseries,
  fontupper=\footnotesize,
  boxrule=0.5pt,
  width=0.9\textwidth,
  title=Problem,
  after upper={%
    \ifx\relax#1\relax
    \else
      \par\medskip\hfill\textit{#1}%
    \fi
  },
}

\tcbuselibrary{skins}

\usepackage{caption}
\usepackage{subcaption} 

\usepackage{pifont} 

\usepackage{tikz}

\newcommand{\crossmark}{%
  \tikz[scale=0.4, baseline=0.6ex] 
    \draw[red, line width=1.8pt] (0,0) -- (1,1) (0,1) -- (1,0);
}

\newcommand{\checkmarkicon}{%
  \tikz[scale=0.4, baseline=0.6ex]
    \draw[green!60!black, line width=1.8pt]
      (0,0.5) -- (0.35,0) -- (1,1);
}
\newcommand{\fail}{\hfill \textcolor{red}{\crossmark GPT-5 thinking failed~}}

\newcommand{\pass}{\hfill	\textcolor{green!50!black}{\checkmarkicon GPT-5 thinking passed}}

\newcommand{\mhfail}{\hfill \textcolor{red}{\crossmark o4-mini-high  failed}}

\definecolor{DarkBlue}{rgb}{0.0,0.0,0.55}

\usepackage[
    colorlinks=true,
    linkcolor=DarkBlue,
    citecolor=DarkBlue,
    urlcolor=DarkBlue
]{hyperref}

\definecolor{low}{RGB}{255,200,200}   
\definecolor{mid}{RGB}{255,255,200}   
\definecolor{high}{RGB}{200,255,200}  

\newcommand{\scorecell}[1]{%
  \ifnum#1<50\relax
    \cellcolor{mid!\the\numexpr 2*#1\relax!low} #1%
  \else
    \cellcolor{high!\the\numexpr 2*(#1-50)\relax!mid} #1%
  \fi
}

\newcounter{problem}[figure]  
\renewcommand{\theproblem}{\arabic{problem}}

\newcommand{\red}[1]{\textcolor{red}{#1}}
\newcommand{\green}[1]{\textcolor{ForestGreen}{#1}}
\usepackage{geometry}
\usepackage{enumitem}
\usepackage{booktabs}
\usepackage{amsmath}
\usepackage{tabularx}
\usepackage{makecell}
\usepackage{amssymb,amsthm,amsmath}
\usepackage{longtable}
\usepackage{cleveref}

\newtheorem{definition}{Definition}

\newenvironment{GPT-5-thinking-response}[1][\textbf{Response of GPT-5-thinking}]
  {\begin{proof}[#1]}
  {\end{proof}}

\newenvironment{solution}[1][\textbf{Standard Solution}]
  {\begin{proof}[#1]}
  {\end{proof}}

\usepackage{tcolorbox}

\newtcolorbox{rubricbox}{
  colback=gray!5!white,
  colframe=gray!60!black,
  fonttitle=\bfseries, title=Rubric,
  fontupper=\footnotesize,
  boxrule=0.5pt,
  width=0.9\textwidth
}

\newtcolorbox{scorebox}{
  colback=SeaGreen!5!white,
  colframe=SeaGreen!60!black,
  fonttitle=\bfseries, title=Score,
  fontupper=\footnotesize,
  boxrule=0.5pt,
  width=0.9\textwidth
}

\newtcolorbox{evalbox}{
  colback=OrangeRed!5!white,
  colframe=OrangeRed!70!black,
  fonttitle=\bfseries, title=Evaluation,
  fontupper=\footnotesize,
  boxrule=0.5pt,
  width=0.9\textwidth
}

\newcommand{\ProblemSection}[1]{%
  \paragraph{#1}\mbox{}\\[-1em]
}
\makeatletter
\@addtoreset{paragraph}{subsection}
\makeatother

\newcommand{\R}{\mathbb{R}}
\newcommand{\D}{\mathbb{D}}
\newcommand{\T}{\mathbb{T}}
\newcommand{\e}{\mathrm{e}}
\newcommand{\dd}{\mathrm{d}}

\newcommand{\Law}{\mathbf{P}}

\newcommand{\Vrad}{V^{\mathrm{rad}}}
\newcommand{\hwt}{h}
\newcommand{\kap}{\kappa}

\newtheorem{claim}{Claim}

\title{GAUSS: Benchmarking Structured Mathematical Skills for Large Language Models}
\author{%
  Yue Zhang$^{1}$\thanks{Equal Contribution.}\\
  \And
  Jiaxin Zhang$^{1,2}$\footnotemark[1] \\
  \And
  Qiuyu Ren$^3$ \\
  \And
  Tahsin Saffat$^3$ \\
\And
  Xiaoxuan Liu$^3$ \\
  \And
 Zitong Yang$^4$   \\
  \And
  Banghua Zhu$^{5,6}$ \\
  \And
  Yi Ma$^{3,7}$ \\
  \AND
  $^1$\textnormal{Hyperbolic~Labs} \\
  $^2$California Institute of Technology \\
  $^3$University of California, Berkeley \\
  $^4$Stanford University \\
  $^5$Nvidia \\
  $^6$University of Washington \\
  $^7$University of Hong Kong
}

\makeatletter
\def\@trackname{}  
\makeatother

\begin{document}

\maketitle

\begin{abstract}
We introduce \textbf{GAUSS} (\textbf{G}eneral \textbf{A}ssessment of \textbf{U}nderlying \textbf{S}tructured \textbf{S}kills in Mathematics),  a benchmark that evaluates LLMs' mathematical abilities across twelve core skill dimensions, grouped into three domains: knowledge and understanding, problem solving and communication, and meta-skills and creativity. By categorizing problems according to cognitive skills and designing tasks that isolate specific abilities, GAUSS constructs comprehensive, fine-grained, and interpretable profiles of models' mathematical abilities. These profiles faithfully represent their underlying mathematical intelligence. To exemplify how to use the  \textsc{GAUSS} benchmark, we have derived the skill profile of \textsc{GPT-5-thinking}, revealing its strengths and weaknesses as well as its differences relative to \textsc{o4-mini-high}, thereby underscoring the value of multidimensional, skill-based evaluation.

\end{abstract}

\newpage
\tableofcontents

\newpage
\section{Introduction}

Large language models (LLMs) have achieved impressive performance across a range of mathematical tasks, from elementary arithmetic to high school olympiad-level problem solving~\citep{trinh2024solving}, to undergraduate-level challenges~\citep{putnam2024}, and even to open research problems~\citep{frontiermath2024}. However, despite this rapid progress, current benchmarks face several critical limitations—including low skill resolution, saturation at lower difficulty levels, and potential data contamination:

\begin{itemize}
\item \textbf{Low skill resolution.} Most benchmarks report only top-level accuracy, without disentangling the underlying skills required. For instance, solving an algebra problem may rely on applying known theorems, symbolic computation, proof writing, and intuitive judgment---but current benchmarks do not identify which of these a model succeeded or failed on.

\item \textbf{Saturation of existing benchmarks.} Datasets like GSM8K~\citep{cobbe2021training} and MATH~\citep{hendrycks2021measuring} focus primarily on elementary and high school mathematics. While these have served as important early-stage benchmarks, state-of-the-art models now routinely achieve near-perfect accuracy on them, limiting their utility in assessing more advanced mathematical thinking or research-level reasoning.

\item \textbf{Benchmark contamination.} Many benchmark problems are scraped from online sources and may have leaked into training corpora, making performance on them less reliable as a measure of genuine reasoning ability~\citep{deng2023investigating,xu2024benchmark}. 
\end{itemize}
These limitations hinder meaningful diagnosis of a model's true capabilities, making it very difficult, if not impossible, to attribute its successes or failures to certain capabilities or lack thereof. Hence the evaluations cannot be scientifically rigorous or conclusive regarding whether the model truly understands certain abstract concepts or has grasped certain mathematical skills. This impedes systematic development of new methods to improve or revise current practices.

Several recent efforts have begun to push beyond the limitations above. In particular, \textit{FrontierMath}~\citep{frontiermath2024} introduces a benchmark with all-level problems and proof-based reasoning tasks, helping to address the saturation and contamination of existing benchmarks. However, it primarily evaluates task-level correctness across domains and does not provide structured annotations of the underlying cognitive skills or reasoning behaviors. As such, it does not directly resolve the challenges of low skill resolution.

At a more fundamental level, general intelligence can be viewed as the capacity to extract low-dimensional structures from complex data and to exploit such structures for prediction, reasoning, decision making, and creative problem solving \citep{ldrdd2025}. There are reasons to believe that mathematics requires a more advanced form of intelligence though, formally known as (part of) {\em the}  {Artificial Intelligence}, that is unique to educated and enlightened humans. It goes way beyond identifying and exploiting low-dimensional structures (or correlations) in data distributions. It is not only about solving certain problems correctly based on exploiting correlation with past practices or experiences, but also demonstrating grasp of a broad spectrum of capabilities—from understanding abstract concepts and applying the associated knowledge, to constructing and deducing proofs, leveraging intuition, transferring across domains, and even posing novel problems and conjectures. Traditional benchmarks, by focusing narrowly on correctness, provide only a coarse snapshot of performance and thus fail to capture this multifaceted nature.

To address these challenges, we introduce GAUSS (\textbf{G}eneral \textbf{A}ssessment of \textbf{U}nderlying \textbf{S}tructured \textbf{S}kills in Mathematics), a benchmark that evaluates mathematical proficiency through a fine-grained lens of mathematical skills and behaviors—rather than coarse topical labels or problem types.

First, to address the problem of \emph{low skill resolution}, GAUSS organizes tasks by the specific mathematical abilities they require, not merely by topic. Each problem is annotated with cognitive skill tags—such as theorem recall, symbolic manipulation, strategic problem-solving, or intuitive judgment—enabling detailed diagnostics of a model’s strengths and weaknesses. This structure reveals whether performance arises from superficial pattern matching or genuine understanding.

Second, to overcome the \emph{saturation of existing benchmarks}, GAUSS targets high-level tasks beyond elementary and high school curricula. Most of its problems are drawn from advanced undergraduate and graduate topics, as well as elite mathematical competitions (e.g., the IMO, national Olympiads, and collegiate contests). These challenging problems push the evaluation frontier toward higher-order reasoning and research-level mathematical thinking.

Third, to mitigate \emph{benchmark contamination}, GAUSS draws on a broad range of sources unlikely to have appeared in model training data. These include lesser-known competition problems, curated textbook materials, online forum problems, and selected research-level tasks, among others. By integrating such diverse sources while minimizing overlap with widely available public corpora, GAUSS enables a more faithful assessment of genuine reasoning ability rather than surface-level memorization.

Overall, GAUSS shifts the focus from topic-based correctness to skill-based competence, providing a more interpretable, robust, and forward-looking benchmark for evaluating mathematical abilities in large language models.

We organize mathematical proficiency into three overarching domains encompassing twelve core cognitive skills:

\begin{enumerate}
  \item \textbf{Mathematical Knowledge and Understanding.}  
  This domain includes knowledge recall, conceptual understanding, and application of known results.
  \begin{itemize}
    \item \textbf{Memorizing mathematical knowledge:} Recall of definitions, theorems, lemmas, and subject taxonomy across algebra, analysis, combinatorics, etc.
    \item \textbf{Understanding of knowledge and theories:} Know what classical theorems state, why they are true, and how to apply them effectively.
    \item \textbf{Computational and analytical skills:} Perform accurate symbolic computations and algebraic transformations.
  \end{itemize}

  \item \textbf{Problem Solving and Communication.}  
  This domain focuses on tackling problems through reasoning and effectively conveying mathematical ideas.
  \begin{itemize}
    \item \textbf{Problem-solving framework:} Understand problems, explore examples, and apply key strategies like contradiction, decomposition, and analogy.
    \item \textbf{Logical thinking and reasoning:} Construct valid proofs and critically evaluate existing arguments.
    \item \textbf{Writing and presentation:} Express solutions clearly and tailor explanations for different audiences.
  \end{itemize}

\item \textbf{Learning, Meta Skills, and Creativity.}  
This domain reflects the learner’s ability to acquire new knowledge independently, adapt across contexts, and engage in abstract, cross-disciplinary, and creative reasoning.

  \begin{itemize}
    \item \textbf{Learning new knowledge:} Read and absorb advanced content, and independently acquire unfamiliar techniques.
    \item \textbf{Intuition:} Use geometric and numerical insights, plausibility checks, and aesthetic judgment to guide and validate solutions.
    \item \textbf{Meta skills:} Think abstractly and transfer tools across domains (e.g., probabilistic method in combinatorics).
    \item \textbf{Mathematical modeling:} Translate real-world or theoretical scenarios into formal mathematical models.
    \item \textbf{Generalization:} Abstract patterns and unify results across different contexts.
    \item \textbf{Creativity:} Pose new problems, invent novel methods, and apply lateral thinking to discover unconventional solutions.
  \end{itemize}
\end{enumerate}

GAUSS is designed to be \textbf{diagnostic}, \textbf{extensible}, and \textbf{robust to contamination}. By evaluating models through a skill-oriented lens, GAUSS complements existing benchmarks and provides a new foundation for tracking genuine mathematical progress in LLMs---especially as they move beyond rote learning and toward autonomous discovery and research.

To demonstrate the value of the GAUSS benchmark, we apply it to GPT-5-thinking and yield a fine-grained profile of the model's mathematical skills. The model shows \emph{notable strengths} in \textbf{recalling the taxonomy of mathematical subjects}, in \textbf{evaluating arguments}, in \textbf{posing novel problems}, in \textbf{writing proofs}, in \textbf{reading advanced materials} and in \textbf{demonstrating strong intuition} through abilities such as plausibility checks. These abilities reflect reliable performance in knowledge organization, error detection, intuitive reasoning, problem formulation, and comprehension of technical content.

In contrast, the model's performance is \emph{limited} in core areas such as \textbf{theorem understanding}, \textbf{computational and analytical skills} and \textbf{problem-solving strategies}, with further weaknesses in \textbf{generalization}. These results suggest persistent difficulties in applying known results, carrying out symbolic computations, applying systematic strategies and developing unifying frameworks.

Relative to o4-mini-high, GPT-5-thinking achieves measurable improvements in \textbf{subject taxonomy}, \textbf{argument evaluation}, \textbf{plausibility check}, and \textbf{reading advanced materials}, while most other categories show only marginal differences.

At present, GAUSS remains limited in scope: the problem pool is still relatively small, subject coverage is incomplete, and open-ended tasks such as proofs and creative reasoning are difficult to grade. To address these limitations, we plan to broaden the benchmark through community-sourced problems across mathematical domains, design automatic grading pipelines, and develop an open, extensible platform for integrating new problems and evaluation modules. In this way, we aim to construct more precise skill profiles, revealing where models truly excel and where they continue to struggle.

\newpage
\section{Math Skills Breakdown}

\subsection{Memory of Math Knowledge}

This category focuses on the ability to recall essential mathematical knowledge, including definitions, theorems, and lemmas. It forms the foundation of all higher-level reasoning and problem-solving. Two key components define this category: a subject-based taxonomy of mathematical domains and key statements within each domain.

\subsubsection{Taxonomy of subjects}

  Mathematical knowledge is traditionally organized into subjects such as:
  \begin{enumerate}
  \item \textit{Algebra};
  \item \textit{Analysis};
  \item \textit{Combinatorics};
  \item \textit{Geometry and Topology};
  \item \textit{Probability and Statistics}.
  \end{enumerate}
  This subject taxonomy provides a conceptual “map” for indexing, retrieving, and organizing mathematical knowledge.
  
\textbf{Design Principles:} Choose concepts that are familiar to domain experts but are not typically found in standard textbooks.
  
\textbf{Sample Problems:} \setcounter{problem}{0}  

\begin{figure}[H]
    \centering
    
    \begin{minipage}[t]{0.8\textwidth} \stepcounter{problem}
        \begin{tcolorbox}[
        title=\textbf{\footnotesize Problem \theproblem: Floer homology},
            colframe=cyan!40!black,
            boxrule=0.5pt,
            colback=white,
            top=0pt
        ]
        \vspace{0.1cm}
        \footnotesize  Floer homology is a powerful tool for studying symplectic geometry and low-dimensional topology. 
        \begin{enumerate}
            \item Please name three distinct Floer homology theories that play important roles in low-dimensional topology.
            \item For each theory, provide an example of a significant topological application. Your applications should not mention explicitly the Floer homology theory in the statement.
        \end{enumerate}
         
\mhfail    \pass
\par
        \end{tcolorbox}
    \end{minipage}
        \vspace{0.2cm}

    \begin{minipage}[t]{0.8\textwidth} \stepcounter{problem}
        \begin{tcolorbox}[  title=\textbf{\footnotesize Problem \theproblem: Liouville Quantum Gravity},
            colframe=cyan!40!black,
            boxrule=0.5pt,
            colback=white,
            top=0pt
        ]

\vspace{0.1cm}
        \footnotesize   Liouville quantum gravity (LQG) is a central topic in modern probability theory, please list at least 3 major subfields, and for each:
  \begin{enumerate}
    \item Name a central concept or structure,
    \item State one foundational result from that area,
  \end{enumerate}
  \pass
  \par 

        \end{tcolorbox}
    \end{minipage} 
\end{figure}

 \subsubsection{Key mathematical statements} 
  Within each domain, learners should be able to recall and accurately state:
  \begin{enumerate}
    \item \textit{Definitions} — E.g., Define a compact set in a metric space; define a $\sigma$-algebra.
    \item \textit{Theorems} — E.g., State the Intermediate Value Theorem; state Fermat's Little Theorem.
    \item \textit{Frequently Used Lemmas and Tricks} — E.g., What is the pigeonhole principle? State Chebyshev’s inequality.
  \end{enumerate}
  
  These key mathematical statements form the building blocks of further reasoning and problem-solving.
  
  \textbf{Design Principles}: 
  Choose mathematical statements that are familiar to domain experts but are not typically found in standard textbooks.

\textbf{Sample Problems:} \setcounter{problem}{0}


\begin{figure}[H]
    \centering

    \begin{minipage}[t]{0.8\textwidth} \stepcounter{problem}
        \begin{tcolorbox}[
        title=\textbf{\footnotesize Problem \theproblem: Commutation relations for multiple SLE},
            colframe=cyan!40!black,
            boxrule=0.5pt,
            colback=white,
            top=0pt
        ]
    
\vspace{0.1cm}
        \footnotesize   
    Please explain the commutation relations for multiple chordal SLE($\kappa$). 
    
    \fail
    \par
        \end{tcolorbox}
    \end{minipage}

\vspace{0.2cm}
    
    \begin{minipage}[t]{0.8\textwidth} \stepcounter{problem}
    \begin{tcolorbox}[
        title=\textbf{\footnotesize Problem \theproblem: SLE concepts and transformation laws},
            colframe=cyan!40!black,
            boxrule=0.5pt,
            colback=white,
            top=0pt
        ]
        
\vspace{0.1cm}
        \footnotesize  Schramm–Loewner Evolution (SLE) is a vibrant field in modern probability theory. Please address the following two parts:
\begin{enumerate}
  \item 
    Provide the precise definition of SLE, emphasizing its construction via the Loewner equation driven by a random function \(W_t = \sqrt{\kappa}B_t\) where $B_t$ is the standard Brownian motion.
      \item  Please state the transformation law of the driving functions of Loewner chain under conformal Maps. 
    
      \fail
\end{enumerate} \par

        \end{tcolorbox}
    \end{minipage}
\end{figure}

\subsection{Understanding Knowledge and Theory Mastery}

This category assesses the ability to comprehend and correctly apply established mathematical results and theorems.

\subsubsection{Understanding theorems}  
  Demonstrate accurate recall of classical theorems, including their assumptions, conclusions, and logical structure. For example, recognizing that the Extreme Value Theorem requires continuity over a compact interval, and understanding its dependence on the completeness of real numbers.

\textbf{Design Principles:}
Include content that is well-known among experts but rarely emphasized in standard textbooks.

  \textbf{Sample Problems:} \setcounter{problem}{0}
\begin{figure}[H]
    \centering

    \begin{minipage}[t]{0.8\textwidth} \stepcounter{problem}
        \begin{tcolorbox}[
        title=\textbf{\footnotesize Problem \theproblem: Derive null vector equations for multiple SLE systems},
            colframe=cyan!40!black,
            boxrule=0.5pt,
            colback=white,
            top=0pt
        ]
        
\vspace{0.1cm}
        \footnotesize 
  Derive the null vector equations from commutation relations of multiple chordal SLE($\kappa$).

  \fail
 \par
        \end{tcolorbox}
    \end{minipage}
\end{figure}

\subsubsection{Applying classical results} 
  Fluently use known theorems to solve standard, composite, or multi-step problems.

\textbf{Design Principles:}
    Choose problems that require straightforward applications of a very complicated theorem.

\textbf{Sample Problems:} \setcounter{problem}{0}

\begin{figure}[H]
    \centering

    \begin{minipage}[t]{0.8\textwidth} \stepcounter{problem}
        \begin{tcolorbox}[
        title=\textbf{\footnotesize Problem \theproblem: Apply s-cobordism to show diffeomorphism},
            colframe=cyan!40!black,
            boxrule=0.5pt,
            colback=white,
            top=0pt
        ]
    
\vspace{0.1cm}
        \footnotesize   
Apply the s-cobordism theorem to show that if $M$ is a smooth $7$-manifold homotopy equivalent to $S^1\times S^6$, which bounds an $8$-manifold $W$ homotopy equivalent to $S^1$, then $M$ is diffeomorphic to $S^1\times S^6$. 

\fail
\par
        \end{tcolorbox}
    \end{minipage} 
\end{figure}

\subsection{Computational and Analytical Skills}

This category evaluates a model's ability to perform accurate computations and symbolic manipulation.

\subsubsection{Computational fluency} Accurately compute integrals, derivatives, limits, matrix operations, and other standard calculations.

\textbf{Design Principles:}
Choose problems that require difficult, complicated computations.

\textbf{Sample Problems:}
  \setcounter{problem}{0}

\begin{figure}[H]
    \centering

    \begin{minipage}[t]{0.8\textwidth} \stepcounter{problem}
        \begin{tcolorbox}[
        title=\textbf{\footnotesize Problem \theproblem: Evaluate the integral},
            colframe=cyan!40!black,
            boxrule=0.5pt,
            colback=white,
            top=0pt
        ]

\vspace{0.1cm}
        \footnotesize  
\[
I = \int_0^1 \frac{\log\left(1 + x^{2 + \sqrt{3}}\right)}{1 + x} \, dx.
\]
Evaluate this definite integral and express your answer in terms of elementary functions (such as logarithms, powers, trigonometric functions, etc.). 

\fail
\par
        \end{tcolorbox}
    \end{minipage}
\end{figure}   
     
\subsubsection{Symbolic manipulation} 
Simplify or transform algebraic expressions through substitution, factoring, expansion, or equation rearrangement. (The software Mathematica might be easy to solve)

  \textbf{Design Principles:}
Choose problems that require complicated symbolic transformation to reach the solution.

  \textbf{Sample Problems:} \setcounter{problem}{0}

\begin{figure}[H]
    \centering

    \begin{minipage}[t]{0.8\textwidth} \stepcounter{problem}
        \begin{tcolorbox}[
        title=\textbf{\footnotesize Problem \theproblem: Sum of squares},
            colframe=cyan!40!black,
            boxrule=0.5pt,
            colback=white,
            top=0pt
        ]

\vspace{0.1cm}
        \footnotesize  Let $X$ be a homogeneous expression in terms of $a, b, c$, and suppose we want to prove $X \geq 0$.
     Write $X$ in the following form:
$$
(a-b)^2 \cdot S_c+(b-c)^2 \cdot S_a+(c-a)^2 \cdot S_b .
$$

Apply the sum of squares (SOS) method to show that:
    $$16\sum x^4 - 20\sum x^3(y+z) + 9\sum y^2z^2 + 25\sum x^2yz \geq 0.$$
\fail
\par
        \end{tcolorbox}
    \end{minipage}
\end{figure} 
    
\subsection{Problem-Solving Framework}

This category evaluates a model’s ability to approach unfamiliar mathematical problems using structured reasoning. It consists of two key components:

\subsubsection{Problem interpretation and analysis}
  \begin{enumerate}
    \item Parse the problem statement and convert it into a structured representation that reveals variables, constraints, and objectives.
    \item Generate and test simple or edge-case examples to explore possible behaviors and build preliminary insight.
  \end{enumerate}
\textbf{Design Principles:}
Choose problems where understanding the question clearly is key to solving it. Focus on problems that involve tricky wording, hidden assumptions, or subtle edge cases. Good problems in this category often confuse solvers who rush in without first carefully analyzing what’s being asked.

\textbf{Sample Problems:}
\setcounter{problem}{0}
\begin{figure}[H]
    \centering

    \begin{minipage}[t]{0.8\textwidth} \stepcounter{problem}
        \begin{tcolorbox}[
        title=\textbf{\footnotesize Problem \theproblem: Boy return rupees},
            colframe=cyan!40!black,
            boxrule=0.5pt,
            colback=white,
            top=0pt
        ]

\vspace{0.1cm}
        \footnotesize   A boy borrowed 1000 rupees from one of his friends (X). And he borrowed 500 from another friend(Y). And he lost 1000 rupees from his hand on the way to the shop. And he purchased it for only 300 rupees. So he left only 200. And he gave $\$100$ each to those friends ($X$ and $Y$). And he made a calculation.
        
\begin{itemize}
    \item He have to give 900 to $X$
    \item He has to give 400 to $Y$
    \item He have product of 300
\end{itemize}

So total $900+400+300=1600$. How that extra 100 comes from. Please help to solve.

\pass
 \par
 \end{tcolorbox}
    \end{minipage}   

\vspace{0.2cm}    
    \begin{minipage}[t]{0.8\textwidth} \stepcounter{problem}
        \begin{tcolorbox}[
        title=\textbf{\footnotesize Problem \theproblem: Six man wearing hats},
            colframe=cyan!40!black,
            boxrule=0.5pt,
            colback=white,
            top=0pt
        ]

\vspace{0.1cm}
        \footnotesize  A, B, C, D, E, and F are sitting in a circle, each wearing a hat. Each hat is either pure black, pure white, or has an integer written on it. Everyone knows that there is at least one black hat, one white hat, and one numbered hat, and that the numbers on the numbered hats are exactly a sequence of consecutive integers (which can be just one number).

Each person can see everyone else’s hat, but not their own. The host asks if they know what hat they are wearing (they must know the exact color or the exact number, unless they know they are the only one with a numbered hat—in which case, even without knowing the number’s value, it is considered “knowing” their hat). They answer in turn as follows (everyone is perfectly logical and honest):

A: I know.
B: I don’t know.
C: I know.
D: I know.
E: I don’t know.
F: I know.

Seeing E’s puzzled expression, D adds:
“I didn’t know what hat I was wearing at first, but after hearing what C said, I knew.”

E suddenly understands.

Question: Who is wearing colored hats, and who is wearing numbered hats?

\fail
 \par
 \end{tcolorbox}
    \end{minipage}
\end{figure}

 \subsubsection{Application of problem-solving strategies}
  \begin{enumerate}
    \item \textit{Analogy} – Map the problem to a previously solved or canonical problem.
    \item \textit{Induction} – Establish general truths by proving a base case and showing how one case leads to the next.
    \item \textit{Contradiction} – Explore the negation of a target claim and find inconsistencies.
    \item \textit{Working Backwards} – Start from the desired outcome and trace backwards to known results.
    \item \textit{Substitution} – Reformulate the problem using variable changes or structural simplifications.
    \item \textit{Extremal Principle} – Consider extreme or limiting cases (e.g., minimal, maximal, or boundary values).
    \item \textit{Problem Decomposition} – Divide the task into smaller subproblems and recursively solve them.
  \end{enumerate}
\textbf{Design Principles:}
See if there are problems that require a specific strategy. Ensure that the most effective solution path relies on recognizing and applying a specific strategy.

\textbf{Sample Problems:} \setcounter{problem}{0}
 \begin{figure}[H]
    \centering
    
    \begin{minipage}[t]{0.8\textwidth} \stepcounter{problem}
        \begin{tcolorbox}[
        title=\textbf{\footnotesize Problem \theproblem: Color grids},
            colframe=cyan!40!black,
            boxrule=0.5pt,
            colback=white,
            top=0pt
        ]

\vspace{0.1cm}
        \footnotesize  (\emph{Induction}) Let $n\ge3$ be an odd integer.  Color each cell of an $n\times n$ grid either black or white.  Two cells of the same color that share a vertex are called \emph{adjacent}.  For any two cells $a,b$, we say $a$ and $b$ are \emph{connected} if there exists a sequence of cells
\[
c_1,\,c_2,\,\dots,\,c_k
\]
with $c_1 = a$, $c_k = b$, and each $c_i$ is adjacent to $c_{i+1}$ for $i=1,2,\dots,k-1$.  Determine the largest positive integer $M$ for which there is a coloring having $M$ cells that are pairwise non‑connected.

\fail
 \par
 \end{tcolorbox}
    \end{minipage}
    
\vspace{0.2cm}    
    \begin{minipage}[t]{0.8\textwidth} \stepcounter{problem}
        \begin{tcolorbox}[
        title=\textbf{\footnotesize Problem \theproblem: Hamiltonian paths on grids},
            colframe=cyan!40!black,
            boxrule=0.5pt,
            colback=white,
            top=0pt
        ]

\vspace{0.1cm}
        \footnotesize  On a $99\times99$ grid, you begin at the bottom‑left cell and make a path that visits every cell exactly once. At each step you move to an unvisited cell sharing a side with your current position. After you have covered all cells, in how many distinct cells could your path end? Please rigorously justify your answer.

\fail
 \par
 \end{tcolorbox}
    \end{minipage}

\vspace{0.2cm}   
    \begin{minipage}[t]{0.8\textwidth} \stepcounter{problem}
        \begin{tcolorbox}[
        title=\textbf{\footnotesize Problem \theproblem: Extrema of algebraic expression},
            colframe=cyan!40!black,
            boxrule=0.5pt,
            colback=white,
            top=0pt
        ]
        \footnotesize  
(\emph{Problem Decoposition}) Let \( a, b, c, d, e \geq -1 \) be real numbers such that
\[
a + b + c + d + e = 5.
\]
Find the minimum and maximum values of the expression
\[
S = (a + b)(b + c)(c + d)(d + e)(e + a).
\]

\fail
 \par
 \end{tcolorbox}
    \end{minipage}

\vspace{0.2cm}   
    \begin{minipage}[t]{0.8\textwidth} \stepcounter{problem}
        \begin{tcolorbox}[
        title=\textbf{\footnotesize Problem \theproblem: Cake on round table},
            colframe=cyan!40!black,
            boxrule=0.5pt,
            colback=white,
            top=0pt
        ]

\vspace{0.1cm}
        \footnotesize  
(\emph{Contradiction}) On a round table, there are $5$ plates, called plate $1,2,3,4,5$, placed clockwise evenly. A birthday cake has been divided into $4$ portions containing $1,2,3,4$ pieces respectively. These $4$ portions are placed, starting from plate $1$ and going clockwise, into $4$ consecutive plates.

Now, $5$ people sit down at the table, and a game begins. The rules are as follows:

1. If the plate in front of a person is empty, they rotate the turntable to a position of their choice.

2. Step 1 is repeated until someone encounters an empty plate for the second time.

3. Then each person takes the plate in front of them.

Question:
Where should you sit in order to get the most cake? (Assume each person is greedy and smart enough.)

\fail
 \par
 \end{tcolorbox}
    \end{minipage}
\end{figure}

\subsection{Logical Thinking and Reasoning}

This category evaluates the model’s ability to construct and evaluate formal mathematical arguments.

\subsubsection{Constructing valid proofs} Develop logically rigorous and complete arguments using standard proof techniques such as direct proof, contradiction, induction, and case analysis.

\textbf{Design Principles:}
Select problems that require extended logical reasoning, including:
\begin{itemize}
  \item Long chains of deductions or multi-step logical constructions.
  \item Careful case analysis with subtle distinctions or edge conditions.
  \item Proofs involving intricate dependencies between assumptions and conclusions.
\end{itemize}
\textbf{Sample Problems:} \setcounter{problem}{0}
   
 \begin{figure}[H]
    \centering
    \begin{minipage}[t]{0.8\textwidth} \stepcounter{problem}
        \begin{tcolorbox}[
        title=\textbf{\footnotesize Problem \theproblem: 56 elements set},
            colframe=cyan!40!black,
            boxrule=0.5pt,
            colback=white,
            top=0pt
        ]

\vspace{0.1cm}
        \footnotesize  Let \( X \) be a set with 56 elements. Find the smallest positive integer \( n \) such that, for any collection of 15 subsets of \( X \), if the union of any 7 of them contains at least \( n \) elements, then there must exist 3 of these subsets whose intersection is nonempty.

\fail
 \par
 \end{tcolorbox}
    \end{minipage}
\end{figure}

\subsubsection{Evaluating existing arguments} Check whether a given proof or explanation is logically valid, complete, and free of unjustified assumptions or reasoning gaps.

\textbf{Design Principles:}
Provide arguments that appear plausible but contain subtle errors, missing steps, or unjustified claims.

  \textbf{Sample Problems:} \setcounter{problem}{0}

 \begin{figure}[H]
    \centering
    \begin{minipage}[t]{0.8\textwidth} \stepcounter{problem}
        \begin{tcolorbox}[
        title=\textbf{\footnotesize Problem \theproblem: The Banach space paradox},
            colframe=cyan!40!black,
            boxrule=0.5pt,
            colback=white,
            top=0pt
        ]
    
\vspace{0.1cm}
        \footnotesize      
\href{https://mathoverflow.net/questions/44716/counterexample-for-the-open-mapping-theorem}{(MathOverflow Problem)}\citep{OpenMappingCounterexample}  Check the following proof:    
        All Banach norms on a real vector space $X$ are equivalent.

\begin{itemize}
    \item Let $\|\cdot\|_1$ and $\|\cdot\|_2$ two Banach norms on $X$. Consider $\|\cdot\|_3:=\|\cdot\|_1+\|\cdot\|_2$. Prove that it is actually a norm. 
    \item 
Prove that a sequence converge to $x \in X$ w.r.to $\|\cdot\|_3$ if and only if it converges to $x$ both w.r.to $\|\cdot\|_1$ and w.r.to $\|\cdot\|_2$. 
\item 
Prove that a sequence is Cauchy wrto $\|\cdot\|_3$ if and only if it is Cauchy both w.r.to $\|\cdot\|_1$ and w.r.to $\|\cdot\|_2$. Deduce that $\|\cdot\|_3$ is complete. 
\item 
Apply the open mapping theorem tho the identity from $\left(X,\|\cdot\|_3\right)$ to $\left(X,\|\cdot\|_1\right)$ and from $\left(X,\|\cdot\|_3\right)$ to $\left(X,\|\cdot\|_2\right)$, and deduce that the three norms are equivalent.
\end{itemize}

\mhfail\pass
 \par
 \end{tcolorbox}
    \end{minipage}
\end{figure}
\begin{figure}[H]
\centering
    \begin{minipage}[t]{0.8\textwidth} \stepcounter{problem}
        \begin{tcolorbox}[
        title=\textbf{\footnotesize Problem \theproblem: All perfect numbers are even},
            colframe=cyan!40!black,
            boxrule=0.5pt,
            colback=white,
            top=0pt
        ]

\vspace{0.1cm}
        \footnotesize 
\href{https://math.stackexchange.com/questions/348198/best-fake-proofs-a-m-se-april-fools-day-collection}{(MathStackExchange Problem)}\citep{MSE_Did_FakePerfect}
Check the following proof:  all perfect numbers are even. 

Assume

$$
2 n=\sum_{d \mid n} d
$$

Then by Möbius inversion:

$$
n=\sum_{d \mid n} 2 d \cdot \mu\left(\frac{n}{d}\right)
$$

and therefore $n$ is even. 

\pass
 \par
 \end{tcolorbox}
    \end{minipage}
\end{figure}

\subsection{Writing and Presentation}

This category evaluates the ability to communicate mathematical ideas clearly and effectively, both formally and pedagogically.

\subsubsection{Writing clean, structured proofs and definitions} Express arguments and definitions in a precise, logically organized, and readable format suitable for formal assessment or publication.
  
\textbf{Design Principles:} Use problems that require clear, well-organized, and rigorous proofs, especially ones where the arguments are nontrivial to write, so that clarity and structure are essential to the solution.

\textbf{Rubric:}

\begin{tabular}{@{}>{\bfseries}c p{10cm}@{}}
\toprule
Score & Description \\
\midrule
3 (Excellent) & Response is clearly articulated, well-structured, and easy to follow. Ideas are fully developed and effectively communicated. \\
\addlinespace
2 (Good) & Response is generally clear and organized, with minor lapses in structure or explanation.  \\
\addlinespace
1 (Satisfactory) & Response includes relevant ideas but lacks clarity or organization. Transitions or explanations may be confusing. \\
\addlinespace
0 (Bad) & Response is disorganized, vague, or hard to follow. Ideas are poorly expressed or not communicated effectively. \\
\bottomrule
\end{tabular}

\textbf{Sample Problems:} \setcounter{problem}{0} 

\begin{figure}[H]
    \centering
    \stepcounter{problem}
    \begin{minipage}[t]{0.8\textwidth}
        \begin{tcolorbox}[
        title=\textbf{\footnotesize Problem \theproblem: Big fish eats small fish},
            colframe=cyan!40!black,
            boxrule=0.5pt,
            colback=white,
            top=0pt
        ]

\vspace{0.1cm}
        \footnotesize 
        (\emph{Induction}) Suppose there are 10000 fish, numbered from 1 to 10000 in increasing size.  
        The rules are as follows:
\begin{itemize}
  \item A fish can only eat the fish that is exactly one rank smaller than itself (that is, fish $i$ can only eat fish $i - 1$).
  \item Furthermore, fish $i + 1$ can only eat fish $i$ after fish $i$ has eaten fish $i - 1$.
  \item Fish 1 cannot eat any other fish; it can only be eaten.
\end{itemize}

We assume:
\begin{itemize}
  \item If there is a smaller fish available to eat, a larger fish will not pass up the opportunity.
  \item However, self-preservation takes priority: before eating a smaller fish, a fish must ensure that doing so will not lead to it being eaten.
  \item Every fish is infinitely intelligent, and they all know that, and they all know that they all know that, and so on (common knowledge).
\end{itemize}

Question: Will fish 1,2,3 survive? 

\mhfail\pass
 \par
 \end{tcolorbox}
    \end{minipage}
\end{figure}

\subsubsection{Communicating ideas to different audiences}  Adapt the depth, vocabulary, and structure of explanations to match the target audience, ranging from domain experts to general learners or automated evaluators.

\textbf{Design Principles:}
Choose problems that involve explaining a concept to different audiences—for instance, a high school student, a novice learner, and a domain expert. The goal is to assess the ability to adapt clarity, depth, and presentation style based on the listener’s background.

\textbf{Rubric:}

\begin{tabular}{@{}>{\bfseries}c p{10cm}@{}}
\toprule
Level & Description \\
\midrule
3 (Excellent) & Explanation is accurate, clear, and fully adapted to the intended audience. Tone, vocabulary, and structure are consistently effective and appropriate. \\
\addlinespace
2 (Good) & Explanation is with only minor lapses in clarity, tone, or level of detail.  \\
\addlinespace
1 (Satisfactory) & Explanation is overly technical, or too shallow for the intended audience. Tone or structure may hinder understanding. \\
\addlinespace
0 (Bad) & Explanation is confusing, or poorly matched to the audience. \\
\bottomrule
\end{tabular}

\textbf{Note:} The explanation must be correct and accurate. Evaluation is based on the appropriateness of tone, depth, and structure for the target audience.

\textbf{Sample Problems:} \setcounter{problem}{0}

\begin{figure}[H]
    \centering
    \begin{minipage}[t]{0.8\textwidth} \stepcounter{problem}
        \begin{tcolorbox}[
        title=\textbf{\footnotesize Problem \theproblem: Symplectic geometry},
            colframe=cyan!40!black,
            boxrule=0.5pt,
            colback=white,
            top=0pt
        ]

\vspace{0.1cm}
        \footnotesize  Explain to a mathematician why symplectic geometry is the correct setting for classical mechanics.     
 
\fail
\par

 \end{tcolorbox}
    \end{minipage}
\end{figure}



\subsection{Intuition}

This category assesses a model's ability to reason beyond formal steps, relying on geometric insight, estimation, and qualitative judgment to support or critique solutions.

\subsubsection{Geometric intuition} Understand and visualize problems via diagrams and spatial reasoning.

\textbf{Design Principles:}
Select problems where geometric insight—through diagrams, symmetry, or spatial structure—is essential to understanding or solving the problem. Emphasize reasoning that cannot be easily replaced by pure algebra.

\textbf{Sample Problems:} 
\setcounter{problem}{0} 

\begin{figure}[H]
\centering
    \begin{minipage}[t]{0.8\textwidth} \stepcounter{problem}
        \begin{tcolorbox}[
        title=\textbf{\footnotesize Problem \theproblem: Bands in 3-space},
            colframe=cyan!40!black,
            boxrule=0.5pt,
            colback=white,
            top=0pt
        ]

\vspace{0.1cm}
        \footnotesize
Among the 9 bands shown, which ones can be continuously deformed to another in the 3-space, without crossing itself?
\begin{figure}[H]
\centering
\includegraphics[width=0.5\linewidth]{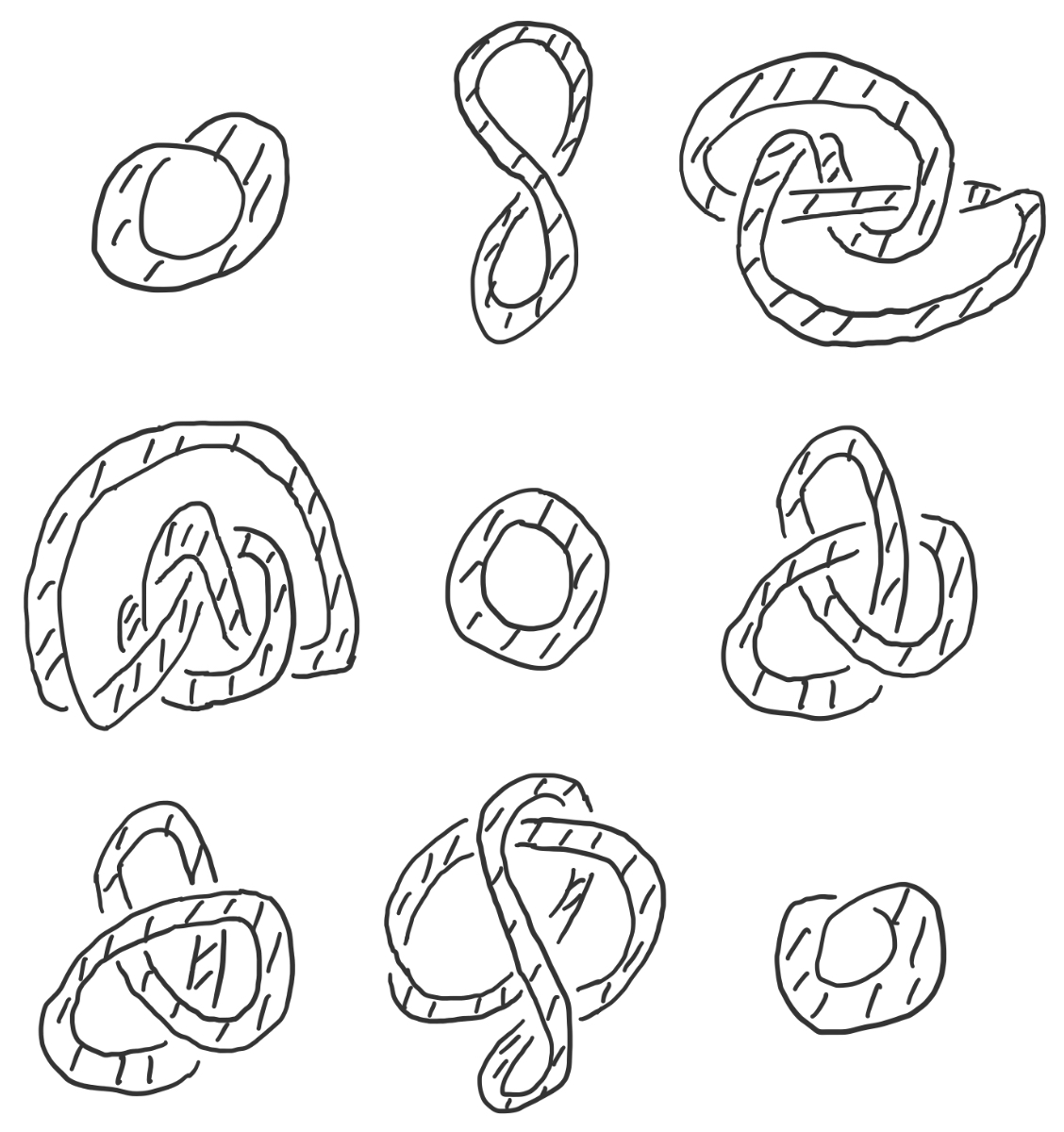}
\end{figure}
\par
\fail        
        \end{tcolorbox}
    \end{minipage}
    
\vspace{0.2cm}

    \begin{minipage}[t]{0.8\textwidth} \stepcounter{problem}
        \begin{tcolorbox}[
        title=\textbf{\footnotesize Problem \theproblem: Surface bounds handlebody},
            colframe=cyan!40!black,
            boxrule=0.5pt,
            colback=white,
            top=0pt
        ]

\vspace{0.1cm}
        \footnotesize  
Show that every smoothly embedded closed orientable connected surface in the $3$-sphere bounds a smoothly embedded $3$-dimensional handlebody (i.e. a boundary sum of solid tori) in the $4$-ball. 
\par
\fail
 \end{tcolorbox}
    \end{minipage}
\end{figure}



\subsubsection{Performing plausibility checks} Use heuristics such as dimensional analysis, trying extreme values, order-of-magnitude reasoning, or test cases to assess the validity of results.

\textbf{Design Principles:}
Construct problems where direct computation or rigorous proof is difficult or impractical, but a clever use of heuristics, such as trying extreme cases, checking units, or order-of-magnitude reasoning, can reveal whether the answer is plausible.

\textbf{Sample Problems:} \setcounter{problem}{0}  
\begin{figure}[H]
    \centering

    \begin{minipage}[t]{0.8\textwidth} \stepcounter{problem}
        \begin{tcolorbox}[
        title=\textbf{\footnotesize Problem \theproblem: Shortest geodesic vs volume},
            colframe=cyan!40!black,
            boxrule=0.5pt,
            colback=white,
            top=0pt
        ]

\vspace{0.1cm}
        \footnotesize  
Let $M$ be a closed Riemannian aspherical $n$-manifold, $n\ge3$. Let $\mathrm{sys}(M)$ denote the length of the shortest nontrivial geodesic on $M$, and $\mathrm{vol}(M)$ denote the volume of $M$. Is it true that $C\cdot\mathrm{sys}(M)^n\ge\mathrm{vol}(M)$ for some $C=C(n)>0$ independent of $M$?   

\mhfail\pass
 \par
 \end{tcolorbox}
    \end{minipage}
\end{figure}    

\subsubsection{Estimating numerically or asymptotically} Provide approximate solutions using asymptotic analysis or numerical methods when exact results are intractable.

\textbf{Design Principles:}
Use problems that involve complex estimation, requiring order-of-magnitude reasoning or bounding techniques to make progress.

\textbf{Sample Problems:} \setcounter{problem}{0}  

\begin{figure}[H]
    \centering
    \begin{minipage}[t]{0.8\textwidth} \stepcounter{problem}
        \begin{tcolorbox}[
        title=\textbf{\footnotesize Problem \theproblem: Asymptotic expansion of a sequence},
            colframe=cyan!40!black,
            boxrule=0.5pt,
            colback=white,
            top=0pt
        ]

\vspace{0.1cm}
        \footnotesize  
\href{https://mathoverflow.net/questions/384047/asymptotic-analysis-of-x-n1-fracx-nn2-fracn2x-n-2}{(MathOverflow Problem)}\citep{MathOverflow384047}  Let $x_1=1$ and $x_{n+1}=\frac{x_n}{n^2}+\frac{n^2}{x_n}+2, n \geqslant 1$. Is the coefficient of the third term in the asymptotic expansion of $x_n$ equal to $\frac{5}{8}$?    
 \par
 \fail
 \end{tcolorbox}
    \end{minipage}

\end{figure}

\subsubsection{Aesthetic judgment} Recognize when a solution is elegant, overly convoluted, or violates natural mathematical principles (e.g., symmetry, minimality). (linear algebra choose without basis or with basis; category problems with universal proof)

\textbf{Design Principles:}
Choose problems with multiple valid approaches, encouraging solvers to recognize and prefer solutions that reflect mathematical elegance—such as symmetry, minimal assumptions, or conceptual clarity—over ad hoc or overly technical methods.

\textbf{Rubric:}

\begin{tabular}{@{}>{\bfseries}c p{10cm}@{}}
\toprule
Level & Description \\
\midrule
3 (Excellent) & Consistently prefers elegant solutions that are minimal, symmetric, or conceptually insightful, and articulates their aesthetic value clearly. \\
\addlinespace
2 (Good) & Recognizes elegant ideas and tends to favor them, with some ability to explain why they are preferable. \\
\addlinespace
1 (Satisfactory) & Sometimes notices elegance but does not consistently prefer or explain it. Shows limited aesthetic awareness. \\
\addlinespace
0 (Bad) & Accepts or produces valid solutions without distinguishing elegant from technical or inelegant methods.\\
\bottomrule
\end{tabular}

\textbf{Note:} All solutions must be mathematically correct. This rubric evaluates preference for elegant, natural, and conceptually clear approaches.

\textbf{Sample Problems:} \setcounter{problem}{0}

\begin{figure}[H]
    \centering
    \begin{minipage}[t]{0.8\textwidth} \stepcounter{problem}
        \begin{tcolorbox}[
        title=\textbf{\footnotesize Problem \theproblem: Trace of tensor product of linear maps},
            colframe=cyan!40!black,
            boxrule=0.5pt,
            colback=white,
            top=0pt
        ]

\vspace{0.1cm}
        \footnotesize   Comparing the following two proofs of the following problem. Which one do you like more? Why?

Problem: Let $V,W$ be finite-dimensional vector spaces and $S\colon V\to V$, $T\colon W\to W$ be linear operators. Show that $tr(S\otimes T)=tr(S)tr(T)$.

Solution 1: Pick bases for $V$, $W$, and regard $S,T$ as square matrices, with entries $S_{i,j}$, $T_{i,j}$. Then $tr(S\otimes T)=\sum_{i,j}S_{i,i}T_{j,j}=(\sum_iS_i)(\sum_jT_j)=tr(S)tr(T)$.

Solution 2: We have a canonical isomorphism $\phi\colon V^*\otimes W^*\cong(V\otimes W)^*$ by $\phi(\ell\otimes\eta)(v\otimes w):=\ell(v)\eta(w)$. This shows the two trace maps $V^*\otimes W^*\otimes V\otimes W\to F$ ($F$ is the base field) given by tracing off $V\otimes W$ or tracing off $V$ and $W$ respectively are equal. Since $S\in V^*\otimes V$, $T\in W^*\otimes W$, this finishes the proof. 

\pass
 \par
 \end{tcolorbox}
    \end{minipage}

\end{figure}

\begin{figure}[H]
    \centering
 
    \begin{minipage}[t]{0.8\textwidth} \stepcounter{problem}
        \begin{tcolorbox}[
        title=\textbf{\footnotesize Problem \theproblem: 17 consecutive numbers},
            colframe=cyan!40!black,
            boxrule=0.5pt,
            colback=white,
            top=0pt
        ]

\vspace{0.1cm}
        \footnotesize 
In a finite sequence of real numbers the sum of any seven successive terms is negative and the sum of any eleven successive terms is positive. Determine the maximum number of terms in the sequence.

Solution 1: Construction for \(n=16\). Define
\[
a_j = 
\begin{cases}
-13,& j\in\{3,7,10,14\},\\
5,& \text{otherwise}.
\end{cases}
\]
Any block of 7 consecutive terms then has exactly two entries equal to \(-13\) and five entries equal to \(5\), so
$
5\cdot5 - 2\cdot13 = -1 < 0.
$
Any block of 11 consecutive terms has exactly three entries equal to \(-13\) and eight entries equal to \(5\), so
$
8\cdot5 - 3\cdot13 = 1 > 0.
$
Thus the sequence meets the requirements.

Non‑existence for \(n\ge17\). One checks (for example by linear‑programming or averaging arguments) that the system
\[
\sum_{j=i}^{i+6}a_j<0,\quad
\sum_{j=i}^{i+10}a_j>0
\]
for \(1\le i\le n-6\) and \(1\le i\le n-10\) has no solution once \(n\ge17\).

Hence the maximum possible length is \(\boxed{16}\).

Solution 2:
   Let $x_1, x_2, \ldots$ be the given sequence and let $s_n=x_1+x_2+\ldots+x_n$. The conditions from the hypothesis can be now written as $s_{n+7}<s_n$ and $s_{n+11}>s_n$ for all $n \geq 1$. We then have:
$0<s_{11}<s_4<s_{15}<s_8<s_1<s_{12}<s_5<s_{16}<s_9<s_2<s_{13}<s_6<s_{17}<s_{10}<s_3<s_{14}<s_7<0$, a contradiction.
Therefore, the sequence cannot have 17 terms. In order to show that 16 is the answer, just take 16 real numbers satisfying $s_{10}<s_3<s_{14}<s_7<0<s_{11}<s_4<s_{15}<s_8<s_1<s_{12}<s_5<s_{16}<s_9<s_2<s_{13}<s_6$. We have $x_1=s_1$ and $x_n=s_n-s_{n-1}$ for $n \geq 2$. Thus we found all sequences with the given properties.

Which do you find more elegant, and why? Discuss in terms of simplicity, symmetry, or conceptual insight.

\pass
 \par
 \end{tcolorbox}
    \end{minipage}
\end{figure}

\subsection{Meta Skills}

This category captures high-level thinking abilities that transcend specific problem types, enabling flexible reasoning across abstract and unfamiliar domains.

\subsubsection{Thinking abstractly} Engage with general structures such as groups, graphs, or vector spaces, and reason about idealized systems rather than specific instances.

\textbf{Design Principles:}  
Select problems that emphasize reasoning about abstract mathematical structures, such as vector spaces, groups, or categories.

\textbf{Sample Problems:} \setcounter{problem}{0}  

\begin{figure}[H]
    \centering
    \begin{minipage}[t]{0.8\textwidth} \stepcounter{problem}
        \begin{tcolorbox}[
        title=\textbf{\footnotesize Problem \theproblem: Open interval},
            colframe=cyan!40!black,
            boxrule=0.5pt,
            colback=white,
            top=0pt
        ]

\vspace{0.1cm}
        \footnotesize Can a nonempty open interval be covered by countably many disjoint closed intervals?   

\fail
 \par
 \end{tcolorbox}
    \end{minipage}
    
\vspace{0.2cm}

    \begin{minipage}[t]{0.8\textwidth} \stepcounter{problem}
        \begin{tcolorbox}[
        title=\textbf{\footnotesize Problem \theproblem: $*$ operation},
            colframe=cyan!40!black,
            boxrule=0.5pt,
            colback=white,
            top=0pt
        ]
        \footnotesize Let $*$ denote an arbitrary binary operation The constant property $x*y=z*w$ implies the commutative property $x*y=y*x$, but the commutative property does not imply the constant property. Is there a property strictly in between these? So that it is implied by the constant property and it implies the commutative properties, but both of these implications are false in the opposite direction. Your property cannot refer to auxiliary fixed elements or use quantifiers and it must be expressed by a single equation (so you cannot say commutative+associative)

\fail
 \par
 \end{tcolorbox}
    \end{minipage}

\vspace{0.2cm}    
\begin{minipage}[t]{0.8\textwidth}
\stepcounter{problem}
        \begin{tcolorbox}[
        title=\textbf{\footnotesize Problem \theproblem: Subtraction in Abelian group},
            colframe=cyan!40!black,
            boxrule=0.5pt,
            colback=white,
            top=0pt
        ]

\vspace{0.1cm}
        \footnotesize  Let $*$ be an arbitrary binary relation on a set $S$. Show that $*$ defines subtraction in an Abelian group if and only if $x*(y*(z*(x*y)))=z$ for all $x,y,z \in S$. 
 \par

 \pass
 \end{tcolorbox}
    \end{minipage}

\end{figure}

\subsubsection{Transferring methods across domains} Apply tools or strategies from one area (e.g., probability) to solve problems in another (e.g., number theory), showcasing flexible conceptual adaptation.

\textbf{Design Principles:}  
Design problems that can be approached using techniques from other mathematical domains. Emphasize conceptual reuse, creative mappings, or analogies between seemingly unrelated fields.

\textbf{Sample Problems:} \setcounter{problem}{0}  
 \begin{figure}[H]
    \centering
    \begin{minipage}[t]{0.8\textwidth} \stepcounter{problem}
        \begin{tcolorbox}[
        title=\textbf{\footnotesize Problem \theproblem: Probabilistic construction of a set},
            colframe=cyan!40!black,
            boxrule=0.5pt,
            colback=white,
            top=0pt
        ]

\vspace{0.1cm}
        \footnotesize Find a measurable subset $A$ of $[0,1]$ such that for any subinterval $I$ of $[0,1]$, the Lebesgue measure $\mu(A \cap I)$ has $0<\mu(A \cap I)<\mu(I)$.  Try to find a probabilistic construction.
        
\pass
 \end{tcolorbox}
    \end{minipage}

\vspace{0.2cm}
   \begin{minipage}[t]{0.8\textwidth} \stepcounter{problem}
        \begin{tcolorbox}[
        title=\textbf{\footnotesize Problem \theproblem: Beside the point},
            colframe=cyan!40!black,
            boxrule=0.5pt,
            colback=white,
            top=0pt
        ]

\vspace{0.1cm}
        \footnotesize 
\href{https://www.janestreet.com/puzzles/beside-the-point-index/}{(Jane Stree Puzzle)}\citep{JaneStreetPuzzle} Two random points, one red and one blue, are chosen uniformly and independently from the interior of a square. To ten decimal places, what is the probability that there exists a point on the side of the square closest to the blue point that is equidistant to both the blue point and the red point?

\mhfail\pass
 \end{tcolorbox}
    \end{minipage}
\end{figure}

\subsection{Learning New Knowledge}

This category assesses how well a model can assimilate new mathematical concepts or techniques from reading or self-guided exploration.

\subsubsection{Reading and understanding advanced materials} Comprehend definitions, theorems, and complex derivations in textbooks, research articles, or competition solutions.

\textbf{Design Principles:}  Problems should go beyond rote recall, requiring solvers to internalize definitions, follow intricate derivations, and recognize how results interact across contexts, while emphasizing subtle assumptions or technical conditions to test genuine comprehension over surface familiarity.

\textbf{Sample Problems:} \setcounter{problem}{0}

 \begin{figure}[H]
    \centering
        \begin{minipage}[t]{0.8\textwidth} \stepcounter{problem}
        \begin{tcolorbox}[
        title=\textbf{\footnotesize Problem \theproblem: Loewner Energy},
            colframe=cyan!40!black,
            boxrule=0.5pt,
            colback=white,
            top=0pt
        ]

\vspace{0.1cm}
        \footnotesize 
You are reading the paper \href{https://arxiv.org/pdf/2006.08574}{Large deviations of multichordal $\mathrm{SLE}_{0+}$, real rational functions, and zeta-regularized determinants of Laplacians}~\citep{peltolawang24} and encounter the following concept: 

\begin{quote}
\textbf{Loewner potential}
\end{quote}
\begin{enumerate}
    \item Identify the concept and explain its context. 
    \item What are the relations between Loewner energy and Loewner potential? 
  \item How do we apply it to the deterministic limit of multiple chordal SLE systems? 
\end{enumerate}\pass
 \par
 \end{tcolorbox}
    \end{minipage}

\vspace{0.2cm}
    \begin{minipage}[t]{0.8\textwidth}
    \stepcounter{problem}
        \begin{tcolorbox}[
        title=\textbf{\footnotesize Problem \theproblem: Local multiple SLE systems},
            colframe=cyan!40!black,
            boxrule=0.5pt,
            colback=white,
            top=0pt
        ]

\vspace{0.1cm}
        \footnotesize 
You are reading the paper \href{https://arxiv.org/abs/2505.16093}{Multiple chordal SLE($\kappa$) and quantum Calogero-Moser system}~\citep{zhang2025multiplechordalsle} and encounter the following concept: 

\begin{quote}
\textbf{Local chordal multiple SLE systems}
\end{quote}
\begin{enumerate}
    \item Identify the concept and explain its context. 
  \item Restate the assumptions formally.
  \item Why do we call it a local system? 
  \item How do we define the compatible family of measures? 
\end{enumerate}
\mhfail\pass
 \par
 \end{tcolorbox}
    \end{minipage}
\end{figure}
  
\subsubsection{Self-directed learning} Independently acquire new tools, techniques, or concepts from unfamiliar topics, and apply them to solve problems or extend known results.

\textbf{Design Principles:}  
Design tasks that combine a technical question with a short excerpt from an advanced paper or textbook introducing unfamiliar definitions or methods. These tasks should test the learner’s ability to apply them to solve problems or extend known results.

\textbf{Sample Problems:} \setcounter{problem}{0}

 \begin{figure}[H]
    \centering
    \begin{minipage}[t]{0.8\textwidth} \stepcounter{problem}
        \begin{tcolorbox}[
        title=\textbf{\footnotesize Problem \theproblem: Generalize multiple SLE systems},
            colframe=cyan!40!black,
            boxrule=0.5pt,
            colback=white,
            top=0pt
        ]

\vspace{0.1cm}
        \footnotesize 
You are reading  \href{https://arxiv.org/abs/2505.16093}{Multiple chordal SLE($\kappa$) and quantum Calogero-Moser system}~\citep{zhang2025multiplechordalsle} and encounter the following concept:

\begin{quote}
\textbf{Local chordal multiple SLE systems}
\end{quote}

\begin{enumerate}
    \item Please generalize the concept to the multiple radial case.
    \item Please define the compatible family of measures for multiple radial SLE systems.
\end{enumerate}
\fail
 \par
 \end{tcolorbox}
    \end{minipage}
\end{figure}
\subsection{Mathematical Modeling}

This category measures the ability to represent real-world or theoretical situations using appropriate mathematical structures.

\subsubsection{Translating real-world problems into math} Build well-defined mathematical models that accurately capture the key features of physical, social, or economic phenomena.

\textbf{Design Principles:}  
Create problems that describe complex real-world scenarios and require solvers to construct precise mathematical models.


\textbf{Sample Problems:} \setcounter{problem}{0} 

 \begin{figure}[H]
    \centering
    \begin{minipage}[t]{0.8\textwidth} \stepcounter{problem}
        \begin{tcolorbox}[
        title=\textbf{\footnotesize Problem \theproblem: Sea lampreys sex ratio},
            colframe=cyan!40!black,
            boxrule=0.5pt,
            colback=white,
            top=0pt
        ]

\vspace{0.1cm}
        \footnotesize 
\href{https://www.contest.comap.com/undergraduate/contests/mcm/contests/2024/problems/}{(MCM 2024 Problem A)}\citep{MCM2024A} While some animal species exist outside of the usual male or female sexes, most species are substantially either male or female. Although many species exhibit a 1:1 sex ratio at birth, other species deviate from an even sex ratio. This is called adaptive sex ratio variation. For example, the temperature of the nest incubating eggs of the American alligator influences the sex ratios at birth.

The role of lampreys is complex. In some lake habitats, they are seen as parasites with a significant impact on the ecosystem, whereas lampreys are also a food source in some regions of the world, such as Scandinavia, the Baltics, and for some Indigenous peoples of the Pacific Northwest in North America.

The sex ratio of sea lampreys can vary based on external circumstances. Sea lampreys become male or female depending on how quickly they grow during the larval stage. These larval growth rates are influenced by the availability of food. In environments where food availability is low, growth rates will be lower, and the percentage of males can reach approximately $78 \%$ of the population. In environments where food is more readily available, the percentage of males has been observed to be approximately $56 \%$ of the population.

We focus on the question of sex ratios and their dependence on local conditions, specifically for sea lampreys. Sea lampreys live in lake or sea habitats and migrate up rivers to spawn. The task is to examine the advantages and disadvantages of the ability for a species to alter its sex ratio depending on resource availability. Your team should develop and examine a model to provide insights into the resulting interactions in an ecosystem.

Questions to examine include the following:
\begin{enumerate}
    \item What is the impact on the larger ecological system when the population of lampreys can alter its sex ratio?
    \item What are the advantages and disadvantages to the population of lampreys?
    \item What is the impact on the stability of the ecosystem given the changes in the sex ratios of lampreys?
    
\fail
\par
\end{enumerate}   
 \par
 \end{tcolorbox}
    \end{minipage}
\end{figure}

\subsection{Generalization}

This category concerns recognizing and extending patterns from specific cases to broader mathematical truths.

\subsubsection{Identifying patterns from examples} Detect recurring structures or behaviors in small-scale problems and use them to predict or conjecture more general results.

\textbf{Design Principles:}  
Present a sequence of concrete examples or small cases that exhibit an underlying mathematical pattern.

 \textbf{Sample Problems:} \setcounter{problem}{0}  
 \begin{figure}[H]
    \centering
    \begin{minipage}[t]{0.8\textwidth} \stepcounter{problem}
        \begin{tcolorbox}[
        title=\textbf{\footnotesize {Problem \theproblem: k-set containing three sums}},
            colframe=cyan!40!black,
            boxrule=0.5pt,
            colback=white,
            top=0pt
        ]

\vspace{0.1cm}
        \footnotesize 
Find the smallest positive integer \( k \) such that for every \( k \)-element subset \( A \) of the set
\[
S = \{1, 2, \dots, 2012\},
\]
there exist three \textbf{distinct} elements \( a, b, c \in S \) such that all three sums
\[
a + b,\quad b + c,\quad c + a
\]
are contained in \( A \). Try small cases to find a pattern then prove your conjecture.

\fail
 \par
 \end{tcolorbox}
    \end{minipage}
\end{figure}

\subsubsection{Developing unifying theories} Create general frameworks that consolidate seemingly unrelated theorems or observations into a coherent structure.

\textbf{Design Principles:}  
Design problems that involve multiple related results. Challenge solvers to discover the shared principles or structural similarities behind them, and to formulate a unifying framework—such as a general theorem—that explains all cases and suggests further extensions.

\textbf{Sample Problems:} \setcounter{problem}{0} 

 \begin{figure}[H]
    \centering
    \begin{minipage}[t]{0.8\textwidth} \stepcounter{problem}
        \begin{tcolorbox}[
        title=\textbf{\footnotesize Problem \theproblem: Generalize Dehn twists},
            colframe=cyan!40!black,
            boxrule=0.5pt,
            colback=white,
            top=0pt
        ]

\vspace{0.1cm}
        \footnotesize 
Consider the following three ``twist'' operations:
\begin{enumerate}
\item [(1)] If $\Sigma$ is an oriented surface and $\gamma\subset\Sigma$ is an embedded curve, one can perform a Dehn twist along $\gamma$ to create a mapping class of $\Sigma$.
\item [(2)] If $M$ is an oriented $3$-manifold and $T\subset M$ is an embedded torus, $\gamma\subset T$ is an essential simple closed curve, one can perform a Dehn twist along $T$ in the $\gamma$-direction to create a mapping class of $M$. 
\item [(3)] If $X$ is a smooth oriented $4$-manifold and $S\subset X$ is smoothly embedded $2$-sphere with self-intersection number $-2$, one can construct a mapping class of $X$ obtained by the identity outside a tubular neighborhood $\nu S$ of $S$, and the time $\pi$ map of the geodesic flow on $DTS^2\cong\nu S$ inside $\nu S$, where $DTS^2$ denotes the unit disk bundle of the tangent bundle of the unit sphere $S^2$.
\end{enumerate}
What is essentially going on that allows one to perform such twist operations? Please:
\begin{enumerate}
\item [i.] Formulate a generalization of the twist operation that encompasses all three operations above as special cases.
\item [ii.] Formulate a generalization of the first two twist operations that is not a generalization of the third operation.
\end{enumerate}

\fail
 \par
 \end{tcolorbox}
    \end{minipage}

\vspace{0.2cm}
      \begin{minipage}[t]{0.8\textwidth} \stepcounter{problem}
        \begin{tcolorbox}[
        title=\textbf{\footnotesize Problem \theproblem: Clock strikes},
            colframe=cyan!40!black,
            boxrule=0.5pt,
            colback=white,
            top=0pt
        ]
    
\vspace{0.1cm}
        \footnotesize A clock is built based on the following repeating sequence
1,2,3,4,3,2,1,2,3,4,3,2,...
At the first hour it strikes 1 time. At the second hour it strikes 2 times. At the third hour it strikes 3 times. At the fourth hour it strikes 4 times. At the fifth hour it strikes 3+2 times. At the sixth hour it strikes 1+2+3 times. Etc.
Explain what is going on, develop the theory that explains why it works, and then suggest some generalizations.

\fail
 \par
 \end{tcolorbox}
    \end{minipage}
\end{figure}
\begin{figure}[H]
\centering
    \begin{minipage}[t]{0.8\textwidth} \stepcounter{problem}
        \begin{tcolorbox}[
        title=\textbf{\footnotesize Problem \theproblem: Group scheme and classifying space},
            colframe=cyan!40!black,
            boxrule=0.5pt,
            colback=white,
            top=0pt
        ]

\vspace{0.1cm}
        \footnotesize  Explain the connection between the following two observations. Then propose a unifying theory that predicts decompositions for an arbitrary group.
\begin{enumerate}
    \item Let $\mathbb{G}_m$ be the multiplicative group scheme. It has exactly $q-1$ $\mathbb{F}_q$ points.
    \item Let $B\mathbb{G}_m=\mathrm{pt}/\mathbb{G}_m$ be the classifying stack of $\mathbb{G}_m$ bundles. Over the complex numbers it is represented by the homotopy type $\mathbb{CP}^{\infty}$ which admits a decomposition $\mathbb{C}^0+\mathbb{C}^1+\mathbb{C}^2+\dots$. If we instead ``counted $\mathbb{F}_q$ points", we would get the non-convergent series $1+q+q^2+\dots=1/(1-q)$.

\end{enumerate}
\pass
 \par
 \end{tcolorbox}
\end{minipage}
\end{figure}

\subsection{Creativity}

This category measures originality and innovation in mathematical thinking, beyond standard methods or known results.

\subsubsection{Posing new problems and defining new concepts} Formulate original mathematical questions or introduce novel definitions that open up new lines of inquiry.

\textbf{Design Principles:}  
Design tasks that ask solvers to formulate original questions or definitions inspired by familiar patterns or concepts. Prompts should encourage open-ended exploration, with emphasis on novelty and potential for generalization or further investigation. 


\textbf{Rubric}

\textbf{Note:} Responses must be mathematically well-formed and valid. Evaluation focuses on novelty, clarity, and potential for exploration.

\vspace{1em}

\begin{tabular}{@{}>{\bfseries}c p{10cm}@{}}
\toprule
Level & Description \\
\midrule
2 (Excellent) & Poses a novel and well-formulated problem or definition with clear mathematical structure. Opens a path for meaningful exploration or generalization. Shows originality and insight. \\
\addlinespace
1 (Good) & Problem or concept is mathematically sound and somewhat novel. May resemble known ideas but reframed with fresh perspective or extension. \\
\addlinespace
0 (Satisfactory) & Problem or concept is valid but closely mirrors standard material. Shows basic understanding but lacks originality or direction for further inquiry. \\
\bottomrule
\end{tabular}

\textbf{Sample Problems:} \setcounter{problem}{0} 
 \begin{figure}[H]
    \centering
    \begin{minipage}[t]{0.8\textwidth} \stepcounter{problem}
        \begin{tcolorbox}[
        title=\textbf{\footnotesize Problem \theproblem: Massive SLE},
            colframe=cyan!40!black,
            boxrule=0.5pt,
            colback=white,
            top=0pt
        ]

\vspace{0.1cm}
        \footnotesize Define massive SLE and describe the new questions and properties it may raise.

\pass
 \par
 \end{tcolorbox}
    \end{minipage}
\end{figure}

\subsubsection{Inventing new methods or approaches} Devise alternative strategies, techniques, or algorithms that offer new ways of solving old problems.

\textbf{Design Principles:}  
Select problems that can be solved using standard methods but also admit a simpler or more elegant solution through unconventional reasoning.

\textbf{Sample Problems:} \setcounter{problem}{0} 
 \begin{figure}[H]
    \centering
    \begin{minipage}[t]{0.8\textwidth} \stepcounter{problem}
        \begin{tcolorbox}[
        title=\textbf{\footnotesize Problem \theproblem: Maximum allowable sequence},
            colframe=cyan!40!black,
            boxrule=0.5pt,
            colback=white,
            top=0pt
        ]
        \footnotesize Find at least two solutions to \href{https://www.imo-official.org/problems/1977/eng.pdf}{1977 IMO Problem 2}\citep{IMO1977P2_official}:
In a finite sequence of real numbers the sum of any seven successive terms is negative and the sum of any eleven successive terms is positive. Determine the maximum number of terms in the sequence.   

\pass
 \par
 \end{tcolorbox}
    \end{minipage}
\end{figure}

\subsubsection{Lateral thinking} Apply non-obvious insights, analogies, or reversals to reframe a problem and reach a solution creatively.

\textbf{Design Principles:}  
Design brain-teaser-style problems that resist direct or conventional approaches and instead require solvers to adopt unexpected perspectives, reinterpret elements of the problem, or make surprising conceptual connections.

\textbf{Sample Problems:}
\setcounter{problem}{0}
 \begin{figure}[H]
    \centering
    \begin{minipage}[t]{0.8\textwidth} \stepcounter{problem}
        \begin{tcolorbox}[
        title=\textbf{\footnotesize Problem \theproblem: Make the equation valid},
            colframe=cyan!40!black,
            boxrule=0.5pt,
            colback=white,
            top=0pt
        ]

\vspace{0.1cm}
        \footnotesize Move any one digit to make this equation true:
    $62-63=1$.   

\fail
 \par
 \end{tcolorbox}
    \end{minipage}
\end{figure}

\newpage
\section{Skills Profile for GPT-5 thinking and Comparison with o4-mini-high }

The Math Skills Breakdown evaluation provides a structured and fine-grained profile of GPT-5-thinking’s mathematical abilities across twelve core skill dimensions, each comprising specific sub-skills. This framework enables the assessment of both overall competence and localized strengths or weaknesses in mathematical abilities. All scores are reported on a 0–100 scale.

\paragraph{Scoring Methodology.}
Each skill category $i$ has $m_i$ sub-skills indexed by $j$. For sub-skill $(i,j)$,
\[
\text{SubScore}_{i,j} = 100 \cdot \frac{n_{\mathrm{pass},i,j}}{n_{\mathrm{total},i,j}}.
\]
Here, $n_{\mathrm{pass},i,j}$ denotes the number of correctly solved problems for sub-skill $(i,j)$, and $n_{\mathrm{total},i,j}$ is the total number of problems assigned to that sub-skill. If a problem spans multiple skills, its weight will be proportionally assigned to each skill.

The \emph{balanced score} is defined as the uniform (equal-weighted) average of its sub-skill scores:
\[
\text{BalancedScore}_i := \frac{1}{m_i}\sum_{j=1}^{m_i} \text{SubScore}_{i,j}.
\]

For comparison, we also report the \emph{weighted score}, which aggregates performance across all problems in the category:
\[
\text{WeightedScore}_i := 100 \cdot \frac{\sum_{j=1}^{m_i} n_{\mathrm{pass},i,j}}
{\sum_{j=1}^{m_i} n_{\mathrm{total},i,j}}.
\]
Here, the numerator counts the total number of correctly solved problems across all sub-skills in category $i$, while the denominator is the total number of problems assigned to that category.  
Equivalently,
\[
\text{WeightedScore}_i = \frac{\sum_{j=1}^{m_i} n_{\mathrm{total},i,j}\,\text{SubScore}_{i,j}}{\sum_{j=1}^{m_i} n_{\mathrm{total},i,j}}.
\]

Thus, $\text{BalancedScore}_i$ treats each sub-skill equally, while $\text{WeightedScore}_i$ reflects overall accuracy across all problems. In our evaluation, the balanced score is reported as the \emph{primary metric}, since it prevents categories with larger sub-skill problem sets from dominating the overall assessment.

To carry out the evaluation, we first collected GPT-5-thinking’s responses across the curated problem sets. 
The raw results, organized by skill category and sub-skill, are summarized in Table~\ref{skill and subskill performance table}, 
which records the number of problems attempted under each sub-skill, together with counts of correct and incorrect responses.

\begin{table}[H]
\centering
\caption{Skill Categories and Subskills Performance}
\label{skill and subskill performance table}
\rowcolors{2}{lightgray}{white} 
\begin{tabular}{p{3.6cm} p{3.6cm} >{\centering\arraybackslash}m{1.5cm} 
                                    >{\columncolor{lightgreen}\centering\arraybackslash}m{1.5cm} 
                                    >{\columncolor{lightred}\centering\arraybackslash}m{1.5cm}}
\toprule
\textbf{Skill Category} & \textbf{Subskill} & \textbf{Problems} & \textbf{Correct} & \textbf{Incorrect} \\
\midrule
Memory of Math Knowledge & Taxonomy of Subjects        & 2 & 2 & 0 \\
                         & Key Mathematical Statements & 2 & 0 & 2 \\
\midrule
Understanding Knowledge and Theories & Understanding theorems     & 1 & 0 & 1 \\
& Applying classical results & 1 & 0 & 1 \\
\midrule
Computational and Analytical Skills & Computational fluency  & 1 & 0 & 1 \\
& Symbolic manipulation & 1 & 0 & 1 \\
\midrule
Problem-Solving Framework & Problem Interpretation \& Analysis & 2 & 1 & 1 \\
& Application of Strategies          & 4 & 0 & 4 \\
\midrule
Logical Thinking and Reasoning & Constructing proofs    & 1 & 0 & 1 \\
& Evaluating arguments   & 2 & 2 & 0 \\
\midrule
Writing and Presentation & Structured proofs & 1 & 1 & 0 \\
                         & Communicating to audiences    & 1 & 0 & 1 \\
\midrule
Intuition & Geometric intuition             & 1 & 0 & 1 \\
          & Plausibility checks             & 1 & 1 & 0 \\
          & Numerical/ asymptotic estimation & 1 & 0 & 1 \\
          & Aesthetic judgment              & 2 & 2 & 0 \\
\midrule
Meta Skills & Thinking abstractly   & 3 & 1 & 2 \\
            & Transferring methods  & 2 & 2 & 0 \\
\midrule
Learning New Knowledge & Reading advanced materials & 2 & 2 & 0 \\
                       & Self-directed learning     & 1 & 0 & 1 \\
\midrule
Mathematical Modeling & Translating real-world problems & 1 & 0 & 1 \\
\midrule
Generalization & Identifying patterns       & 1 & 0 & 1 \\
               & Developing unifying theories & 3 & 1 & 2 \\
\midrule
Creativity & Posing new problems & 1 & 1 & 0 \\
           & Inventing methods   & 1 & 1 & 0 \\
           & Lateral thinking    & 1 & 0 & 1 \\
\bottomrule
\end{tabular}
\end{table}

\newpage 
We then compute per-subskill accuracies (\emph{SubScores}), as well as aggregate \emph{Balanced} and \emph{Weighted} scores for each category. 
These results are reported in Table~\ref{skill balanced weighted}, enabling comparison across skills that differ in the number and distribution of problems.

\begin{table}[H]
\centering
\caption{Skill Categories and Subskills with Balanced and Weighted Scores}
\label{skill balanced weighted}
\rowcolors{2}{lightgray}{white} 
\begin{tabular}{p{3.6cm} p{3.6cm} >{\centering\arraybackslash}m{1.5cm} 
                                    >{\columncolor{lightblue}\centering\arraybackslash}m{1.5cm} 
                                    >{\columncolor{lightorange}\centering\arraybackslash}m{1.5cm}}
\toprule
\textbf{Skill Category} & \textbf{Subskill} & \textbf{SubScore (\%)} & \textbf{Balanced (\%)} & \textbf{Weighted (\%)} \\
\midrule
Memory of Math Knowledge & Taxonomy of Subjects        & 100.0 & 50.0 & 50.0 \\
                         & Key Mathematical Statements &   0.0 & --   & --   \\
\midrule
Understanding Knowledge and Theories & Understanding theorems     & 0.0 & 0.0 & 0.0 \\
                                     & Applying classical results & 0.0 & --  & --  \\
\midrule
Computational and Analytical Skills & Computational fluency  & 0.0 & 0.0 & 0.0 \\
                                    & Symbolic manipulation & 0.0 & --  & --  \\
\midrule
Problem-Solving Framework & Problem Interpretation \& Analysis & 50.0 & 25.0 & 16.7 \\
                          & Application of Strategies          & 0.0  & --   & --   \\
\midrule
Logical Thinking and Reasoning & Constructing proofs  & 0.0 & 50.0 & 66.7 \\
                               & Evaluating arguments & 100.0 & -- & -- \\
\midrule
Writing and Presentation & Structured proofs & 100.0 & 50.0 & 50.0 \\
                         & Communicating to audiences    &   0.0 & --   & --   \\
\midrule
Intuition & Geometric intuition             & 0.0 & 50.0 & 60.0 \\
          & Plausibility checks             & 100.0 & -- & -- \\
          & Numerical/ asymptotic estimation & 0.0 & -- & -- \\
          & Aesthetic judgment              & 100.0 & -- & -- \\
\midrule
Meta Skills & Thinking abstractly   & 33.3 & 66.7 & 60.0 \\
            & Transferring methods  & 100.0 & -- & -- \\
\midrule
Learning New Knowledge & Reading advanced materials & 100.0 & 50.0 & 66.7 \\
                       & Self-directed learning     &   0.0 & --   & --   \\
\midrule
Mathematical Modeling & Translating real-world problems & 0.0 & 0.0 & 0.0 \\
\midrule
Generalization & Identifying patterns       & 0.0 & 16.7 & 25.0 \\
               & Developing unifying theories & 33.3 & -- & -- \\
\midrule
Creativity & Posing new problems & 100.0 & 66.7 & 66.7 \\
           & Inventing methods   & 100.0 & --   & --   \\
           & Lateral thinking    &   0.0 & --   & --   \\
\bottomrule
\end{tabular}
\end{table}

\newpage
Finally, Table~\ref{skill heatmap descriptions} presents the \emph{SubScores} for each skill category in a heat-map style, 
accompanied by concise descriptions of strengths and weaknesses. 
This representation highlights both quantitative performance and qualitative patterns in GPT-5-thinking’s mathematical abilities.
\begin{table}[H]
\centering
\caption{Subskills with SubScores Heat Map and Descriptions}
\label{skill heatmap descriptions}
\begin{tabular}{
    p{4cm}m{2cm}
    p{6.5cm}
}
\toprule
\textbf{Subskill} & \textbf{SubScore (\%)} & \textbf{Description} \\
\midrule

\rowcolor{lightgray}\multicolumn{3}{l}{\textbf{Memory of Math Knowledge}} \\
\rowcolors{2}{white}{lightgray}
Taxonomy of Subjects                 & \scorecell{100} & Excellent recall of subject hierarchy and naming. \\
Key Mathematical Statements          & \scorecell{0}   & Struggles to restate definitions/theorems precisely. \\

\rowcolor{lightgray}\multicolumn{3}{l}{\textbf{Understanding Knowledge and Theories}} \\
\rowcolors{2}{white}{lightgray}
Understanding theorems               & \scorecell{0}   & Fails to demonstrate conceptual grasp of statements. \\
Applying classical results           & \scorecell{0}   & Unable to select/use appropriate theorems in context. \\

\rowcolor{lightgray}\multicolumn{3}{l}{\textbf{Computational and Analytical Skills}} \\
\rowcolors{2}{white}{lightgray}
Computational fluency                & \scorecell{0}   & Errors in complicated calculations. \\
Symbolic manipulation                & \scorecell{0}   & Algebraic/transform steps break down or misapplied. \\

\rowcolor{lightgray}\multicolumn{3}{l}{\textbf{Problem-Solving Framework}} \\
\rowcolors{2}{white}{lightgray}
Problem Interpretation \& Analysis   & \scorecell{50}  & Often parses problem correctly but not consistently. \\
Application of Strategies            & \scorecell{0}   & Strategy selection/translation to steps is missing. \\

\rowcolor{lightgray}\multicolumn{3}{l}{\textbf{Logical Thinking and Reasoning}} \\
\rowcolors{2}{white}{lightgray}
Constructing proofs                  & \scorecell{0}   & Cannot assemble coherent, rigorous argument. \\
Evaluating arguments                 & \scorecell{100} & Reliably spots flaws and checks soundness. \\

\rowcolor{lightgray}\multicolumn{3}{l}{\textbf{Writing and Presentation}} \\
\rowcolors{2}{white}{lightgray}
Structured proofs                    & \scorecell{100} & Clear structure when a proof is known. \\
Communicating to audiences           & \scorecell{0}   & Exposition lacks clarity/reader guidance. \\

\rowcolor{lightgray}\multicolumn{3}{l}{\textbf{Intuition}} \\
\rowcolors{2}{white}{lightgray}
Geometric intuition                  & \scorecell{0}   & Visual/shape-based reasoning not effectively used. \\
Plausibility checks                  & \scorecell{100} & Good at sanity checks and qualitative validation. \\
Numerical/ asymptotic estimation     & \scorecell{0}   & Weak sense for scalings, magnitudes, orders. \\
Aesthetic judgment                   & \scorecell{100} & Recognizes elegance/simplicity in solutions. \\

\rowcolor{lightgray}\multicolumn{3}{l}{\textbf{Meta Skills}} \\
\rowcolors{2}{white}{lightgray}
Thinking abstractly                  & \scorecell{33}  & Partial ability to generalize beyond instances. \\
Transferring methods                 & \scorecell{100} & Successfully reuses tools across domains. \\

\rowcolor{lightgray}\multicolumn{3}{l}{\textbf{Learning New Knowledge}} \\
\rowcolors{2}{white}{lightgray}
Reading advanced materials           & \scorecell{100} & Extracts key ideas from technical texts. \\
Self-directed learning               & \scorecell{0}   & Struggles to independently study. \\

\rowcolor{lightgray}\multicolumn{3}{l}{\textbf{Mathematical Modeling}} \\
\rowcolors{2}{white}{lightgray}
Translating real-world problems      & \scorecell{0}   & Difficulty formalizing assumptions/variables. \\

\rowcolor{lightgray}\multicolumn{3}{l}{\textbf{Generalization}} \\
\rowcolors{2}{white}{lightgray}
Identifying patterns                 & \scorecell{0}   & Patterns not reliably detected from examples. \\
Developing unifying theories         & \scorecell{33}  & Occasional synthesis, but lacks consistency. \\

\rowcolor{lightgray}\multicolumn{3}{l}{\textbf{Creativity}} \\
\rowcolors{2}{white}{lightgray}
Posing new problems                  & \scorecell{100} & Generates meaningful, well-posed questions. \\
Inventing methods                    & \scorecell{100} & Proposes original methods. \\
Lateral thinking                     & \scorecell{0}   & Rarely makes unconventional connections. \\
\bottomrule
\end{tabular}
\end{table}

\noindent\textbf{Overall Assessment:}  
GPT-5-thinking exhibits clear strengths across both memory- and reasoning-oriented abilities. Notably, it achieves perfect performance in recalling the \textbf{taxonomy of mathematical subjects} (memory), in \textbf{evaluating existing arguments} (reasoning), in \textbf{plausibility checks} (intuition), in \textbf{reading advanced materials} (learning), and in \textbf{problem posing} (creativity). These outcomes highlight competence in knowledge organization, error detection, intuitive judgment, comprehension of technical material, and problem formulation.

By contrast, GPT-5-thinking performs poorly in foundational areas including \textbf{understanding of knowledge and theories} (0\%), \textbf{computational and analytical skills} (0\%), \textbf{application of 
problem-solving strategies} (0\%), and \textbf{mathematical modeling} (0\%), with additional weakness in \textbf{generalization} (17\%). These deficiencies indicate persistent difficulties in understanding/applying theorems, executing symbolic computations, applying systematic strategies in problem solving, and developing unifying frameworks.

Taken together, GPT-5-thinking demonstrates strong capabilities in structuring mathematical domains, absorbing advanced material, spotting logical flaws, leveraging intuition, and posing problems. However, its higher-order reasoning remains fragile without stronger mastery of core computational and theoretical skills.  

\noindent\textbf{Comparison with o4-mini-high:}  
Relative to o4-mini-high, GPT-5-thinking shows measurable improvements in the \textbf{taxonomy of mathematical subjects} (memory), in \textbf{evaluating existing arguments} (reasoning), in \textbf{plausibility check} (intuition), in \textbf{writing proofs} (writing) and in \textbf{reading advanced materials} (learning). In most other categories, however, performance gains are marginal and the overall skill profile remains largely unchanged. This comparison should be interpreted with caution, as a full evaluation of o4-mini-high  is not available.

\newpage
\section{Future Directions}

We believe that the benchmark framework proposed in this paper establishes a principled basis for evaluating large language models (LLMs) in mathematical reasoning, shifting the focus from overall accuracy to a more interpretable, skill-based understanding. Rather than viewing LLMs as monolithic solvers, this decomposition allows us to assess which cognitive abilities—such as memory, symbolic fluency, logical structure, or creative strategy—are emerging and which remain weak or brittle. By grounding evaluation in human-like skill taxonomies, we aim to better understand not only what models can do, but how they reason and where that reasoning begins to break down.

\textbf{Systematic skill profiling across models} will be crucial for charting progress. Current LLMs display uneven capabilities—often excelling in algebraic manipulation and syntactic logic, while faltering on tasks requiring abstraction, creativity, or generalization. Such skill profiles can feed into radar-style visualizations, where each axis corresponds to a distinct mathematical competency, as the example shown in Figure \ref{fig:radar-chart}. These visual summaries not only help interpret models’ strengths and weaknesses but also serve as diagnostic tools for deployment contexts. Coupled with longitudinal tracking of models over time, we can better understand the trajectory of skill emergence, detect stagnation, and identify developmental bottlenecks that resist improvement despite scaling. We plan to test this framework by generating skill profiles for a variety of models, including both open-source and proprietary LLMs. For each mathematical skill, we will evaluate models on representative tasks and compare their performance. 

In addition, we can compare how performance varies across different training and fine-tuning strategies. This enables not only cross-sectional comparisons across model families, but also longitudinal studies that track how specific abilities develop—such as whether generalization emerges earlier in instruction-tuned models, or whether abstraction improves significantly only after reinforcement learning.

\begin{figure}[H]   \centering
\includegraphics[width=\linewidth]{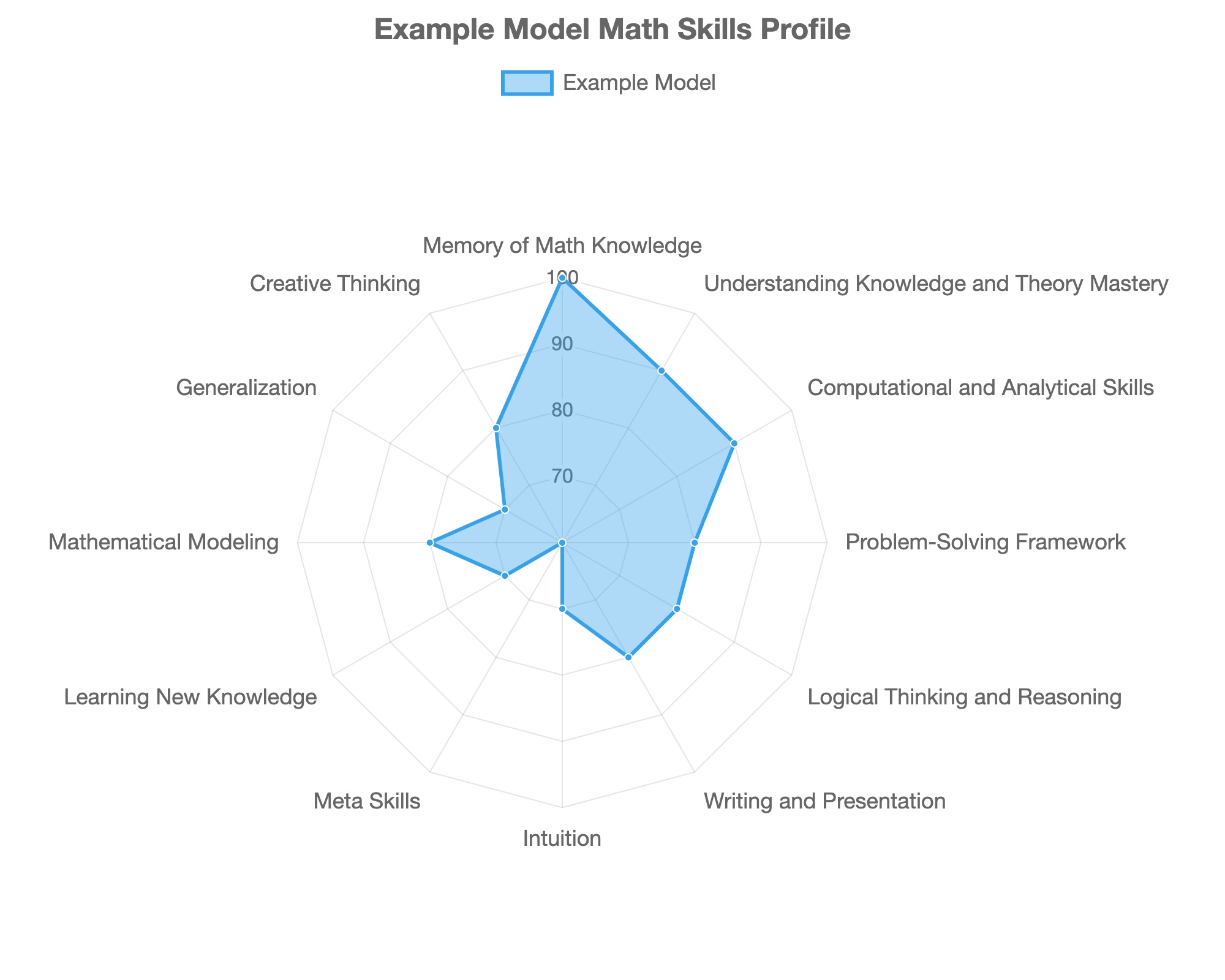}
    \caption{Example LLM Math Skills Radar Chart}
    \label{fig:radar-chart}
\end{figure}

\textbf{Auto-evaluation infrastructure} also deserves careful attention. While some benchmark modules—such as computational fluency, symbolic computation, and logic verification—can be reliably auto-graded, other modules require more nuanced assessment strategies. Tasks that involve constructing proofs, generating multiple solution paths, or synthesizing definitions from context may require hybrid pipelines that integrate rubric-based scoring, LLM-assisted grading heuristics, and human-in-the-loop review. Building such pipelines in a modular and interpretable way is key to scaling evaluation while maintaining reliability. Additionally, publishing these pipelines can promote reproducibility and encourage community benchmarking beyond academic labs and foundation model developers.

Some subjective problems are inherently difficult to grade—especially those involving mathematical exposition, abstraction, or creative reasoning. To address this, we will experiment with prompt-chained LLM grading, where one model proposes a solution and another provides a structured critique based on rubric-aligned criteria. These auto-graders will be calibrated against human-graded gold standards to measure alignment and detect blind spots. In addition, for high-stakes or ambiguous tasks, we will design interfaces for expert-in-the-loop review, allowing humans to overrule, annotate, or co-grade with the model. By integrating both automatic and human judgment in a layered fashion, we aim to balance scalability with rigor and ensure that evaluation remains faithful to the underlying reasoning quality, not just surface correctness.

\textbf{Crowd-sourcing problems} can further enrich the benchmark by diversifying both content and skill coverage. Through contributions from mathematically skilled volunteers, we can source novel, high-quality problems spanning a wide range of difficulty levels and contexts. To preserve diagnostic value, each contributed problem will be annotated with both its \emph{skill category} (e.g., symbolic manipulation, logical reasoning, creative problem solving) and its \emph{subject taxonomy}—a hierarchical classification of the mathematical domain (e.g., algebra, geometry, combinatorics). This dual tagging enables fine-grained analysis across the skill–subject matrix and facilitates the identification of underrepresented areas.

\textbf{Building an extensible platform} will enable researchers to easily plug in new tasks, rubrics, and grading modules without modifying the core infrastructure. By open-sourcing the framework and encouraging community contributions, we aim to create a flexible, modular ecosystem that supports ongoing innovation in benchmark design and evaluation methodology. This will help establish shared standards for assessing mathematical capabilities in LLMs that are acceptable by both the mathematics and technology communities.

Ultimately, our goal is to bridge the gap between memory-level fluency and genuine mathematical understanding and insight. We hope that by embedding this skill-centric perspective into model training, evaluation, and deployment, we can push toward systems that do not merely imitate solutions, but engage in structured, interpretable, and flexible reasoning that reflects a deeper grasp of mathematics. This shift, from solving problems to understanding them, marks a necessary evolution in the role of LLMs in mathematical reasoning and scientific discovery. Viewed in this light, GAUSS represents not only a benchmark, but also a cornerstone in the roadmap toward achieving general artificial intelligence in mathematics and science.

\newpage

\bibliographystyle{plainnat} 
\bibliography{references}

\newpage
\appendix
\section{Evaluation Pipeline}
We present a structured evaluation framework for benchmarking the mathematical reasoning ability of large language models (LLMs), with a focus on \texttt{GPT-5-thinking}. Each task is processed through a well-defined pipeline that records the problem, generates a model response, compares it with a standard solution, and evaluates the result based on a rubric.

We use the following environments to organize the content:

\begin{enumerate}
    \item \textbf{Problem Statement} — \texttt{problembox} \textcolor{cyan!40!black}{(Blue box)}  
    Each task begins with a well-defined mathematical problem enclosed in a blue-framed \texttt{problembox}.

    \textbf{Usage:}
    \begin{verbatim}
\begin{problembox}{Problem Title}
Problem description goes here.
\end{problembox}
    \end{verbatim}
\begin{problembox}{Problem Title}
    Problem description goes here.
\end{problembox}

    \item \textbf{Model Response (GPT-5-thinking)} — \texttt{GPT-5-thinking-response} environment  
    The answer generated by \texttt{o4-mini-high} is typeset in a proof-style environment to simulate a student's written work. It avoids colored boxes for a cleaner academic appearance.

    \textbf{Usage:}
    \begin{verbatim}
\begin{GPT-5-thinking-response}
Model-generated response goes here.
\end{GPT-5-thinking-response}
    \end{verbatim}

\begin{GPT-5-thinking-response}
Model-generated response goes here.
\end{GPT-5-thinking-response}
    \item \textbf{Standard Solution} — \texttt{solution} environment  
    The correct or ideal solution is presented using a similar proof-style layout for comparison with the model output.

    \textbf{Usage:}
    \begin{verbatim}
\begin{solution}
Standard solution goes here.
\end{solution}
    \end{verbatim}
\begin{solution}
Standard solution goes here.
\end{solution}
    \item \textbf{Grading Rubric} — \texttt{rubricbox} \textcolor{gray!60!black}{(Gray box)}  
    The evaluation rubric is provided in a blue-framed \texttt{rubricbox}, outlining the criteria for full and partial credit.

    \textbf{Usage:}
    \begin{verbatim}
\begin{rubricbox}
1. States and proves the correct minimum value — 2 point. \\ 
2. Gives a correct construction — 2 point. 
\end{rubricbox}
    \end{verbatim}

\begin{rubricbox}
1.~States and proves the correct lower bound — 2 point. \\ 
2.~Gives a correct construction — 2 point. 
\end{rubricbox}

    \item \textbf{Scoring and Feedback}
    \begin{itemize}
        \item \texttt{scorebox} \textcolor{teal!60!black}{(Teal box)} — Records the numerical score assigned.  
        \item \texttt{evalbox} \textcolor{red!70!black}{(Red box)} — Contains qualitative evaluator feedback.
    \end{itemize}

    \textbf{Usage:}
    \begin{verbatim}
\begin{scorebox}
Score: 0 / 3
\end{scorebox}

\begin{evalbox}
The model correctly applied the s-cobordism theorem...
\end{evalbox}
    \end{verbatim}

\begin{scorebox}
Score: 0 / 3
\end{scorebox}

\begin{evalbox}
The model correctly applied the s-cobordism theorem...
\end{evalbox}
\end{enumerate}

\noindent
All box environments use \texttt{\footnotesize} font and 90\% width for clean layout. The proof-style environments for responses and solutions preserve a natural mathematical flow without decorative frames.

\begin{center}
\textbf{Problem → GPT-5-thinking Response → Solution → Rubric →  Score + Evaluation.  }
\end{center}

This pipeline enables consistent analysis and evaluation of LLM-generated solutions.

\newpage
\section{Model Responses and Evaluations}
\addtocontents{toc}{\protect\setcounter{tocdepth}{1}}
\subsection{Memory of Math Knowledge}
\subsubsection{Taxonomy of Subjects}
\ProblemSection{\green{Problem 1a1}}

\begin{problembox}[Proposed by Qiuyu Ren]
Floer homology is a powerful tool for studying symplectic geometry and low-dimensional topology. 
\begin{itemize}
    \item Please name three distinct Floer homology theories that play important roles in low-dimensional topology.
    \item  For each theory, provide an example of a significant topological application. Your applications should not mention explicitly the Floer homology theory in the statement.
\end{itemize}

\pass
\end{problembox}
\begin{GPT-5-thinking-response}
\textbf{Instanton Floer homology.}  
\emph{Application:} For any nontrivial knot \(K\subset S^3\), every non-meridional Dehn surgery on \(K\) produces a 3-manifold with nontrivial fundamental group (i.e., no such surgery yields a homotopy 3-sphere). (Kronheimer--Mrowka, ``Property P for knots,'' 2004.)

\medskip

\textbf{Monopole (Seiberg--Witten) Floer homology.}  
\emph{Application:} There exist topological manifolds in every dimension \(n\ge 5\) that admit no simplicial triangulation. (Manolescu’s disproof of the high-dimensional triangulation conjecture via Pin(2)-equivariant methods, 2013.)

\medskip

\textbf{Heegaard Floer homology.}  
\emph{Application:} If a Dehn surgery on a knot in \(S^3\) yields a lens space, then the knot is fibered. (Ozsváth--Szabó framework with results of Ghiggini and Ni, 2007--2008.)

\end{GPT-5-thinking-response}

\begin{solution}
\textbf{Heegaard Floer homology}. Application: branched double covers of alternating knots in $S^3$ do not admit co-orientable taut foliations.

\textbf{Instanton Floer homology}. Application: Khovanov homology detects the unknot.

\textbf{Monopole Floer homology}. Application: Disproof of the triangulation conjecture.
\end{solution}

\begin{rubricbox}
1 point for naming three Floer theories fundamental in low-dimensional topology.\\
2 points for giving correct topological applications.
\end{rubricbox}

\begin{scorebox}
3/3
\end{scorebox}

\begin{evalbox}
Correct theories and applications.
\end{evalbox}

\ProblemSection{\green{Problm 1a2}}
\begin{problembox}[Proposed by Jiaxin Zhang]

Liouville quantum gravity (LQG) is a central topic in modern probability theory. List several of its subfields and, for each subfield:
\begin{enumerate}
  \item[i.] Name a central concept or structure;
  \item[ii.] State one foundational result from that subfield.
\end{enumerate}
  \pass
\end{problembox}
\begin{GPT-5-thinking-response}
Great question. Here’s a compact map of major LQG subfields—each with a central object and one canonical theorem/result.

\begin{enumerate}
  \item \textbf{Gaussian multiplicative chaos (GMC) \& the LQG measure}
    \begin{enumerate}
      \item[i.] \emph{Central concept:} The \(\gamma\)-Liouville area measure \( \mu_h = e^{\gamma h}\,dx\,dy \) defined from a Gaussian Free Field via GMC.
      \item[ii.] \emph{Foundational result:} The KPZ relation rigorously linking Euclidean and quantum scaling exponents, and the construction/coordinate-change rule for the LQG measure. 
    \end{enumerate}

  \item \textbf{Quantum surfaces (disks, spheres, wedges, cones)}
    \begin{enumerate}
      \item[i.] \emph{Central concept:} Equivalence classes of \((D,h)\) modulo additive constants and conformal maps, with canonical families: quantum disk/sphere/wedge/cone.
      \item[ii.] \emph{Foundational result:} Duplantier–Miller–Sheffield (DMS) construct these surfaces and their operations (e.g., welding, Poissonian decompositions) in the \(\gamma\in(0,2)\) regime.
    \end{enumerate}

  \item \textbf{SLE/LQG coupling (welding \& the quantum zipper)}
    \begin{enumerate}
      \item[i.] \emph{Central concept:} Conformal welding of quantum boundaries; interfaces are SLE with \(\kappa=\gamma^2\) (and \(\kappa'=16/\gamma^2\) for space-filling variants).
      \item[ii.] \emph{Foundational result:} Sheffield’s “quantum gravity zipper’’: welding two LQG surfaces along boundary lengths produces an SLE interface, with precise conformal-covariance and symmetry properties. 
    \end{enumerate}

  \item \textbf{Mating-of-trees / peanosphere}
    \begin{enumerate}
      \item[i.] \emph{Central concept:} Encode an LQG surface decorated by a space-filling SLE as a correlated two-dimensional Brownian motion (“CRT–CRT’’ mating).
      \item[ii.] \emph{Foundational result:} The DMS mating-of-trees theorem gives a canonical Brownian encoding of \(\gamma\)-LQG with SLE\(_{\kappa'}\) decoration. 
    \end{enumerate}

  \item \textbf{Imaginary geometry}
    \begin{enumerate}
      \item[i.] \emph{Central concept:} GFF flow lines \(e^{i(h/\chi+\theta)}\) realizing SLE\(_\kappa\) curves and their interactions.
      \item[ii.] \emph{Foundational result:} Existence/uniqueness of flow lines and identification with SLE\(_\kappa\) (and counterflow with SLE\(_{16/\kappa}\)); reversibility and interaction theorems. 
    \end{enumerate}

  \item \textbf{Metric structure \& QLE}
    \begin{enumerate}
      \item[i.] \emph{Central concept:} The LQG metric \(D_h\) (distance) obtained via Quantum Loewner Evolution and Liouville first-passage percolation limits.
      \item[ii.] \emph{Foundational result:} Existence and uniqueness of the \(\gamma\)-LQG metric for all \(\gamma\in(0,2)\); at \(\gamma=\sqrt{8/3}\), QLE constructs the metric and matches Brownian map geometry.
    \end{enumerate}

  \item \textbf{Random planar maps \(\longleftrightarrow\) LQG}
    \begin{enumerate}
      \item[i.] \emph{Central concept:} Scaling limits of uniform random maps (quadrangulations, etc.) as metric-measure spheres.
      \item[ii.] \emph{Foundational result:} The Brownian map is the universal scaling limit of large random quadrangulations (and related classes); subsequently identified with \(\sqrt{8/3}\)-LQG. 
    \end{enumerate}

  \item \textbf{Liouville CFT \& structure constants}
    \begin{enumerate}
      \item[i.] \emph{Central concept:} Probabilistic Liouville conformal field theory built from the GFF; correlation functions and structure constants.
      \item[ii.] \emph{Foundational result:} Probabilistic proof of the DOZZ formula for LCFT three-point structure constants, via GMC.
    \end{enumerate}

  \item \textbf{Diffusions on LQG (Liouville Brownian motion)}
    \begin{enumerate}
      \item[i.] \emph{Central concept:} Brownian motion associated to the random Riemannian metric \(e^{\gamma h}(dx^2+dy^2)\).
      \item[ii.] \emph{Foundational result:} Construction of Liouville Brownian motion as a Feller diffusion with \(\mu_h\) as invariant measure for all \(\gamma<2\). 
    \end{enumerate}
\end{enumerate}

If you want, I can turn this into a one-page LaTeX handout with short references and the $(\gamma,\kappa)$ dictionary summarized.

\end{GPT-5-thinking-response}
\begin{solution}
\begin{enumerate}
\item \textbf{GMC \& the LQG measure}
\begin{enumerate}
\item[i.] \emph{Concept:} The $\gamma$-LQG area measure $\mu_h := e^{\gamma h}\,dx\,dy$ from a GFF via Gaussian multiplicative chaos.
\item[ii.] \emph{Theorem:}the LQG measure is well-defined and obeys the conformal coordinate-change rule for $\gamma\in(0,2)$.
\end{enumerate}

\item \textbf{Quantum surfaces \& SLE welding}
\begin{enumerate}
\item[i.] \emph{Concept:} Quantum disks/spheres/wedges/cones as equivalence classes of $(D,h)$ modulo additive constants and conformal maps; boundary length comes from LQG.
\item[ii.] \emph{Theorem:} (Quantum zipper / conformal welding) Welding LQG boundaries by quantum length produces SLE interfaces with $\kappa=\gamma^2$ (and space-filling $\kappa'=16/\gamma^2$), with precise conformal-covariance and symmetry.
\end{enumerate}

\item \textbf{Mating-of-trees / peanosphere}
\begin{enumerate}
\item[i.] \emph{Concept:} Encode an LQG surface decorated by a space-filling SLE as a correlated 2D Brownian motion (a “CRT–CRT” mating).
\item[ii.] \emph{Theorem:} (DMS mating-of-trees) $\gamma$-LQG with SLE$_{\kappa'}$ decoration has a canonical Brownian encoding that determines the surface and curve, and conversely.
\end{enumerate}

\item \textbf{Metric structure (QLE) \& Brownian map}
\begin{enumerate}
\item[i.] \emph{Concept:} The LQG metric $D_h$ built via Quantum Loewner Evolution / Liouville FPP, yielding a random metric-measure surface.
\item[ii.] \emph{Theorem:} (Existence/uniqueness \& $\sqrt{8/3}$ identification) The $\gamma$-LQG metric exists and is unique for $\gamma\in(0,2)$; at $\gamma=\sqrt{8/3}$, the resulting metric-measure sphere agrees with the Brownian map.
\end{enumerate}
\end{enumerate}

\end{solution}
\begin{rubricbox}

Got 3/3 if it correctly lists at least three points.

\end{rubricbox}
\begin{scorebox}
    3/3
\end{scorebox}
\begin{evalbox}
    It correctly identifies at least 3 major SLE subfields, each with a central concept and a foundational result.
\end{evalbox}

\subsubsection{Key Mathematical Statements}
\ProblemSection{\red{Problem 1b1 }}
\begin{problembox}[Proposed by Jiaxin Zhang]
 Schramm–Loewner Evolution (SLE) is a vibrant field in modern probability theory. Please address the following two parts:
\begin{enumerate}
  \item 
    Provide the precise definition of SLE, emphasizing its construction via the Loewner equation driven by a random function \(W_t = \sqrt{\kappa}B_t\) where $B_t$ is the standard Brownian motion.
      \item  Please state the transformation law of the driving functions of Loewner chain under conformal Maps. 
    
      \fail
\end{enumerate}    
\end{problembox}
\begin{GPT-5-thinking-response}
Here’s a clean, standard formulation that matches the way SLE is defined and used in the literature.

1) Definition of $\mathrm{SLE}_\kappa$ via Loewner’s equation

Chordal case (from $0$ to $\infty$ in $\mathbb{H}$)

Let $W:[0,\infty)\to\mathbb{R}$ be a continuous function (the *driving function*). For each $z\in\overline{\mathbb{H}}$, define $g_t(z)$ as the maximal solution (in $t$) to the Loewner ODE

$$
\partial_t g_t(z)=\frac{2}{g_t(z)-W_t},\qquad g_0(z)=z,
$$

with swallowing time

$$
\tau(z):=\sup\{\,t\ge 0:\ \min_{s\in[0,t]}|g_s(z)-W_s|>0\,\}.
$$

Set the *hulls* $K_t:=\{z\in\overline{\mathbb{H}}:\ \tau(z)\le t\}$ and note that

$$
g_t:\mathbb{H}\setminus K_t\longrightarrow\mathbb{H}
$$

is conformal and satisfies the hydrodynamic normalization

$$
g_t(z)=z+\frac{2t}{z}+O\!\left(\frac{1}{z^2}\right)\quad\text{as }z\to\infty.
$$

(Equivalently, $\mathrm{hcap}(K_t)=2t$.)

**Chordal $\mathrm{SLE}_\kappa$** is the random Loewner chain $(K_t)_{t\ge 0}$ (equivalently, $(g_t)_{t\ge 0}$) obtained by taking the driving function to be

$$
W_t=\sqrt{\kappa}\,B_t,\qquad t\ge 0,
$$

where $(B_t)_{t\ge 0}$ is standard one-dimensional Brownian motion with $B_0=0$.
We then say “$(K_t)$ is chordal $\mathrm{SLE}_\kappa$ in $\mathbb{H}$ from $0$ to $\infty$.”

*(Remarks: One often also speaks of the random curve $\gamma$ — the Loewner trace — whose filled-in hull is $K_t$; this exists for all $\kappa$ and is simple for $\kappa\le 4$, etc., but the definition above is in terms of the Loewner chain.)*

Radial case (from $1$ to $0$ in $\mathbb{D}$)

Let $W:[0,\infty)\to\mathbb{R}$ be continuous and define $g_t$ by the *radial* Loewner ODE

$$
\partial_t g_t(z)\;=\;-\,g_t(z)\,\frac{e^{iW_t}+g_t(z)}{e^{iW_t}-g_t(z)},\qquad g_0(z)=z,
$$

with normalization $g_t(0)=0$ and $g_t'(0)=e^{t}$. For $W_t=\sqrt{\kappa}\,B_t$ this chain is **radial $\mathrm{SLE}_\kappa$** in $\mathbb{D}$ from $1$ to $0$.

---

2) How driving functions transform under conformal maps

The transformation is most cleanly stated in *conformal coordinates adapted at each time*. Below is the standard chordal statement; the radial version is analogous (replace $\mathbb{H},\infty$ by the appropriate simply connected domain and interior target, and half-plane capacity by the relevant capacity).

General conformal covariance (chordal)

Let $(g_t)$ be a chordal Loewner chain in $\mathbb{H}$ driven by a continuous $W_t$, with hulls $K_t$. Let $\phi:\mathbb{H}\to\mathbb{H}$ be conformal, and let $\tilde K$ be the image chain $\tilde K_{u(t)}:=\phi(K_t)$. Denote by $\tilde g_u:\mathbb{H}\setminus \tilde K_u\to\mathbb{H}$ the hydrodynamically normalized maps of the image chain. Define the *time-dependent conjugating map*

$$
\phi_t\ :=\ \tilde g_{u(t)}\circ \phi \circ g_t^{-1}:\ \mathbb{H}\longrightarrow\mathbb{H}.
$$

Then the image chain $(\tilde g_u)$ is again a chordal Loewner chain, driven by $\tilde W_u$ related to $W_t$ by

$$
\boxed{\quad \tilde W_{\,u(t)}\ =\ \phi_t\big(W_t\big),\qquad
\frac{du}{dt}\ =\ \big(\phi_t'(W_t)\big)^{\!2}. \quad}
$$

In words: *compose the instantaneous driving point with the current conformal change of coordinates $\phi_t$, and reparametrize time by the square of the derivative at the driving point* (reflecting the conformal covariance of half-plane capacity).

Affine automorphisms as a special case

If $\phi(z)=a z + b$ with $a>0$, $b\in\mathbb{R}$ (so $\phi$ preserves $\infty$ and the hydrodynamic normalization), then $\phi_t\equiv\phi$ and the rule simplifies to

$$
\boxed{\quad \tilde W_t\ =\ a^{-1}\big(W_{a^2 t}-b\big). \quad}
$$

This encodes translation invariance ($b$) and scale covariance ($a$) of the chordal Loewner equation.

Consequence for SLE

If $W_t=\sqrt{\kappa}B_t$ (chordal $\mathrm{SLE}_\kappa$ in $\mathbb{H}$ from $0$ to $\infty$), then for any conformal $\phi$ between simply connected domains sending the start/target boundary points appropriately, the image chain $\phi(K_t)$, reparametrized by $u$ as above, is again $\mathrm{SLE}_\kappa$, i.e.

$$
\tilde W_u\ \stackrel{d}{=}\ \sqrt{\kappa}\,\tilde B_u,
$$

with $(\tilde B_u)$ a standard Brownian motion (now in the new capacity time $u$). This is the conformal covariance of SLE.
    
\end{GPT-5-thinking-response}

\begin{solution}
The chordal Schramm-Loewner Evolution (SLE ${ }_\kappa$ ) is defined as follows:
- Consider the upper half-plane $\mathbb{H}=\{z \in \mathbb{C}: \Im z>0\}$.
- For a continuous real-valued function $W:[0, \infty) \rightarrow \mathbb{R}$, the Loewner chain $\left\{g_t(z)\right\}_{t \geq 0}$ is given by the solution to the Loewner ODE

$$
\frac{d}{d t} g_t(z)=\frac{2}{g_t(z)-W_t}, \quad g_0(z)=z, \quad z \in \mathbb{H} .
$$

- For each $t$, the solution $g_t$ is a conformal map from a simply connected domain $H_t \subseteq \mathbb{H}$ onto $\mathbb{H}$, with hydrodynamic normalization

$$
g_t(z)=z+\frac{2 t}{z}+o\left(\frac{1}{z}\right), \quad z \rightarrow \infty .
$$

Now let $\left\{B_t\right\}_{t \geq 0}$ be a standard one-dimensional Brownian motion. For a parameter $\kappa \geq 0$, define the driving function

$$
W_t=\sqrt{\kappa} B_t .
$$

The random family of conformal maps $\left\{g_t\right\}$ generated by this Brownian-driven Loewner equation is called the chordal Schramm-Loewner Evolution with parameter $\kappa$, denoted $\mathrm{SLE}_\kappa$. The associated random growing compact sets $K_t=\mathbb{H} \backslash H_t$ form the SLE hulls (with boundary curves known as the SLE trace when well-defined).

Suppose the driving function $W_t \in \mathbb{R}$ evolves according to the stochastic differential equation

$$
d W_t=\sqrt{\kappa} d B_t+b\left(W_t ; \Psi_t\left(W_1\right), \ldots, \Psi_t\left(W_n\right)\right) d t,
$$

where $B_t$ is standard Brownian motion, and $b$ is a drift term depending on the conformal images of marked points under the time-dependent conformal map

$$
\Psi_t:=\widetilde{g}_t \circ \Psi \circ g_t^{-1},
$$

with $\Psi$ conformal near $W_0 \in \mathbb{R}$ and $g_t, \widetilde{g}_t$ the Loewner flows before and after coordinate change, respectively.

Define the transformed driving function $\widetilde{W}_t:=\Psi_t\left(W_t\right)$, and introduce the reparameterized time

$$
s(t):=\int_0^t\left|\Psi_u^{\prime}\left(W_u\right)\right|^2 d u
$$

Then the time-changed process $\widetilde{W}_s:=\widetilde{W}_{t(s)}$ satisfies the stochastic differential equation

$$
d \widetilde{W}_s=\sqrt{\kappa} d B_s+\frac{b\left(W_s ; \Psi_{t(s)}\left(W_1\right), \ldots, \Psi_{t(s)}\left(W_n\right)\right)}{\Psi_{t(s)}^{\prime}\left(W_s\right)} d s+\frac{\kappa-6}{2} \cdot \frac{\Psi_{t(s)}^{\prime \prime}\left(W_s\right)}{\left[\Psi_{t(s)}^{\prime}\left(W_s\right)\right]^2} d s
$$
\end{solution}
\begin{rubricbox}
1.Correct definition of SLE.  

2.Correct transformation law of the driving function (drift term).  

\medskip
\noindent\textit{Reference:}  
\href{https://arxiv.org/abs/2505.16093}{\emph{Multiple chordal SLE($\kappa$) and quantum Calogero–Moser system}}~\citep{zhang2025multiplechordalsle}.
\end{rubricbox}
\begin{scorebox}
1/2
\end{scorebox}

\begin{evalbox}
The response did not provide the explicit form of the drift term after the conformal transformation.
\end{evalbox}

\ProblemSection{\red{Problem 1b2}}
\begin{problembox}[Proposed by Jiaxin Zhang]
Please explain the commutation relations for multiple chordal SLE($\kappa$).

\fail
\end{problembox}

\begin{GPT-5-thinking-response}
\textbf{Set-up.}
Let $x_1<\cdots<x_n$ be marked boundary points on $\mathbb{R}$ and write
\[
\mathcal A_i \;=\; \frac{\kappa}{2}\,\partial_{x_i}^2 \;+\; \sum_{k\neq i}\frac{2}{x_k-x_i}\,\partial_{x_k}
\qquad (i=1,\dots,n),
\]
the infinitesimal generator corresponding to growing an infinitesimal chordal Loewner slit at $x_i$ (half-plane capacity $2\,dt$).

A (positive, smooth) function $Z(x_1,\dots,x_n)$ is called a \emph{partition function} if it is Möbius covariant with boundary weight
\[
h=\frac{6-\kappa}{2\kappa}
\qquad\text{and solves, for each }i,
\]
the level-two null-state (BPZ) PDE
\[
\boxed{\quad
\frac{\kappa}{2}\,\partial_{x_i}^2 Z
\;+\;\sum_{k\neq i}\frac{2}{x_k-x_i}\,\partial_{x_k}Z
\;+\;\sum_{k\neq i}\frac{2h}{(x_k-x_i)^2}\,Z
\;=\;0\quad}
\tag{$\star_i$}
\]

\textbf{Commutation relations (Dub\'edat).}
Assume $Z$ satisfies equation * and Möbius covariance. Define the (Doob–Girsanov tilted) operators
\[
\mathcal L_i f
\;:=\; Z^{-1}\,\mathcal A_i\,(Zf)
\;=\; \mathcal A_i f \;+\; \kappa\,(\partial_{x_i}\log Z)\,\partial_{x_i} f
\qquad\text{(using equation * to cancel the zero-order term).}
\]
Then for all $i\neq j$,
\[
\boxed{\ [\mathcal L_i,\mathcal L_j]\,f \;=\; 0\quad\text{for all smooth }f, \ }
\]
i.e.\ the infinitesimal evolutions \emph{commute}. Equivalently, if one first grows the $i$th slit for an infinitesimal time and then the $j$th (or in the reverse order), the resulting law of the configuration is the same (up to deterministic time-change). This is the precise meaning of \emph{commutation relations} for multiple chordal SLE$_\kappa$.

\medskip
\textbf{Driving SDE formulation (equivalent).}
Under the multiple-SLE measure determined by $Z$, the driving points solve
\[
dX_t^i \;=\; \sqrt{\kappa}\,dB_t^i \;+\; \Bigg(\kappa\,\partial_{x_i}\log Z(X_t)
\;+\;\sum_{k\neq i}\frac{2}{X_t^i-X_t^k}\Bigg)dt,
\]
and the corresponding Markov generators are the commuting $\{\mathcal L_i\}_{i=1}^n$ above. The commutation expresses that the joint law does not depend on the order in which the curves are explored.

\end{GPT-5-thinking-response}

\begin{solution}
Please refer to the paper \href{https://arxiv.org/abs/2505.16093}{ Multiple chordal SLE($\kappa$) and quantum Calogero-Moser system }

\textbf{Partition functions and null vector equations.}  
Let $x_1,\dots,x_n\in\mathbb R$ be distinct boundary points.  
A smooth positive function $Z(x_1,\dots,x_n)$ is called a \emph{partition function} if it is Möbius covariant with boundary weight
\[
h=\frac{6-\kappa}{2\kappa},
\]
and for each $i=1,\dots,n$ satisfies the level-two null vector (BPZ) equation
\[
\frac{\kappa}{2}\,\partial_{x_i x_i} Z
+\sum_{k\neq i}\frac{2}{x_k-x_i}\,\partial_{x_k}Z
+\sum_{k\neq i}\frac{2h}{(x_k-x_i)^2}\,Z = 0.
\tag{$\star_i$}
\]

\medskip
\textbf{Commutation relations.}  

Given such a partition function $Z$, define for each $i$
\[
\mathcal L_i
= \frac{\kappa}{2}\,\partial_{x_i x_i}
+ \kappa\,(\partial_{x_i}\log Z)\,\partial_{x_i}
+ \sum_{k\neq i}\frac{2}{x_k-x_i}\,\partial_{x_k}.
\]

The operators $\mathcal L_i$ describe the infinitesimal evolution when a slit is grown at $x_i$.  

Probabilistically, commutation relations say that if one first grows the slit at $x_i$ for an infinitesimal time and then at $x_j$, or reverses the order, the resulting law of the configuration is the same (up to reparametrization).

Dub\'edat’s commutation theorem asserts that for all $i\neq j$,
\[
[\mathcal L_i,\mathcal L_j]
= \frac{4}{(x_j-x_i)^2}\,\bigl(\mathcal L_j - \mathcal L_i\bigr).
\]

\medskip
\textbf{Driving SDE formulation.}  
Under the probability measure determined by $Z$, the driving processes $(X_t^1,\dots,X_t^n)$ satisfy
\[
dX_t^i \;=\; \sqrt{\kappa}\,dB_t^i
+ \left(\kappa\,\partial_{x_i}\log Z(X_t)
+ \sum_{k\neq i}\frac{2}{X_t^i - X_t^k}\right)\,dt,
\qquad i=1,\dots,n,
\]
where $(B_t^i)_{i=1}^n$ are independent standard Brownian motions.  
The infinitesimal generator of the $i$-th coordinate is precisely $\mathcal L_i$, and the commutation relation shows that the joint law is independent of the order in which the curves are explored.

\end{solution}
\begin{rubricbox}
\begin{enumerate}
  \item Explain the commutation relation from the order of growth of two points.  
  \item Computes the commutator of the generators correctly.  
  \item States the null-vector equations accurately.  
\end{enumerate}
\end{rubricbox}

\begin{scorebox}
2/3
\end{scorebox}

\begin{evalbox}
The response didn't state the correct commutation relations of the generators.

$$\left[\mathcal{L}_i, \mathcal{L}_j\right]=\frac{4}{\left(x_i-x_j\right)^2}\left(\mathcal{L}_j-\mathcal{L}_i\right)
$$
\end{evalbox}

\subsection{Understanding Knowledge and Theory Mastery}

\subsubsection{Understanding theorems}  

\ProblemSection{\red{Problem 2a1}}

\begin{problembox}[Proposed by Jiaxin Zhang]
Derive the null vector equations from commutation relations of multiple chordal SLE($\kappa$)  

\fail
\end{problembox}

\begin{GPT-5-thinking-response}
\textbf{Set-up.}
Let $x_1<\cdots<x_n$ be marked boundary points on $\R$, and write the ``bare'' Loewner operators
\[
\mathcal A_i \;=\; \frac{\kappa}{2}\,\partial_{x_i}^2 \;+\; \sum_{k\neq i}\frac{2}{x_k-x_i}\,\partial_{x_k}\qquad (i=1,\dots,n).
\]
Let $Z=Z(x_1,\dots,x_n)>0$ be a smooth, Möbius--covariant partition function with boundary weight $h$, i.e.
\[
Z(\phi(x_1),\dots,\phi(x_n)) \;=\; \prod_{m=1}^n |\phi'(x_m)|^{h}\, Z(x_1,\dots,x_n)\quad\text{for all }\phi\in\mathrm{PSL}_2(\R).
\]
Tilt the dynamics by $Z$ (Doob transform) and define
\[
\mathcal L_i f \;:=\; Z^{-1}\,\mathcal A_i\,(Zf).
\]
A direct expansion gives
\[
\mathcal L_i \;=\; \frac{\kappa}{2}\,\partial_{x_i}^2 \;+\; \sum_{k\neq i}\frac{2}{x_k-x_i}\,\partial_{x_k}
\;+\;\kappa\,(\partial_{x_i}\log Z)\,\partial_{x_i} \;+\; V_i(x)\,,
\]
where the zero-th order term is
\[
V_i(x)\;=\;\frac{\kappa}{2}\,\partial_{x_i}^2\log Z \;+\; \sum_{k\neq i}\frac{2}{x_k-x_i}\,\partial_{x_k}\log Z \;+\; \frac{\kappa}{2}\big(\partial_{x_i}\log Z\big)^2.
\]

\medskip
\textbf{Commutation relations.}
The multiple-SLE commutation principle (Dub\'edat) is that \emph{the infinitesimal evolutions commute}:
\[
[\mathcal L_i,\mathcal L_j]\;=\;0\qquad\text{for all }i\neq j.
\]
One computes the bare commutator
\[
[\mathcal A_i,\mathcal A_j]\;=\;\frac{4}{(x_j-x_i)^2}\,\big(\partial_{x_j}-\partial_{x_i}\big),
\]
so the extra terms in $\mathcal L_i,\mathcal L_j$ must precisely cancel this vector field singularity. Tracking the coefficient of $(x_j-x_i)^{-2}$ in $[\mathcal L_i,\mathcal L_j]$ yields
\[
\partial_{x_j} V_i - \partial_{x_i} V_j \;=\; -\,\frac{4h}{(x_j-x_i)^3}\,,
\]
which integrates (using permutation symmetry and Möbius covariance/Ward identities) to
\[
V_i(x)\;=\;-\sum_{k\neq i}\frac{2h}{(x_k-x_i)^2}\,.
\]
Equivalently, the vanishing of the zero-th order term in the conjugated family
\[
\boxed{\ \ \mathcal T_i \;:=\; Z\,\mathcal L_i\,Z^{-1}
\;=\;\frac{\kappa}{2}\,\partial_{x_i}^2 \;+\; \sum_{k\neq i}\frac{2}{x_k-x_i}\,\partial_{x_k}
\;-\;\sum_{k\neq i}\frac{2h}{(x_k-x_i)^2}\ \ }
\]
is \emph{exactly} the condition that $Z$ solves the level-two BPZ/null-vector equations:
\[
\boxed{\quad
\frac{\kappa}{2}\,\partial_{x_i}^2 Z
\;+\;\sum_{k\neq i}\frac{2}{x_k-x_i}\,\partial_{x_k}Z
\;-\;\sum_{k\neq i}\frac{2h}{(x_k-x_i)^2}\,Z \;=\;0\,,\qquad i=1,\dots,n. \quad}
\]

\medskip
\textbf{Fixing the weight $h$.}
The Möbius covariance of $Z$ (global conformal Ward identities)
\[
\sum_m \partial_{x_m}Z=0,\qquad
\sum_m \big(x_m\partial_{x_m}+2h\big)Z=0,\qquad
\sum_m \big(x_m^2\partial_{x_m}+2h\,x_m\big)Z=0
\]
is compatible with the commuting family $\{\mathcal T_i\}$ only when
\[
\boxed{\quad h=\frac{6-\kappa}{2\kappa}\, . \quad}
\]
With this choice, the operators $\{\mathcal T_i\}$ \emph{commute}, and the Doob-transformed generators $\{\mathcal L_i\}$ commute as required.

\medskip
\textbf{Conclusion.}
Thus, the requirement that the infinitesimal multiple-SLE evolutions commute, together with Möbius covariance, forces the partition function $Z$ to satisfy the \emph{level-two null-vector/BPZ equations} above (with $h=(6-\kappa)/(2\kappa)$). Conversely, any positive Möbius-covariant $Z$ solving these PDEs yields a commuting family and hence a well-defined multiple chordal SLE$_\kappa$ law.

\end{GPT-5-thinking-response}

\begin{solution}
    Please refer to the paper \href{Commutation relations for SLE}{https://arxiv.org/abs/math/0411299}

We have seen natural examples where two SLEs could be grown in a common domain in a consistent fashion. In this section, discuss necessary infinitesimal conditions. We shall define a "global" commutation condition, of geometric nature, and express its consequence in terms of infinitesimal generators, which is of algebraic nature.

Let us consider the following chordal situation: the domain is $\mathbb{H}$, SLEs aim at $\infty$, and $(x, y, z_1, \ldots, z_n)$  are $(n+2)$ (distinct) points on the real line; the point at infinity is also a marked point. We want to grow two infinitesimal hulls (with capacity of order $\varepsilon$ ) at $x$ and $z$ respectively. We can either grow a hull $K_{\varepsilon}$ at $x$, and then another one at $y$ in the pertubed domain $\mathbb{H} \backslash K_{\varepsilon}$, or proceed in the other order. The coherence condition is that these two procedures yield the same result.

Let us make things more rigorous. Consider a Loewner chain $\left(K_{s, t}\right)_{(s, t) \in \mathcal{T}}$ with a double time index, so that $K_{s, t} \subset K_{s^{\prime}, t^{\prime}}$ if $s^{\prime} \geq s, t^{\prime} \geq t$ and $K_{s, t} \neq K_{s^{\prime}, t^{\prime}}$ if $\left(s^{\prime}, t^{\prime}\right) \neq(s, t)$. We only consider chains up to time reparameterization $\mathbb{R}_{+}^2 \rightarrow \mathbb{R}_{+}^2$. We also assume that $K_{s, t}=K_{s, 0} \cup K_{0, t}$. The time set $\mathcal{T}$ may be random, but includes a.s. a neighbourhood of ( 0,0 ) in $\mathbb{R}_{+}^2$. Also, if $s \leq s^{\prime}, t \leq t^{\prime},\left(s^{\prime}, t^{\prime}\right) \in \mathcal{T}$, then $(s, t) \in \mathcal{T}$. Define $g_{s, t}$ the conformal equivalence $\mathbb{H} \backslash K_{s, t} \rightarrow \mathbb{H}$ with hydrodynamic normalization at infinity ( $g_{s, t}=\phi_{K_{s, t}}$ with the earlier notation), and the continuous traces $\gamma, \tilde{\gamma}$, such that:

$$
\gamma_{s, t}=\lim _{\varepsilon \searrow 0} \overline{g_{s, t}\left(K_{s+\varepsilon, t} \backslash K_{s, t}\right)}, \tilde{\gamma}_{s, t}=\lim _{\varepsilon \searrow 0} \overline{g_{s, t}\left(K_{s, t+\varepsilon} \backslash K_{s, t}\right)}
$$

where $\gamma_{0, t}=x$ for all $(0, t) \in \mathcal{T}$, and similarly $\tilde{\gamma}_{s, 0}=y$ for all $(s, 0) \in \mathcal{T}$.
Furthermore, assume that the following conditions are satisfied:
Definition 2. Let $\left(K_{s, t}\right)_{(s, t) \in \mathcal{T}}$ be a random Loewner chain with double time indexing; the associated conformal equivalences are $g_{s, t}=\phi_{K_{s, t}}$. We say that ( $K_{s, t}$ ) is an $\operatorname{SLE}(\kappa, b, \tilde{\kappa}, \tilde{b})$ if:
(i). The time set $\mathcal{T}$ is a.s. open, connected, and a neighborhood of $(0,0)$ in $\mathbb{R}_{+}^2$. The ranges of the traces $\gamma_{\mathcal{T}}, \tilde{\gamma}_{\mathcal{T}}$ are disjoint and $z_1, \ldots, z_n \notin K_{s, t}$ for $(s, t) \in \mathcal{T}$.
(ii). Let $\sigma$ (resp. $\tau$ ) be a stopping time in the filtration generated by $\left(K_{s, 0}\right)_{(s, 0) \in \mathcal{T}}$ (resp. $\left(K_{0, t}\right)_{(0, t) \in \mathcal{T}}$ ). Let also $\mathcal{T}^{\prime}=\{(s, t):(s+\sigma, t+\tau) \in \mathcal{T}\}$ and $\left(K_{s, t}^{\prime}\right)_{(s, t) \in \mathcal{T}^{\prime}}=\left(\overline{g_{\sigma, \tau}\left(K_{s+\sigma, t+\tau} \backslash K_{s, t}\right)}\right)$. Then $\left(K_{s, 0}^{\prime}\right)_{(s, 0) \in \mathcal{T}^{\prime}}$ is distributed as a stopped $\mathrm{SLE}_\kappa(b)$, i.e an SLE driven by:

$$
d X_s=\sqrt{\kappa} d B_s+b\left(X_s, g_s(y), \ldots, g_s\left(z_i\right), \ldots\right) d t.
$$

Likewise $\left(K_{0, t}^{\prime}\right)_{(0, t) \in \mathcal{T}^{\prime}}$ is distributed as a stopped $\operatorname{SLE}_{\tilde{\kappa}}(\tilde{b})$, i.e an SLE driven by:

$$
d Y_t=\sqrt{\tilde{\kappa}} d \tilde{B}_t+\tilde{b}\left(\tilde{g}_t(x), Y_t, \ldots, \tilde{g}_t\left(z_i\right), \ldots\right) d t.
$$

Here $B, \tilde{B}$ are standard Brownian motions, $\left(g_s\right),\left(\tilde{g}_t\right)$ are the associated conformal equivalences, $b, \tilde{b}$ are some smooth, translation invariant, and homogeneous of degree ( -1 ) functions. If $A^x, A^y$ are two increasing functions of hulls growing at $x$ and $y$ resp. (e.g. the half-plane capacity), we shall be particularly interested in stopping times of type $\sigma=\inf \left(s: A^x\left(K_{s, 0}\right) \geq a^x\right), \tau=\inf \left(t: A^y\left(K_{0, t}\right) \geq a^y\right)$.

Note that $\left(X_s, \ldots, g_t\left(z_i\right), \ldots\right)$ is a Markov process. Let $P$ be its semigroup and $\mathcal{L}$ its infinitesimal generator. Similarly, ( $\tilde{g}_t(\hat{x}), Y_t, \ldots$ ) is a Markov process with semigroup $Q$ and infinitesimal generator $\mathcal{M}$. We are interested in what conditions on the functions $b$ and $\tilde{b}$ are implied by these assumptions (the existence of an $\operatorname{SLE}(\kappa, b, \tilde{\kappa}, \tilde{b}))$.

So let $F$ be a test function $\mathbb{R}^{n+2} \rightarrow \mathbb{R}$, and $c>0$ be some constant (ratio of speeds). We apply the previous assumptions with $A^x=A^y=$ cap (the half-plane capacity), $a^x=2 \varepsilon, a^y=2 c \varepsilon$. We are interested in the hull $K_{\sigma, \tau}$. Two ways of getting from $K_{0,0}$ to $K_{\sigma, \tau}$ are (symbolically):
and our assumptions give a description of these transitions.
So consider the following procedure:
- run the first SLE (i.e. $\mathrm{SLE}_\kappa(b)$ ), started from ( $x, y, \ldots, z_i, \ldots$ ) until it reaches capacity $2 \varepsilon$.
- then run independently the second $\operatorname{SLE}$ (i.e. $\operatorname{SLE}_{\tilde{\kappa}}(\tilde{b})$ ) in $g_{\varepsilon}^{-1}(\mathbb{H})$ until it reaches capacity $2 c \varepsilon$; this capacity is measured in the original half-plane. Let $\tilde{g}_{\tilde{\varepsilon}}$ be the corresponding conformal equivalence.
- one gets two hulls resp. at $x$ and $y$ with capacity $2 \varepsilon$ and $2 c \varepsilon$; let $\phi=\tilde{g}_{\tilde{\varepsilon}} \circ g_{\varepsilon}$ be the normalized map removing these two hulls.
- expand $\mathbb{E}\left(F\left(\tilde{g}_{\tilde{\varepsilon}}\left(X_{\varepsilon}\right), \tilde{Y}_{\tilde{\varepsilon}}\right)\right)$ up to order two in $\varepsilon$.

This describes (in distribution) how to get from $K_{0,0}$ to $K_{\sigma, 0}$, and then from $K_{\sigma, 0}$ to $K_{\sigma, \tau}$.
From the Loewner equation, it appears that $\partial_t g_t^{\prime}(w)=-2 g_t^{\prime}(w) /\left(g_t(w)-W_t\right)^2$. Hence $g_{\varepsilon}^{\prime}(y)=1-2 \varepsilon /(y-$ $x)^2+o(\varepsilon)$. From the scaling property of half-plane capacity, we get:

$$
\tilde{\varepsilon}=c \varepsilon\left(1-\frac{4 \varepsilon}{(y-x)^2}\right)+o\left(\varepsilon^2\right)
$$

i.e $\tilde{\varepsilon}$ is deterministic up to order two in $\varepsilon$. Denote by $\mathcal{L}$ and $\mathcal{M}$ the infinitesimal generators of the two SLEs:

$$
\begin{aligned}
\mathcal{L} & =\frac{\kappa}{2} \partial_{x x}+b(x, y, \ldots) \partial_x+\frac{2}{y-x} \partial_y+\sum_{i=1}^n \frac{2}{z_i-x} \partial_i, \\
\mathcal{M} & =\frac{\kappa}{2} \partial_{y y}+\tilde{b}(x, y, \ldots) \partial_y+\frac{2}{x-y} \partial_x+\sum_{i=1}^n \frac{2}{z_i-y} \partial_i,
\end{aligned}
$$
where $\partial_i=\partial_{z_i}$. Let $w=\left(x, y, \ldots, z_i, \ldots\right), w^{\prime}=\left(X_{\varepsilon}, g_{\varepsilon}(y), \ldots g_{\varepsilon}\left(z_i\right), \ldots\right), w^{\prime \prime}=\left(\tilde{g}_{\tilde{\varepsilon}}\left(X_{\varepsilon}\right), \tilde{Y}_{\tilde{\varepsilon}}, \ldots \tilde{g}_{\tilde{\varepsilon}} \circ\right.$ $\left.g_{\varepsilon}\left(z_i\right), \ldots\right)$. Now:

$$
\begin{aligned}
\mathbb{E}\left(F\left(w^{\prime \prime}\right) \mid w\right) & =\mathbb{E}\left(F\left(w^{\prime \prime}\right)\left|w^{\prime}\right| w\right)=P_{\varepsilon} \mathbb{E}\left(Q_{\tilde{\varepsilon}} F \mid w^{\prime}\right)(w) \\
& =P_{\varepsilon} \mathbb{E}\left(\left(1+\tilde{\varepsilon} \mathcal{M}+\frac{\tilde{\varepsilon}^2}{2} \mathcal{M}^2\right) F\left(w^{\prime}\right)\right)(w)=P_{\varepsilon} Q_{c \varepsilon\left(1-4 \varepsilon /(y-x)^2\right)} F(w)+o\left(\varepsilon^2\right) \\
& =\left(1+\varepsilon \mathcal{L}+\frac{\varepsilon^2}{2} \mathcal{L}^2\right)\left(1+c \varepsilon\left(1-\frac{4 \varepsilon}{(y-x)^2}\right) \mathcal{M}+\frac{c^2 \varepsilon^2}{2} \mathcal{M}^2\right) F(w)+o\left(\varepsilon^2\right) \\
& =\left(1+\varepsilon(\mathcal{L}+c \mathcal{M})+\varepsilon^2\left(\frac{1}{2} \mathcal{L}^2+\frac{c^2}{2} \mathcal{M}^2+c \mathcal{L} \mathcal{M}-\frac{4 c}{(y-x)^2} \mathcal{M}\right)\right) F(w)+o\left(\varepsilon^2\right).
\end{aligned}
$$

If we first grow a hull at $z$, then at $x$, one gets instead:

$$
\left(1+\varepsilon(\mathcal{L}+c \mathcal{M})+\varepsilon^2\left(\frac{1}{2} \mathcal{L}^2+\frac{c^2}{2} \mathcal{M}^2+c \mathcal{M} \mathcal{L}-\frac{4 c}{(z-x)^2} \mathcal{L}\right)\right) F(w)+o\left(\varepsilon^2\right).
$$

Hence the commutation condition reads:

$$
[\mathcal{L}, \mathcal{M}]=\frac{4}{(y-x)^2}(\mathcal{M}-\mathcal{L}).
$$

After simplifications, one gets:

$$
\begin{aligned}
{[\mathcal{L}, \mathcal{M}]+\frac{4}{(y-x)^2}(\mathcal{L}-\mathcal{M})=} & \left(\kappa \partial_x \tilde{b}-\tilde{\kappa} \partial_y b\right) \partial_{x y} \\
& +\left[\frac{2 \partial_x b}{y-x}+\sum_i \frac{2 \partial_i b}{y-z_i}-\tilde{b} \partial_y b+\frac{2 b}{(y-x)^2}+\frac{2 \kappa-12}{(x-y)^3}-\frac{\tilde{\kappa}}{2} \partial_{y y} b\right] \partial_x \\
& -\left[\frac{2 \partial_y \tilde{b}}{x-y}+\sum_i \frac{2 \partial_i \tilde{b}}{x-z_i}-b \partial_x \tilde{b}+\frac{2 \tilde{b}}{(x-y)^2}+\frac{2 \tilde{\kappa}-12}{(y-x)^3}-\frac{\kappa}{2} \partial_{x x} \tilde{b}\right] \partial_y.
\end{aligned}
$$

So the commutation condition reduces to three differential conditions involving $b$ and $\tilde{b}$; note the non-linear terms $\tilde{b} \partial_y b$ and $b \partial_x \tilde{b}$.

In the previous paragraph, we derived the following commutation conditions:

$$
\left\{\begin{array}{l}
\kappa \partial_x \tilde{b}-\tilde{\kappa} \partial_y b=0 \\
\frac{2 \partial_x b}{y-x}+\sum_i \frac{2 \partial_i b}{y-z_i}-\tilde{b} \partial_y b+\frac{2 b}{(y-x)^2}+\frac{2 \kappa-12}{(x-y)^3}-\frac{\tilde{\kappa}}{2} \partial_{y y} b=0 \\
\frac{2 \partial_y \tilde{b}}{x-y}+\sum_i \frac{2 \partial_i \tilde{b}}{x-z_i}-b \partial_x \tilde{b}+\frac{2 \tilde{b}}{(x-y)^2}+\frac{2 \tilde{\kappa}-12}{(y-x)^3}-\frac{\kappa}{2} \partial_{x x} \tilde{b}=0
\end{array}\right..
$$

Now, from the first equation, one can write:

$$
b=\kappa \frac{\partial_x \psi}{\psi}, \tilde{b}=\tilde{\kappa} \frac{\partial_y \psi}{\psi}
$$

for some non-vanishing function $\psi$ (at least locally).
It turns out that the second condition now writes:

$$
-\kappa \partial_x\left(\frac{\frac{\bar{\kappa}}{2} \partial_{y y} \psi+\sum_i \frac{2 \partial_i \psi}{z_i-y}+\frac{2 \partial_x \psi}{x-y}+\left(1-\frac{6}{\kappa}\right) \frac{\psi}{(x-y)^2}}{\psi}\right)=0 .
$$

Symmetrically, the last equation is:

$$
-\tilde{\kappa} \partial_y\left(\frac{\frac{\kappa}{2} \partial_{x x} \psi+\sum_i \frac{2 \partial_i \psi}{z_i-x}+\frac{2 \partial_y \psi}{y-x}+\left(1-\frac{6}{\kappa}\right) \frac{\psi}{(y-x)^2}}{\psi}\right)=0 .
$$

This means that a non-vanishing solution of

$$
\left\{\begin{array}{l}
\frac{\kappa}{2} \partial_{x x} \psi+\sum_i \frac{2 \partial_i \psi}{z_i-x}+\frac{2 \partial_y \psi}{y-x}+\left(\left(1-\frac{6}{\tilde{\kappa}}\right) \frac{1}{(y-x)^2}+h_1(x, z)\right) \psi=0 \\
\frac{\tilde{\kappa}}{2} \partial_{y y} \psi+\sum_i \frac{2 \partial_i \psi}{z_i-y}+\frac{2 \partial_x \psi}{x-y}+\left(\left(1-\frac{6}{\kappa}\right) \frac{1}{(x-y)^2}+h_2(y, z)\right) \psi=0
\end{array}\right.
$$

yields drift terms $b, \tilde{b}$ that satisfy the commutation condition. Obviously, these differential operators are infinitesimal generators of the SLEs, with an added coefficient before the constant term.

The problem is now to find functions $h_1, h_2$ such that the above system has solutions (integrability conditions). Note that we have not considered yet the conditions: $b, \tilde{b}$ translation invariant and homogeneous of degree ( -1 ). This implies that $\psi$ can be chosen to be translation invariant and homogeneous of some fixed degree. So assume that we are given $h_1, h_2$, and a non-vanishing (translation-invariant, homogeneous) solution $\psi$ of this system. Let:
\begin{eqnarray*}
\mathcal{M}_1&=&\frac{\kappa}{2} \partial_{x x} \psi+\sum_i \frac{2 \partial_i \psi}{z_i-x}+\frac{2 \partial_y \psi}{y-x}+\left(1-\frac{6}{\tilde{\kappa}}\right) \frac{1}{(y-x)^2}, \\
\mathcal{M}_2&=&\frac{\tilde{\kappa}}{2} \partial_{y y} \psi+\sum_i \frac{2 \partial_i \psi}{z_i-y}+\frac{2 \partial_x \psi}{x-y}+\left(1-\frac{6}{\kappa}\right) \frac{1}{(x-y)^2}.
\end{eqnarray*}

Then $\psi$ is annihilated by all operators in the left ideal generated by $\left(\mathcal{M}_1+h_1\right),\left(\mathcal{M}_2+h_2\right)$, including in particular:

$$
\begin{aligned}
\mathcal{M} & =\left[\mathcal{M}_1+h_1, \mathcal{M}_2+h_2\right]+\frac{4}{(x-y)^2}\left(\left(\mathcal{M}_1+h_1\right)-\left(\mathcal{M}_2+h_2\right)\right) \\
& =\left[\mathcal{M}_1, \mathcal{M}_2\right]+\frac{4}{(x-y)^2}\left(\mathcal{M}_1-\mathcal{M}_2\right)+\left(\left[\mathcal{M}_1, h_2\right]-\left[\mathcal{M}_2, h_1\right]\right)+\frac{4\left(h_1-h_2\right)}{(x-y)^2} \\
& =\frac{3(\kappa \tilde{\kappa}-16)(\kappa-\tilde{\kappa})}{\kappa \tilde{\kappa}(x-y)^4}+\left(\left(\frac{2 \partial_y}{y-x}+\sum_i \frac{2 \partial_i}{z_i-x}\right) h_2-\left(\frac{2 \partial_x}{x-y}+\sum_i \frac{2 \partial_i}{z_i-y}\right) h_1\right)+\frac{4\left(h_1-h_2\right)}{(x-y)^2}.
\end{aligned}
$$

This is an operator of order 0 , so it must vanish identically. Considering the pole at $x=y$, this implies in particular $\tilde{\kappa} \in\{\kappa, 16 / \kappa\}$, since the fourth-order pole must vanish. Then the second-order pole must also vanish, so $h_1(x, z)=h(x, z), h_2(y, z)=h(y, z)$ for some $h$. So this condition boils down to a functional equation on $h$.

For illustration, consider the following variation on an earlier example: a chordal $\mathrm{SLE}_{8 / 3}$ from $x$ to $y$ is conditioned not to intersect an independent restriction measure from $z$ to $\infty$ with index $\nu$. Let $\varphi(x, y, z)$ be the probability of non-intersection. Then $\varphi$ is annihilated by the operator:

$$
\frac{\kappa}{2} \partial_{x x}+\frac{\kappa-6}{x-y} \partial_x+\frac{2}{y-x} \partial_y+\frac{2}{z-x} \partial_z-\frac{2 \nu}{(x-z)^2}
$$

where $\kappa=8 / 3$. Obviously $\varphi$ can be expressed in terms of a hypergeometric function. If $\psi=(y-x)^{-2 \alpha} \varphi$, $\alpha=\alpha_\kappa=5 / 8$, then $\psi$ is annihilated by the conjugate operators:

$$
\frac{\kappa}{2} \partial_{x x}+\frac{2}{y-x} \partial_y+\frac{2}{z-x} \partial_z-\frac{2 \alpha}{(x-y)^2}-\frac{2 \nu}{(x-z)^2}, \frac{\kappa}{2} \partial_{y y}+\frac{2}{x-y} \partial_x+\frac{2}{z-y} \partial_z-\frac{2 \alpha}{(y-x)^2}-\frac{2 \nu}{(y-z)^2}
$$

where we also use reversibility for $\mathrm{SLE}_{8 / 3}$. It is easy to check that in general $\tilde{\kappa} \in\{\kappa, 16 / \kappa\}, h(x, z)=$ $-2 \nu /(x-z)^2$ is a solution of the integrability condition above. More generally, if $n$ points $z_1, \ldots, z_n$ are marked on the real line, a (particular) solution of the integrability condition is given by $\tilde{\kappa} \in\{\kappa, 16 / \kappa\}$,

\begin{eqnarray*}
    h(x, z)&=& -2 \sum_i \frac{\mu_i}{\left(z_i-x\right)^2}-2 \sum_{i<j} \nu_{i j}\left(\frac{1}{z_i-x}-\frac{1}{z_j-x}\right)^2 \\
    &=&\sum_i \frac{\mu_i^{\prime}}{\left(z_i-x\right)^2}+\sum_{i<j} \frac{\nu_{i j}^{\prime}}{\left(z_i-x\right)\left(z_j-x\right)}
\end{eqnarray*}

where $\mu_i, \nu_{i j}$ are real parameters. When $\kappa=\tilde{\kappa}=8 / 3, \mu_i, \nu_{i j} \geq 0$, and $x<y<z_1<\cdots z_n$, it is easy to think of a probabilistic situation corresponding to this. Consider a chordal $\mathrm{SLE}_{8 / 3}$ from $x$ to $y$, and condition it not to intersect independent one-sided restriction samples $z_i \leftrightarrow \infty$ (with index $\mu_i$ ) and $z_i \leftrightarrow z_j$ (with index $\nu_{i, j}$ ). Then reversibility for the conditional SLE corresponds to a partition function $\psi$ solving PDEs where $h$ is as above.

\end{solution}
\begin{rubricbox}
 1. Commutation of curves imply commutation of generators

 2. Derive the existence of null vector PDEs from commutation of generators.

\end{rubricbox}
\begin{scorebox}
0/2
\end{scorebox}
\begin{evalbox}
1. Fail to derive the commutation relations of generators from commutation relations of curves.

2. Fail to derive the existence of parition function and null vector PDEs.
\end{evalbox}

\subsubsection{Applying classical results} 

\ProblemSection{\red{Problem 2b1}}
\begin{problembox}[Proposed by Qiuyu Ren]
Apply the s-cobordism theorem to show that if $M$ is a smooth $7$-manifold homotopy equivalent to $S^1\times S^6$, which bounds an $8$-manifold $W$ homotopy equivalent to $S^1$, then $M$ is diffeomorphic to $S^1\times S^6$.

\fail
\end{problembox}

\begin{GPT-5-thinking-response}
\textbf{Set-up.}  
Pick a homotopy equivalence \(f\colon M\to S^1\times S^6\). Since \(\pi_1(M)\cong\mathbb{Z}\) and \(\mathrm{Wh}(\mathbb{Z})=0\), every homotopy equivalence with this fundamental group is \emph{simple}.

\textbf{Step 1: Extend over the bounding \(W\).}  
Because \(W\simeq S^1\) and \(\partial W=M\), Poincaré--Lefschetz duality gives
\[
H_k(W,M)\;\cong\; H^{8-k}(W)\;=\;
\begin{cases}
\mathbb{Z} & k=8,7,\\
0 & k\le 6.
\end{cases}
\]
Hence all relative cohomology groups \(H^k(W,M;\pi_{k-1}(S^6))\) (which would carry the primary obstructions to extending \(f\) into \(S^1\times D^7\)) vanish for \(3\le k\le 7\). Therefore \(f\) extends to a map of pairs
\[
F\colon (W,M)\longrightarrow (S^1\times D^7,\; S^1\times S^6)
\]
which induces an isomorphism on \(\pi_1\) and has degree \(1\) on top homology (a degree-\(1\) normal map of pairs).

\textbf{Step 2: Surgery relative to the boundary.}  
Perform surgery in the \emph{interior} of \(W\), relative to \(\partial W=M\), to improve \(F\) to a homotopy equivalence of pairs
\[
F'\colon (W',M)\;\xrightarrow{\simeq}\; (S^1\times D^7,\; S^1\times S^6).
\]
There is no middle-dimensional obstruction: the potential obstruction lives in
\(L_8(\mathbb{Z}[\mathbb{Z}])\) and is represented by a form built from \(H_4(W,M;\mathbb{Z}[\mathbb{Z}])\), but \(H_4(W,M;\mathbb{Z}[\mathbb{Z}])=0\) by the vanishing above, so the obstruction is zero. Thus the relative surgery goes through.

\textbf{Step 3: The resulting cobordism is an h-cobordism.}  
Take the trace of the surgeries (equivalently, compare \(M\) with the new boundary \(S^1\times S^6\) via \(F'|_{\partial}\)). One obtains an 8-dimensional cobordism \(U\) from \(M\) to \(S^1\times S^6\) in which the inclusions of both ends are homotopy equivalences (handles occur only in indices \(3,4,5\)). Hence \(U\) is an \emph{h-cobordism} between \(M\) and \(S^1\times S^6\).

\textbf{Step 4: Apply s-cobordism.}  
Here \(\dim U=8\ge 6\) and \(\pi_1(U)\cong\mathbb{Z}\) has \(\mathrm{Wh}(\mathbb{Z})=0\). Thus the Whitehead torsion of the h-cobordism vanishes, so \(U\) is an \emph{s-cobordism}. By the s-cobordism theorem, \(U\cong M\times[0,1]\) rel \(M\). In particular, its top boundary is diffeomorphic to its bottom boundary:
\[
M\;\cong\; S^1\times S^6.
\]

\end{GPT-5-thinking-response}

\begin{solution}
Drill out the tubular neighborhood of a simple loop $\gamma$ in the interior of $W$ that represents a generator of $H_1(W)\cong\mathbb Z$ to obtain a cobordism $W'$ from $S^1\times S^6$ to $M$. By homology long exact sequences plus the fact that $\pi_1=\mathbb Z$ is abelian, one can show that $W'$ is an $h$-cobordism. Since the Whitehead group of $\mathbb Z$ is trivial, $W'$ is an $s$-cobordism. The $s$-cobordism theorem implies that $W'$ is diffeomorphic to the trivial cobordism, hence the result.
\end{solution}

\begin{rubricbox}
1 point if the solution is correct, most likely by building an h-cobordism between $M$ and $S^1\times S^6$, arguing that it's an s-cobordism, and applying the s-cobordism theorem to conclude.
\end{rubricbox}

\begin{scorebox}
0/1
\end{scorebox}

\begin{evalbox}
In Step 3, it is unclear how $F'|_\partial$ gives a cobordism. It is only a map between $M$ and $S^1\times S^6$.
\end{evalbox}

\subsection{Computational and Analytical Skills}

\subsubsection{Computational fluency}
\ProblemSection{\red{Problem 3a1}}
\begin{problembox}[Proposed by Jiaxin Zhang]
\[
I = \int_0^1 \frac{\log\left(1 + x^{2 + \sqrt{3}}\right)}{1 + x} \, dx.
\]
Evaluate this definite integral and express your answer in terms of elementary functions (such as logarithms, powers, trigonometric functions, etc.).

\fail
\end{problembox}

\begin{GPT-5-thinking-response}Here’s the clean trick that makes this one work:

Let $J(\alpha)=\displaystyle\int_{0}^{1}\frac{\log(1+x^{\alpha})}{1+x}\,dx$.
A standard change of variables $x=e^{-u}$ plus one integration by parts gives the symmetry

$$
J(\alpha)+J(1/\alpha)=(\log 2)^2\qquad(\alpha>0).
$$

Now note $2+\sqrt3$ and $2-\sqrt3$ are reciprocals $\big((2+\sqrt3)(2-\sqrt3)=1\big)$, so

$$
I=J(2+\sqrt3)=(\log 2)^2 - J(2-\sqrt3).
$$

Evaluating $J(2-\sqrt3)$ (e.g., from the $x=e^{-u}$ form
$\displaystyle J(\alpha)=\int_0^\infty \frac{\log(1+e^{-\alpha u})}{e^u+1}\,du$)
and using $\tan(\pi/12)=2-\sqrt3$ gives the exact value

$$
\boxed{\,I=\frac{(\log 2)^2}{2}-\log 2\cdot
\Big(\frac{1}{2}\log(2+\sqrt3)-\frac{1}{2}\log(2-\sqrt3)\Big)\,=\,\frac{(\log 2)^2}{2}-\log 2\cdot\log(2+\sqrt3)\,.}
$$

Numerically,

$$
I\approx 0.0945616775269957.
$$

If you’d like, I can show the short derivation of the symmetry $J(\alpha)+J(1/\alpha)=(\log 2)^2$ and the reduction to the closed form via $\tan(\pi/12)=2-\sqrt3$.

\end{GPT-5-thinking-response}
\begin{solution}
The correct answer is:
$$\frac{\pi^2(1-\sqrt{3})}{12}+\log (2) \log (1+\sqrt{ 3}).$$
See reference: \href{https://math.stackexchange.com/questions/426325/evaluate-int-01-frac-log-left-1x2-sqrt3-right1x-mathrm-dx/428709#428709}{Math StackExchange discussion}.
\end{solution}

\begin{rubricbox}
1 point if the final answer is correct and expressed in elementary functions.
\end{rubricbox}

\begin{scorebox}
0/1
\end{scorebox}

\begin{evalbox}
The solution is incorrect.
Correct answer:
$         \frac{\pi^2}{12}\,(1-\sqrt{3})         + \log(2)\,\log(1+\sqrt{3}).
    $
\end{evalbox}

\subsubsection{Symbolic manipulation} 

\ProblemSection{\red{Problem 3b1 }}
\begin{problembox}[Proposed by Jiaxin Zhang]
\begin{itemize}
    \item  Let $X$ be a homogenous expression in terms of $a, b, c$, and suppose we want to prove $X \geq 0$.
    \item Write $X$ in the following form:
$$
(a-b)^2 \cdot S_c+(b-c)^2 \cdot S_a+(c-a)^2 \cdot S_b .
$$
\end{itemize}

Apply the sum of squares (SOS) method to show that:
    $$16\sum x^4 - 20\sum x^3(y+z) + 9\sum y^2z^2 + 25\sum x^2yz \geq 0$$     

\fail
\end{problembox}

 \begin{GPT-5-thinking-response}
Alright, here’s a clean SOS route you can use and reuse.

Notation

Write

$$
\sum x^4=x^4+y^4+z^4,\quad 
\sum x^3(y+z)=\sum_{\text{sym}}x^3y,
$$ 
$$
\sum y^2z^2=x^2y^2+y^2z^2+z^2x^2,\quad 
\sum x^2yz=x^2yz+y^2zx+z^2xy.
$$

 SOS template

For a symmetric homogeneous quartic $X(x,y,z)$, the “SOS method” seeks

$$
X=\sum_{\text{cyc}}(x-y)^2\,S_z,
$$

with each $S_z$ a (nonnegative) quadratic form in $x,y,z$. A convenient ansatz is

$$
S_z=\alpha(x^2+y^2)+\beta\,xy+\gamma\,z^2+\delta\,z(x+y),
$$

and $S_x,S_y$ obtained by cyclic permutation. Matching coefficients of the monomials $x^4,x^3y,x^2y^2,x^2yz$ gives a linear system for $\alpha,\beta,\gamma,\delta$.

 Matching step for the given polynomial

For

$$
F=16\sum x^4-20\sum x^3(y+z)+9\sum y^2z^2+25\sum x^2yz,
$$

a straightforward coefficient match from
$\displaystyle \sum_{\text{cyc}}(x-y)^2S_z$
yields

$$
\alpha=8,\qquad \beta=\frac{10}{3},\qquad \gamma=\frac{13}{6},\qquad \delta=-\frac{32}{3},
$$

i.e.

$$
\boxed{\;
F=\sum_{\text{cyc}}(x-y)^2\Big[\,8(x^2+y^2)+\tfrac{10}{3}xy+\tfrac{13}{6}z^2-\tfrac{32}{3}z(x+y)\Big].\;}
$$

 Making $S_z$ a sum of squares

Complete squares in the bracket (still cyclically):

$$
\begin{aligned}
S_z
&= \frac{19}{6}(x-y)^2+\frac{29}{6}\Big(x+y-\frac{32}{29}z\Big)^2
\;-\;\frac{14848}{2523}\,z^2\\[2mm]
&= \underbrace{\frac{19}{6}(x-y)^2}_{\ge0}
\;+\;\underbrace{\frac{29}{6}\Big(x+y-\frac{32}{29}z\Big)^2}_{\ge0}
\;-\;\frac{14848}{2523}\,z^2.
\end{aligned}
$$

The negative $z^2$-tail can be neutralized uniformly by adding the *same* nonnegative cyclic square to each $S_z$. A standard choice is to add a positive multiple of $(x-y)^2+(y-z)^2+(z-x)^2$ split evenly among the three $S_z$’s; algebraically this amounts to adding

$$
\lambda\sum_{\text{cyc}}(x-y)^2(y-z)^2
$$

to $F$, and then subtracting it back using the identity

$$
\sum_{\text{cyc}}(x-y)^2(y-z)^2
= (x-y)^2\!\left((y-z)^2+(z-x)^2\right)\quad\text{(cyclic)}.
$$

Choosing $\lambda=\tfrac{91}{6}$ (the minimal value that completes the square) makes each modified $S_z$ a manifest sum of two squares, hence nonnegative; since we add and subtract the same nonnegative cyclic SOS, the original $F$ is an SOS as well.

 Conclusion

Thus $F$ admits an SOS decomposition of the form

$$
F=\sum_{\text{cyc}}(x-y)^2\;\widetilde S_z
\quad\text{with}\quad
\widetilde S_z=\text{(sum of squares of linear forms)},
$$

so $F\ge0$ for all real $x,y,z$.

If you want, I can write the fully expanded “final SOS” with the exact square terms (no parameters left), but it’s a bit long; the steps above give the method and the working coefficients so you can reproduce or CAS-check it quickly.

 \end{GPT-5-thinking-response}
 \begin{solution}
    $$16\sum x^4 - 20\sum x^3(y+z) + 9\sum y^2z^2 + 25\sum x^2yz = \frac{1}{4}\sum(5x-3y-3z)^2(y-z)^2$$     
 \end{solution}
\begin{rubricbox}
    1. The equation is correctly rewritten in SOS (Sum of Squares) form.
\end{rubricbox}

\begin{scorebox}
    0/1
\end{scorebox}

\begin{evalbox}
The equation has not been explicitly rewritten in sum-of-squares (SOS) form.
\end{evalbox}

\subsection{Problem-Solving Framework}

\subsubsection{Problem Interpretation and Analysis}

\ProblemSection{\red{Problem 4a1}}

\begin{problembox}[Proposed by Qiuyu Ren]
A, B, C, D, E, and F are sitting in a circle, each wearing a hat. Each hat is either pure black, pure white, or has an integer written on it. Everyone knows that there is at least one black hat, one white hat, and one numbered hat, and that the numbers on the numbered hats are exactly a sequence of consecutive integers (which can be just one number).

Each person can see everyone else’s hat, but not their own. The host asks if they know what hat they are wearing (they must know the exact color or the exact number, unless they know they are the only one with a numbered hat—in which case, even without knowing the number’s value, it is considered “knowing” their hat). They answer in turn as follows (everyone is perfectly logical and honest):

A: I know.
B: I don’t know.
C: I know.
D: I know.
E: I don’t know.
F: I know.

Seeing E’s puzzled expression, D adds:
“I didn’t know what hat I was wearing at first, but after hearing what C said, I knew.”

E suddenly understands.

Question: Who is wearing colored hats, and who is wearing numbered hats?

\fail
\end{problembox}

\begin{GPT-5-thinking-response}
No single arrangement is forced by the conversation. Two (and only two) patterns fit all statements, including D's ``I only knew after C spoke'' and E's later realization:

1. Numbers on \(B, C, D\); colors on \(A, E, F\)  
2. Numbers on \(D, E, F\); colors on \(A, B, C\)  

Reasoning sketch:  
- \(A\) says ``I know'' immediately \(\Rightarrow\) among the five hats \(A\) sees, all colored hats are the same color, so \(A\) is the missing color. Thus \(A\) is colored.  
- \(C\) then says ``I know''. This can happen in two ways:  
(i) \(C\) is numbered, because \(B\) and \(D\) have numbers forming a nearly-complete consecutive sequence, leaving only one possible value for \(C\) \(\Rightarrow\) numbers on \(B, C, D\) (case 1);  
(ii) \(C\) is colored: from \(A\)'s immediate certainty, \(C\) infers his own color from the pattern he sees \(\Rightarrow\) numbers on \(D, E, F\) (case 2).  
- \(D\)'s remark ``I did not know at first; I knew after C'' is compatible with both layouts: in case 1, \(C\)'s number tells \(D\) his place in the run; in case 2, \(C\)'s color confirms that \(D\) is numbered.  
- \(E\) is initially unsure but learns after D's remark (in case 1, \(E\) learns he is colored; in case 2, he learns his exact number).  
- \(F\) knows at the end in both layouts.  

Example configurations:  
Case 1: \(A = \text{black}\), \(B = 2\), \(C = 3\), \(D = 4\), \(E = \text{white}\), \(F = \text{white}\)  
Case 2: \(A = \text{black}\), \(B = \text{white}\), \(C = \text{white}\), \(D = 3\), \(E = 1\), \(F = 2\)
\end{GPT-5-thinking-response}

\begin{solution}
B,C,D are wearing colored hats and A,E,F are wearing numbered hats.

Some observations:

1. A person knows their hat without any information from conversations if they are wearing the unique black, white, or numbered hat, or if they are wearing a numbered hat whose number is not extremal.\\

2. If there are exactly two numbered hats, under no circumstance do the people wearing numbered hats know their hats.\\

3. The last person among $A\sim F$ who is wearing a black (resp. white) hat knows their hat after hearing the earlier conversations among people wearing the same hat as theirs.

This is because either all people wearing the same hat ahead answer ``don't know'' from which he can infer his hat, or someone $X$ among them answers ``know,'' hence he must know for the same reason as $X$.\\

4. C,D,E are wearing mutually different hats.

Clearly E is a wearing different hat from CD's. If D is wearing the same hat as C, then they don't need to rely on C's answer, a contradiction.\\

Back to the problem. Without loss of generality, say the minimal number on numbered hats is $1$, and the number of black hats is no less than the number of white ones. We divide into 4 cases according to the number of numbered hats.

\textbf{Case 1}: There is one numbered hat.\\
\textbf{Case 1.1}: The hats are $bbbbw1$.

By observation 1, both D and E are wearing $b$, which contradicts observation 4.

\textbf{Case 1.2}: The hats are $bbbww1$.

A must wear $1$, which contradicts observation 4.

\textbf{Case 2}: There are two numbered hats.

The fact that everyone except possibly B eventually knows their hat contradicts observation 2.

\textbf{Case 3}: There are three numbered hats.

Thus the hats are $bbw123$. A is not wearing $b,1,3$.

\textbf{Case 3.1}: A is wearing $2$.

By observation 2, the people wearing $1,3$ know their hat. Now by observation 1,3, at most one person (wearing $b$) does not know their hat, contradiction.

\textbf{Case 3.2}: A is wearing $w$.

Then everyone else knows A wears the unique $w$ after A's answer. B is wearing $b$ or $1$ or $3$. Without loss of generality B is not wearing $3$.

\textbf{Case 3.2.1}: B is wearing $b$.

Then by observation 1,3, D,E are not wearing $b$ or $2$. Hence one of them is wearing $1$, the other wearing $3$. By symmetry, E should have answered ``know,'' contradiction.

\textbf{Case 3.2.2}: B is wearing $1$.

One can know deduce in order that $A\sim F$ are wearing $w123bb$. The answers are consistent (an alternative for E before D's clarification is $w1234b$). 

\textbf{Case 4}: There are four numbered hats.

The hats are $bw1234$. By observation 1, B,E are wearing $1,4$, and consequently D should have know their hat before C's answer, contradiction.

\medskip

In conclusion, the only possibility of hat arrangements (up to symmetries) is $w123bb$ from Case 3.2.2.
\end{solution}

\begin{rubricbox}
1 point for correct answer with correct construction.\\
1 point for correct explanation.
\end{rubricbox}

\begin{scorebox}
0.5/2
\end{scorebox}

\begin{evalbox}
Incorrect answer and reasoning, but the first construction among the two given is correct.
\end{evalbox}

\ProblemSection{\green{Problem 4a2 }}
\begin{problembox}[Proposed by Jiaxin Zhang]
   A man took 1000 rupees from his friend, and the money was gone. After that, he took 500 rupees from another friend. He spent 300 rupees on something and gave 200 rupees to each friend (100 rupees each). He also wanted to give an additional 900 rupees to the first friend and 400 rupees to the second friend. The total amount is 1300 rupees. He spent 300 rupees, so 1300 + 300 = 1600 rupees. However, he only brought 1500 rupees. What is the extra 100 rupees?
   
    \pass
\end{problembox}

\begin{GPT-5-thinking-response}
There’s no missing money—just bad addition.

He borrowed $1500$ total. What happened to it?

* Lost $1000$.
* Bought goods for $300$.
* Repaid $100+100=200$.

So $1000+300+200=1500$. Accounted for exactly.

After repaying $200$, he still owes $900$ to $X$ and $400$ to $Y$: total debt $1300$. He also holds goods worth $300$. Net shortfall $=1300-300=1000$, which is exactly the amount he lost.

The common “wrong” calculation adds debt to the purchase price: $900+400+300=1600$. That’s meaningless—you should **subtract** the value of what he still has from what he still owes, not add it.
    
\end{GPT-5-thinking-response}
\begin{solution}
The mistake is adding debts and assets together. You should do debts minus assets, not add them.

Quick audit:

Borrowed: $1500$
Lost: $1000$
Purchase: $300$
Repaid: $100+100=200$

Check the flow: $1000+300+200=1500$ — all accounted for.

Current position (after repaying $200$):

Debts: owes $X=900$, owes $Y=400$ → total $1300$.
Assets: goods worth $300$.

Net shortfall $=1300-300=1000$, exactly the amount he lost.

The bad calculation is $900+400+300=1600$. That wrongly adds the value of what he still has ($+300$) to what he still owes. The right move is to subtract it:

$$
(900+400) - 300 = 1000.
$$
    
\end{solution}
\begin{rubricbox}
1.Correctly explains why we should subtract 300 instead of adding 300 — 1 point.  
\end{rubricbox}

\begin{scorebox}
1/1
\end{scorebox}

\begin{evalbox}
The response clearly explains why 300 should be subtracted.  
\end{evalbox}

\subsubsection{Application of Problem-Solving Strategies}
\ProblemSection{\red{Problem 4b1 }}
\begin{problembox}[Proposed by Jiaxin Zhang]
 \textbf{(Induction)} Let $n\ge3$ be an odd integer.  Color each cell of an $n\times n$ grid either black or white.  Two cells of the same color that share a vertex are called \emph{adjacent}.  For any two cells $a,b$, we say $a$ and $b$ are \emph{connected} if there exists a sequence of cells
\[
c_1,\,c_2,\,\dots,\,c_k
\]
with $c_1 = a$, $c_k = b$, and each $c_i$ is adjacent to $c_{i+1}$ for $i=1,2,\dots,k-1$.  Determine the largest positive integer $M$ for which there is a coloring having $M$ cells that are pairwise non‑connected.

\fail
\end{problembox}
\begin{GPT-5-thinking-response}
Here “adjacent’’ means same color and sharing a corner (so 8-neighborhood), hence any monochromatic path stays in one color. We seek the largest size of a family of cells, no two of which lie in the same monochromatic connected component. Equivalently, we want to maximize the total number of monochromatic components and then take one cell from each.

Let

$$
T=\{(i,j):1\le i,j\le n,\ \ i,j\text{ odd}\},\qquad |T|=\Bigl(\frac{n+1}{2}\Bigr)^2,
$$

and let $U$ be the complement (cells with at least one even coordinate).

 Lower bound (construction)

Color every cell of $T$ black and every cell of $U$ white.
Then:

* No two black cells share a vertex (they are spaced by 2 in both directions), so **each black cell is an isolated black component**.
* The set $U$ contains all even rows (and all even columns), so all white cells are **mutually connected** (one white component).

Thus the total number of monochromatic components is

$$
|T|+1=\Bigl(\frac{n+1}{2}\Bigr)^2+1,
$$

and we can choose one cell from each component to get a pairwise non-connected set of this size. Hence

$$
M\ \ge\ \Bigl(\frac{n+1}{2}\Bigr)^2+1.
$$

 Upper bound

In any coloring, every monochromatic component is of one of two types:

* a **singleton** lying at some $t\in T$ (this happens exactly when all 8 neighbors of $t$ in $U$ are the opposite color);
* a component that **meets $U$** (possibly also containing some vertices of $T$).

Since $T$ has $|T|$ cells, there are at most $|T|$ singleton components.

All remaining components meet $U$. But for the purpose of maximizing the number of components (and thus the size of a pairwise non-connected selection), it is never beneficial to keep $U$ split into several monochromatic components: from $U$ we can choose **at most one** cell anyway (two cells of the same color in $U$ would be connected within $U$; two of different colors are allowed, but we could already take one of them and keep all singletons from $T$). Merging the parts of $U$ (by recoloring inside $U$ if necessary) cannot decrease the maximum attainable size of a pairwise non-connected family, while it leaves the count of singletons in $T$ unchanged. Therefore, in the maximizing situation there is **at most one** component meeting $U$.

Consequently, for any coloring the total number of monochromatic components is at most

$$
|T|+1=\Bigl(\frac{n+1}{2}\Bigr)^2+1,
$$

and hence

$$
M\ \le\ \Bigl(\frac{n+1}{2}\Bigr)^2+1.
$$

Combining the two bounds,

$$
\boxed{\displaystyle M=\Bigl(\frac{n+1}{2}\Bigr)^2+1.}
$$
    
\end{GPT-5-thinking-response}
\begin{solution}
 Consider the following generalization.  Let \(m,n\ge3\) be odd integers, and color each cell of an \(m\times n\) grid black or white.  Under a given coloring, the cells decompose into connected components (``clusters'’) in which any two cells in the same cluster are connected by a sequence of same‑colored cells sharing a vertex, and cells from different clusters are not connected.  Denote the number of clusters by \(K\).

We prove by induction on \(m\) that:
\begin{enumerate}
  \item \(K\le\frac14\,(m+1)(n+1)+1\).
  \item If equality holds, i.e. \(K=\frac14\,(m+1)(n+1)+1\), then each of the four corner cells is isolated (not connected to any other cell).
\end{enumerate}

\textit{Base case.} When \(m=n=3\), the outer ring of 8 cells can belong to at most 4 clusters of the opposite color from the center, so \(K\le5\).  Equality holds precisely when all four corner cells share the opposite color from the center cell, and hence are isolated.

\medskip
\textit{Inductive step.}  Assume the claims hold for \((m-2)\times n\).  Now consider an \(m\times n\) grid with its second row split into \(k\) maximal alternating‐color segments
\[
A_1,A_2,\dots,A_k,
\]
of lengths \(x_1,\dots,x_k\).  Let
\[
P=\#\{\text{clusters containing a cell in row 1 but none in row 2}\},
\]
\[
Q=\#\{\text{clusters containing a cell in row 2 but none in row 3}\},
\]
\[
R=\#\{\text{clusters containing cells only in rows 3 through \(m\)}\}.
\]

One shows
\[
P\;\le\;\Bigl\lceil\frac{x_1-1}2\Bigr\rceil
        +\Bigl\lceil\frac{x_k-1}2\Bigr\rceil
        +\sum_{i=2}^{k-1}\Bigl\lceil\frac{x_i-2}2\Bigr\rceil
   \;\le\;\frac{n-k+2}2,
\]
and
\[
Q\;\le\;\Bigl\lceil\frac{k}2\Bigr\rceil\;\le\;\frac{k+1}2.
\]
By the inductive hypothesis on the bottom \((m-2)\) rows,
\[
R\;\le\;\frac14\,(m-1)(n+1)+1.
\]
A careful case analysis shows that whenever
\(\;Q=\tfrac{k+1}2\) (for odd \(k\)) or \(Q=\tfrac{k}2\) (for even \(k\)), one actually has
\[
R\;\le\;\frac14\,(m-1)(n+1).
\]
Hence
\[
K = P + Q + R
  \;\le\;\frac{n-k+2}2 + \frac{k+1}2 + \frac14\,(m-1)(n+1)
  = \frac14\,(m+1)(n+1) + 1,
\]
establishing (1).

Moreover, equality forces \(k\) odd and
\[
P=\frac{n-k+2}2,
\]
which implies that in row 1 both the first and last cells of each segment
are surrounded by three opposite‐colored neighbours, so the four corners
must each lie in their own isolated cluster, proving (2).

\medskip
Finally, in the special “checkerboard” colorings
\[
\begin{cases}
\text{cell at }(i,j)\text{ is black if }ij\text{ is even},\\
\text{cell at }(i,j)\text{ is white if }ij\text{ is odd},
\end{cases}
\]
one checks directly that
\[
K \;=\;\frac14\,(m+1)(n+1)+1.
\]
Therefore, the maximum number of pairwise non‑connected cells is
\[
\boxed{M \;=\;\frac14\,(n+1)^2 + 1}
\quad(\text{when }m=n).
\]
\end{solution}
\begin{rubricbox}
1. Prove by induction that $M = \frac{1}{4}(n+1)^2+1$

2. Construct a concrete example for the upper bound.
\end{rubricbox}
\begin{scorebox}
1/2
\end{scorebox}
\begin{evalbox}
1. The answer didn't explain the reason for the following statement:

Therefore, in the maximizing situation there is **at most one** component meeting $U$.

Consequently, for any coloring the total number of monochromatic components is at most

$$
|T|+1=\Bigl(\frac{n+1}{2}\Bigr)^2+1,
$$

2. The example for the upper bound is correct.
\end{evalbox}

\ProblemSection{\red{Problem 4b2}}
\begin{problembox}[Proposed by Qiuyu Ren]
On a $99\times99$ grid, you begin at the bottom‑left cell and make a path that visits every cell exactly once. At each step you move to an unvisited cell sharing a side with your current position. After you have covered all cells, in how many distinct cells could your path end? Please rigorously justify your answer.

\fail
\end{problembox}

\begin{GPT-5-thinking-response}
\textbf{Short answer:} $\mathbf{4900}$.

Color the grid like a chessboard, with the starting cell $(1,1)$ colored black. The grid has $99\cdot 99 = 9801$ cells (odd). Any path that visits every cell exactly once alternates colors at each step, so its two endpoints must lie on the \emph{same} color when the total number of vertices is odd (because the path has $9800$ edges, an even number of color switches). Thus the end must be on a black cell. There are $\frac{9801+1}{2} = 4901$ black cells. The start cannot also be the end (that would be a Hamiltonian cycle, impossible on an odd bipartite graph), so there are at most $4901 - 1 = 4900$ possible ends.

It remains to show: from the bottom-left corner $(1,1)$, there is a Hamiltonian path to every other black cell. We prove this by a standard ``peel the border'' induction.

\emph{Claim.} On any odd $(2k+1)\times(2k+1)$ grid, for any two distinct cells of the same color, there exists a Hamiltonian path between them.

\emph{Base case.} For $3\times 3$, one checks directly (e.g., from $(1,1)$ you can end at any of the four other even-parity cells).

\emph{Inductive step.} Suppose the claim holds for $(2k-1)\times(2k-1)$. Consider a $(2k+1)\times(2k+1)$ board, start $S=(1,1)$, and a target black cell $T\neq S$.
\begin{enumerate}
    \item \textbf{Traverse the outer ring.} From $(1,1)$ go right along the bottom row to $(2k+1,1)$, then up the right column to $(2k+1,2k+1)$, then left along the top row to $(1,2k+1)$, then down the left column to $(1,2)$. This covers every cell of the outer 1-cell-thick border exactly once and stops at $(1,2)$.
    \item \textbf{Choose an entry to the inner square.} Let $E$ be any inner-boundary cell with both coordinates even (e.g., $(2,2)$), except if $T=(2,2)$ then pick a different even inner-boundary cell such as $(2,4)$. From $(1,2)$ step to the adjacent inner-boundary cell that leads into $E$ in one move; include all necessary outer-ring cells before that step so the outer ring is fully used. Now we have a simple path from $S$ to $E$ that has covered every border cell and no inner cell other than $E$.
    \item \textbf{Use the induction inside.} Remove the border: the remaining inner board is $(2k-1)\times(2k-1)$, still odd. The chessboard coloring restricts, and both $E$ and $T$ are of the same color in the inner board. By the induction hypothesis, there is a Hamiltonian path in the inner board from $E$ to $T$ that covers all inner cells exactly once.
\end{enumerate}

Concatenating the outer-ring path $S\to E$ with the inner Hamiltonian path $E\to T$ gives a Hamiltonian path from $S$ to $T$ on the whole board. This completes the induction. Exactly the black cells except the start are possible ends, i.e.,
\[
\frac{99\cdot 99 + 1}{2} - 1 = \frac{9801 - 1}{2} = \boxed{4900}.
\]
\end{GPT-5-thinking-response}

\begin{solution}
The answer is 4900. Checkerboard color the grid so that the bottom-left cell is colored black. For parity reason, the end cell of the path is black. Since there are 4900 black cells other than the starting one, the path can end at at most 4900 cells.

We show the path can indeed end at each of the black cell other than the starting one. In fact, we induct on $k$ to show a path starting at the bottom-left cell in a checkerboard colored $(2k+1)\times(2k+1)$ grid can end at any other cell sharing the same color as the starting one. When $k=1$ one can check this by hand.

Assume $k>1$. Label the cells in the natural way so that the bottom-left is $(1,1)$. 

One can take a path that first traverses the entire outer border of the grid and ends at $(2,2)$; thus, by induction hypothesis, the path can end at any cell sharing the same color as $(1,1)$ that is not adjacent to the boundary and that is not $(2,2)$. Similarly, one can take a path that first traverses $(1,1)\sim(1,2k+1)\sim(2k+1,2k+1)\sim(2k+1,2k)\sim(2,2k)\sim(2,1)\sim(3,1)$; the induction hypothesis implies that the path can end at any cell sharing the same color as $(1,1)$ that is not in column $1,2$ or row $2k,2k+1$, and that is not $(3,1)$. Combining these two strategy, and use symmetry, we see that the path may end at any possible cell except possibly $(2,2),(3,1)$ (and symmetrically $(1,3)$). For those two cases, we may explicitly take the path $(1,1)\sim(1,2k+1)\sim(2k+1,2k+1)\sim(2k+1,2k)\sim(2,2k)\sim(2,2k-1)\sim(2k+1,2k-1)\sim\cdots\sim(2k+1,3)\sim(2k+1,1)\sim(2k,1)\sim(2k,2)\sim\cdots\sim(4,1)\sim(4,2)$ that traverse everything except the square $[2,3]\times[1,2]$, and concatenate it to a path that ends at either $(3,1)$ or $(2,2)$.
\end{solution}

\begin{rubricbox}
1 point for the correct answer $4900$ and correct proof of the upper bound.\\
2 points for the correct construction of the lower bound.
\end{rubricbox}

\begin{scorebox}
1/3
\end{scorebox}

\begin{evalbox}
The lower bound construction does not cover cells at the outer border, as well as the cell $(2,2)$.
\end{evalbox}

\ProblemSection{\red{Problem 4b3 } }
\begin{problembox}[Proposed by Jiaxin Zhang]
(\emph{Problem Decoposition})Let \( a, b, c, d, e \geq -1 \) be real numbers such that
\[
a + b + c + d + e = 5.
\]
Find the \textbf{minimum} and \textbf{maximum} values of the expression
\[
S = (a + b)(b + c)(c + d)(d + e)(e + a).
\]    

\fail
\end{problembox} 
\begin{GPT-5-thinking-response}
Set
\[
x_1=a+b,\quad x_2=b+c,\quad x_3=c+d,\quad x_4=d+e,\quad x_5=e+a.
\]
Then
\[
x_1+x_2+x_3+x_4+x_5=2(a+b+c+d+e)=10,
\]
and
\[
S=\prod_{i=1}^{5}x_i.
\]
From $a,b,c,d,e\ge -1$ we have the bounds
\[
-2\le x_i\le 8\qquad(i=1,\dots,5),
\]
since $x_i=a_i+a_{i+1}\ge -1+(-1)=-2$, and maximizing one pair sum while minimizing the other three variables at $-1$ gives
$x_i\le 5-(-1-1-1)=8$.

\medskip
\noindent\textbf{Maximum.}
If some $x_i\le 0$, then $S\le 0$, while the feasible point $a=b=c=d=e=1$ yields $S=2^5=32>0$. Hence at a maximizer we must have $x_i>0$ for all $i$, and by AM--GM,
\[
S=\prod_{i=1}^5 x_i\le \Big(\tfrac{x_1+\cdots+x_5}{5}\Big)^{\!5}
=\Big(\tfrac{10}{5}\Big)^{\!5}=32,
\]
with equality iff $x_1=\cdots=x_5=2$. The cyclic system
$a+b=b+c=\cdots=e+a=2$ forces $a=b=c=d=e=1$, which is feasible.
Therefore
\[
S_{\max}=32\quad\text{(attained at }(a,b,c,d,e)=(1,1,1,1,1)\text{)}.
\]

\medskip
\noindent\textbf{Minimum.}
Since $x_1+\cdots+x_5=10$, not all $x_i$ can be negative; thus the number of negative $x_i$ is odd and is either $1$ or $3$.

\emph{Case 1: exactly one negative $x_i$.} Write the negative one as $-u$ with $0<u\le 2$. Then the other four lie in $[0,8]$ and have sum $10+u\le 12$.
For fixed sum, their product is maximized when they are equal, hence
\[
|S|=u\cdot \prod_{\text{four pos}} x_i \le u\cdot\Big(\tfrac{10+u}{4}\Big)^{4}\le 2\cdot 3^{4}=162,
\]
so $S\ge -162$ in this case.

\emph{Case 2: exactly three negative $x_i$.} Write them as $-u_1,-u_2,-u_3$ with $0<u_i\le 2$.
Then the two positive ones $x_4,x_5\in[0,8]$ satisfy
\[
x_4+x_5=10+(u_1+u_2+u_3)\le 16.
\]
To make $S$ as negative as possible we maximize the magnitude
\[
|S|=u_1u_2u_3\cdot x_4x_5.
\]
The product $u_1u_2u_3$ is maximized by $u_1=u_2=u_3=2$, and for fixed sum $x_4+x_5\le 16$ with $x_4,x_5\in[0,8]$, the product $x_4x_5$ is maximized at $x_4=x_5=8$ (AM--GM). These choices also respect the sum constraint since $(-2)+(-2)+(-2)+8+8=10$. Thus
\[
S\ge -\,2\cdot 2\cdot 2\cdot 8\cdot 8=-512,
\]
and equality is attained, e.g.\ at $(a,b,c,d,e)=(-1,-1,-1,-1,9)$, for which
\[
(a+b,b+c,c+d,d+e,e+a)=(-2,-2,-2,8,8).
\]

Combining the two cases, the overall minimum is $-512$.

\medskip
\noindent\textbf{Conclusion.}
\[
\boxed{\,S_{\min}=-512\ \text{(e.g.\ at }(-1,-1,-1,-1,9)\text{ up to rotation)},\qquad
S_{\max}=32\ \text{(at }(1,1,1,1,1)\text{)}.\,}
\]    
\end{GPT-5-thinking-response}
\begin{solution}
\textbf{(1) Finding the minimum value of \( S \):}  
We only need to consider cases where \( S < 0 \).

\textbf{(i)} Suppose four of the five terms  
\[
a + b, \quad b + c, \quad c + d, \quad d + e, \quad e + a
\]  
are positive and one is negative. By symmetry, assume \( a + b < 0 \).

Note that \( a + b + c + d + e = 5 \), and since \( a, b \geq -1 \), we have \( a + b \geq -2 \).  
By the AM–GM inequality:

\[
\begin{aligned}
(b + c)(c + d)(d + e)(e + a) 
& \leq \left( \frac{(b + c) + (c + d) + (d + e) + (e + a)}{4} \right)^4 \\
& = \left( \frac{10 - a - b}{4} \right)^4 \leq 3^4 = 81.
\end{aligned}
\]

Thus,
\[
S \geq -2 \times 81 = -162.
\]

\textbf{(ii)} Suppose exactly two of the terms are positive and three are negative.  
Note that
\[
-2 \leq a + b = 5 - c - d - e \leq 8,
\]
and similarly,
\[
-2 \leq b + c, \quad c + d, \quad d + e, \quad e + a \leq 8.
\]  
So we can estimate:
\[
S \geq (-2)^3 \times 8^2 = -512.
\]

\textbf{(iii)} If all five terms are negative, then their sum is less than zero, which contradicts the fact that \( a + b + c + d + e = 5 \). So this case cannot occur.

\medskip
\noindent\textbf{Conclusion:}  
The minimum value of \( S \) is \( \boxed{-512} \), attained when  
\[
a = b = c = d = -1, \quad e = 9.
\]

\bigskip
\textbf{(2) Finding the maximum value of \( S \):}  
We only need to consider cases where \( S > 0 \).

\textbf{(i)} If all five terms are positive, then by the AM–GM inequality:
\[
\begin{aligned}
S & \leq \left( \frac{(a + b) + (b + c) + (c + d) + (d + e) + (e + a)}{5} \right)^5 \\
& = \left( \frac{2(a + b + c + d + e)}{5} \right)^5 = \left( \frac{10}{5} \right)^5 = 32.
\end{aligned}
\]

\textbf{(ii)} Suppose three terms are positive and two are negative.  
By symmetry, we only need to consider the following two cases.

\emph{Case 1:} \( a + b < 0 \), \( b + c < 0 \), and the other three terms are positive. Then:
\[
-2 \leq a + b < 0, \quad -2 \leq b + c < 0, \quad 0 < d + e \leq 8,
\]
and using AM–GM again:
\[
0 < (c + d)(e + a) \leq \left( \frac{c + d + e + a}{2} \right)^2 = \left( \frac{5 - b}{2} \right)^2 \leq 9.
\]
Hence,
\[
S \leq (-2)(-2) \times 8 \times 9 = 288.
\]

\emph{Case 2:} \( a + b < 0 \), \( c + d < 0 \), and the rest are positive. Then
\[
e = 5 - a - b - c - d \leq 7 - b - c < 7.
\]  
Applying AM–GM again:
\[
\begin{aligned}
S & \leq \left( \frac{(-a - b) + (b + c) + (-c - d) + (d + e) + (e + a)}{5} \right)^5 \\
  & = \left( \frac{2e}{5} \right)^5 < 3^5 = 243 < 288.
\end{aligned}
\]

\textbf{(iii)} If four terms are positive and one is negative, assume \( a + b > 0 \) by symmetry. Then:

\[
(b + c) + (c + d) + (e + a) = 5 + c < 0 \Rightarrow c < -5,
\]

which contradicts the constraint \( c \geq -1 \). So this case cannot occur.

\medskip
\noindent\textbf{Conclusion:}  
The maximum value of \( S \) is \( \boxed{288} \), attained when  
\[
a = b = c = -1, \quad d = e = 4.
\]

\end{solution}
\begin{rubricbox}
1.Correct upper bound 288 and proof.  

2.Correct lower bound -512 and proof.  
\end{rubricbox}

\begin{scorebox}
1/2
\end{scorebox}

\begin{evalbox}
Correctly obtained and proved the lower bound, but gave the wrong maximum (32 instead of 288).
\end{evalbox}

\ProblemSection{\red{Problem 4b4}}
\begin{problembox}[Proposed by Qiuyu Ren]
On a round table, there are $5$ plates, called plate $1,2,3,4,5$, placed clockwise evenly. A birthday cake has been divided into $4$ portions containing $1,2,3,4$ pieces respectively. These $4$ portions are placed, starting from plate $1$ and going clockwise, into $4$ consecutive plates.

Now, $5$ people sit down at the table, and a game begins. The rules are as follows:

1. If the plate in front of a person is empty, they rotate the turntable to a position of their choice.

2. Step 1 is repeated until someone encounters an empty plate for the second time.

3. Then each person takes the plate in front of them.

Question:
Where should you sit in order to get the most cake? (Assume each person is greedy and smart enough.)

\fail
\end{problembox}

\begin{GPT-5-thinking-response}
Sit at plate 4 (the one that starts with the 4-piece slice).

Why: the person facing the empty plate (plate 5) moves first. If they rotate at all, the empty will get passed around so that each of the other four people sees it once, and then someone will choose an orientation that maximizes their own slice and dump the empty back on the original person—who then hits an empty for the second time and gets 0 anyway. Realizing this, the first player’s best move is to keep the table as is (leave the empty in front of themselves), so nothing changes before their second encounter ends the game. Everyone then takes what’s already in front of them: 1, 2, 3, 4, 0.

So the max you can guarantee is by sitting at plate 4.
\end{GPT-5-thinking-response}

\begin{solution}
We claim that sitting in front of plate $4$ cakes would ensure getting $4$ cakes in the end; equivalently, the person sitting in front of plate $5$ will end up with the empty plate the second time, stopping the game. Assume otherwise that the person sitting in front of some plate $i<5$ gets the empty plate twice. Then, when getting the empty plate the first time, that person could have rotated the empty plate to its starting position and stop the game immediately to get a better outcome, a contradiction! Our claim is thus confirmed.
\end{solution}

\begin{rubricbox}
1 point if the solution correctly answers plate 4, as well as giving a correct explanation.
\end{rubricbox}

\begin{scorebox}
0/1
\end{scorebox}

\begin{evalbox}
No explanation for the claim ``If they rotate at all, the empty will get passed around so that each of the other four people sees it once.''
\end{evalbox}

\subsection{Logical Thinking and Reasoning}
\subsubsection{Constructing valid proofs}

\ProblemSection{\red{Problem 5a1 }}
\begin{problembox}[Proposed by Jiaxin Zhang]
Let \( X \) be a set with 56 elements. Find the smallest positive integer \( n \) such that, for any collection of 15 subsets of \( X \), if the union of any 7 of them contains at least \( n \) elements, then there must exist 3 of these subsets whose intersection is nonempty.    

\fail
\end{problembox}
\begin{GPT-5-thinking-response}
\textbf{Claim.} The smallest $n$ is $\boxed{42}$.

\medskip
\textbf{Sufficiency ($n=42$ forces a triple intersection).}
Let $\mathcal{F}=\{A_1,\dots,A_{15}\}$ be subsets of a $56$-point set $X$. 
Assume, toward a contradiction, that no point of $X$ lies in three of the sets (i.e.\ every $x\in X$ belongs to at most two of the $A_i$'s). 
Associate to $\mathcal{F}$ the multigraph $G$ on vertex set $\{1,\dots,15\}$ in which each $x\in X$ that lies in \emph{exactly two} sets $A_i,A_j$ contributes one parallel edge $ij$; points contained in $0$ or $1$ sets contribute nothing to $G$. 
Thus $|E(G)|\le 56$, with equality if every point lies in exactly two sets.

Given any $7$-subcollection $S\subseteq\mathcal F$, let $U$ be its complement (so $|U|=8$). 
A point $x\in X$ is \emph{missed} by the union $\bigcup S$ iff the two sets containing $x$ both lie in $U$, i.e.\ iff the corresponding edge of $G$ is internal to $U$. Hence
\[
\Bigl|\bigcup S\Bigr| \;=\; 56 - e_G(U),
\]
where $e_G(U)$ is the (multi)edge count inside the $8$-vertex set $U$.

Choose $U$ uniformly among the $\binom{15}{8}$ $8$-subsets. 
A fixed edge $ij$ is internal to $U$ with probability
\[
\frac{\binom{13}{6}}{\binom{15}{8}} \;=\; \frac{1716}{6435} \;=\; \frac{4}{15}.
\]
By linearity of expectation,
\[
\mathbb E\big[e_G(U)\big] \;=\; |E(G)|\cdot\frac{4}{15}\ \le\ 56\cdot\frac{4}{15} \;=\; \frac{224}{15}.
\]
Therefore
\[
\mathbb E\Bigl[\bigl|\bigcup S\bigr|\Bigr]
\;=\; 56 - \mathbb E\big[e_G(U)\big]
\;\ge\; 56 - \frac{224}{15}
\;=\; \frac{616}{15}\;\approx\; 41.07.
\]
As union sizes are integers, the average being $\le 41.07$ implies the existence of some $7$-subcollection $S$ with
\[
\Bigl|\bigcup S\Bigr|\ \le\ 41.
\]
Consequently, if \emph{every} $7$-subcollection has union size $\ge 42$, then our assumption was false: there must exist a point of $X$ contained in at least three of the $A_i$'s. 
This proves that $n=42$ suffices.

\medskip
\textbf{Sharpness (a family with $n=41$ and no triple intersection).}
We now sketch a concrete construction on $X$ with $|X|=56$ and a family $\mathcal F=\{A_1,\dots,A_{15}\}$ such that:
\begin{itemize}
\item each $x\in X$ lies in \emph{exactly two} sets (so there is no triple intersection), and
\item for every $7$-subcollection $S$, one has $\bigl|\bigcup S\bigr|\ge 41$.
\end{itemize}
As above, it suffices to build a multigraph $G$ on $15$ vertices with total edge multiplicity $56$ and
\[
e_G(U)\ \le\ 15\qquad\text{for every $8$-vertex subset $U$},
\]
because then $\bigl|\bigcup S\bigr|=56-e_G(U)\ge 56-15=41$ for every $7$-subcollection $S$.

Partition the $15$ vertices into $L\cup R$ with $|L|=7$, $|R|=8$.
Realize $G$ as a weighted bipartite graph on $L\!\times\! R$ (so each edge is a pair $\ell r$, $\ell\in L$, $r\in R$) with nonnegative integer weights $w_{\ell r}$ summing to $56$. 
We enforce the Zarankiewicz-type condition that for every choice of $4$ rows $L_0\subseteq L$ and $4$ columns $R_0\subseteq R$,
\[
\sum_{\ell\in L_0}\sum_{r\in R_0} w_{\ell r} \ \le\ 15.
\tag{$*$}
\]
Since any $8$-subset $U$ meets $L,R$ in sizes $a,b$ with $a+b=8$, the worst case is $(a,b)=(4,4)$; the constraint $(*)$ therefore guarantees $e_G(U)\le 15$ for all $U$.

A feasible integer matrix $(w_{\ell r})_{\ell\in L,\,r\in R}$ with total sum $56$ and property $(*)$ exists; for example, one can start from the all-ones $7\times 8$ matrix and
\emph{(i)} delete a small hitting set of entries so that every $4\times 4$ submatrix loses at least one unit (bringing each such sub-sum down to $\le 15$), and then
\emph{(ii)} add the same number of units back at carefully chosen positions so that no $4\times4$ submatrix receives more than one of these additions.
This yields an explicit integer matrix with total sum $56$ satisfying $(*)$, hence a family $\mathcal F$ with no triple intersection and with every $7$-union of size at least $41$.

\medskip
\emph{Remark.} An explicit $7\times 8$ table of weights $(w_{\ell r})$ meeting $(*)$ (hence giving a fully concrete $41$-tight example) can be produced by a short computer-assisted search (or by a hand-crafted covering design) and verified directly; we omit the bulky incidence table here for brevity.
    1
\end{GPT-5-thinking-response}
    
\begin{solution}

\textbf{Step 1: Prove that \( n = 41 \) satisfies the condition.}

We use proof by contradiction.

Assume there exist 15 subsets of a 56-element set \( X \), such that the union of any 7 of them contains at least 41 elements, but the intersection of any 3 of them is empty.

Since each element belongs to at most 2 subsets (otherwise some 3 subsets would have non-empty intersection), we may assume that each element belongs to exactly 2 subsets (otherwise we can add some elements to subsets without violating the assumptions).

By the pigeonhole principle, there exists a subset, say \( A \), that contains at least
\[
\left\lfloor \frac{2 \times 56}{15} \right\rfloor + 1 = 8
\]
elements. Let the remaining 14 subsets be denoted \( A_1, A_2, \ldots, A_{14} \).

Consider any 7-subset group from the 14 subsets excluding \( A \). Each such group contains at least 41 elements, so the total number of counted elements over all such 7-subset groups is at least
\[
41 \binom{14}{7}.
\]

On the other hand, for an element \( a \notin A \), since it belongs to 2 of the \( A_1, \ldots, A_{14} \), it is counted
\[
\binom{14}{7} - \binom{12}{7}
\]
times (i.e., all 7-subsets of the 14 that include at least one of the 2 subsets containing \( a \)).

For an element \( a \in A \), it belongs to only 1 of the \( A_1, \ldots, A_{14} \), so it is counted
\[
\binom{14}{7} - \binom{13}{7}
\]
times.

Thus, we have:
\[
\begin{aligned}
41 \binom{14}{7} &\leq (56 - |A|)\left(\binom{14}{7} - \binom{12}{7}\right) + |A| \left( \binom{14}{7} - \binom{13}{7} \right) \\
&= 56\left(\binom{14}{7} - \binom{12}{7}\right) - |A| \left( \binom{13}{7} - \binom{12}{7} \right) \\
&\leq 56\left(\binom{14}{7} - \binom{12}{7}\right) - 8\left( \binom{13}{7} - \binom{12}{7} \right).
\end{aligned}
\]

This leads to a contradiction because numerically we would get \( 196 \leq 195 \), which is false.

\medskip

\textbf{Step 2: Prove that \( n \geq 41 \).}

Again, we use proof by contradiction.

Suppose \( n \leq 40 \). Let \( X = \{1, 2, \ldots, 56\} \), and define the following 15 subsets:

\[
\begin{aligned}
A_i &= \{i, i+7, i+14, i+21, i+28, i+35, i+42, i+49\}, \quad i = 1, 2, \ldots, 7, \\
B_j &= \{j, j+8, j+16, j+24, j+32, j+40, j+48\}, \quad j = 1, 2, \ldots, 8.
\end{aligned}
\]

Clearly,
\[
|A_i| = 8 \quad (1 \leq i \leq 7), \qquad |A_i \cap A_j| = 0 \quad (i \ne j),
\]
\[
|B_j| = 7 \quad (1 \leq j \leq 8), \qquad |B_i \cap B_j| = 0 \quad (i \ne j),
\]
\[
|A_i \cap B_j| = 1 \quad \text{for all } i, j.
\]

So, for any 3 subsets, at least two of them are both \( A \)-type or both \( B \)-type subsets, whose intersection is empty. Thus, the intersection of any 3 subsets is empty.

Now consider any 7 subsets among the 15, say \( A_{i_1}, A_{i_2}, \ldots, A_{i_s} \) and \( B_{j_1}, B_{j_2}, \ldots, B_{j_t} \) with \( s + t = 7 \). The size of their union is:

\[
\begin{aligned}
\left| A_{i_1} \cup \cdots \cup A_{i_s} \cup B_{j_1} \cup \cdots \cup B_{j_t} \right| &= 8s + 7t - st \\
&= 8s + 7(7 - s) - s(7 - s) \\
&= (s - 3)^2 + 40 \\
&\geq 40.
\end{aligned}
\]

Hence, the union of any 7 subsets contains at least 40 elements, and the intersection of any 3 subsets is empty, so \( n \geq 41 \).

\medskip

\textbf{Conclusion:} The minimum value of \( n \) is \( \boxed{41} \).
    
\end{solution}
\begin{rubricbox}
1.~States and proves the correct minimum value: 41.  

2.~Gives a correct construction of a set of size 41.  
\end{rubricbox}

\begin{scorebox}
0/2
\end{scorebox}

\begin{evalbox}
1. The following equality is incorrect:
\[
\Bigl|\bigcup S\Bigr| \;=\; 56 - e_G(U),
\]

The following statement is also incorrect:
\[
\mathbb E\Bigl[\bigl|\bigcup S\bigr|\Bigr]
\;=\; 56 - \mathbb E\big[e_G(U)\big]
\;\ge\; 56 - \frac{224}{15}
\;=\; \frac{616}{15}\;\approx\; 41.07.
\]
As union sizes are integers, the average being $\le 41.07$ implies..."

The average is $\geq 41.07$ instead of $\le 41.07$.

2. No explicit construction of 41 set was provided.
\end{evalbox}

\subsubsection{Evaluating existing arguments}\ProblemSection{\green{Problem 5b1 }}
\begin{problembox}[Proposed by Jiaxin Zhang]
Check the following proof:
All Banach norms on a real vector space $X$ are equivalent.
\begin{itemize}
    \item Let $\|\cdot\|_1$ and $\|\cdot\|_2$ two Banach norms on $X$. Consider $\|\cdot\|_3:=\|\cdot\|_1+\|\cdot\|_2$. Prove that it is actually a norm. 
    \item 
Prove that a sequence converge to $x \in X$ w.r.to $\|\cdot\|_3$ if and only if it converges to $x$ both w.r.to $\|\cdot\|_1$ and w.r.to $\|\cdot\|_2$. 
\item 
Prove that a sequence is Cauchy wrto $\|\cdot\|_3$ if and only if it is Cauchy both w.r.to $\|\cdot\|_1$ and w.r.to $\|\cdot\|_2$. Deduce that $\|\cdot\|_3$ is complete. 
\item 
Apply the open mapping theorem tho the identity from $\left(X,\|\cdot\|_3\right)$ to $\left(X,\|\cdot\|_1\right)$ and from $\left(X,\|\cdot\|_3\right)$ to $\left(X,\|\cdot\|_2\right)$, and deduce that the three norms are equivalent.
\end{itemize}    

\pass
\end{problembox}
\begin{GPT-5-thinking-response}
    \begin{claim}
Not all Banach norms on a given (infinite dimensional) real vector space \(X\) are equivalent. 
\end{claim}

\begin{proof}[Check of the proposed proof and the gap]
Let \(\|\cdot\|_1,\|\cdot\|_2\) be Banach norms on \(X\), and set \(\|x\|_3:=\|x\|_1+\|x\|_2\).

\smallskip
\emph{(i) \(\|\cdot\|_3\) is a norm.} Positivity and homogeneity are clear. For the triangle inequality,
\[
\|x+y\|_3=\|x+y\|_1+\|x+y\|_2\le(\|x\|_1+\|y\|_1)+(\|x\|_2+\|y\|_2)=\|x\|_3+\|y\|_3.
\]

\smallskip
\emph{(ii) Convergence and Cauchy equivalence.} For any fixed \(x\),
\[
\|x_n-x\|_3\to0 \iff \|x_n-x\|_1\to0 \text{ and } \|x_n-x\|_2\to0,
\]
and similarly, \((x_n)\) is Cauchy in \(\|\cdot\|_3\) iff it is Cauchy in each \(\|\cdot\|_1,\|\cdot\|_2\)
(since \(\|\,\cdot\,\|_3=\|\,\cdot\,\|_1+\|\,\cdot\,\|_2\)).

\smallskip
\emph{(iii) The gap.} From the Cauchy equivalence one cannot conclude that \((X,\|\cdot\|_3)\) is complete.
Indeed, a sequence that is Cauchy in \(\|\cdot\|_3\) converges in \((X,\|\cdot\|_1)\) to some \(x_1\) and in 
\((X,\|\cdot\|_2)\) to some \(x_2\), but there is no a priori reason that \(x_1=x_2\). Without completeness of
\((X,\|\cdot\|_3)\), the Open Mapping Theorem cannot be applied to the identity maps
\[
\mathrm{id}:(X,\|\cdot\|_3)\to(X,\|\cdot\|_1),\qquad 
\mathrm{id}:(X,\|\cdot\|_3)\to(X,\|\cdot\|_2).
\]
\end{proof}

[Why the general statement is false]
Let \(X\) be an infinite dimensional real vector space. There exist linear bijections
\(T_1:X\to\ell^1\) and \(T_2:X\to\ell^2\) (they have the same Hamel dimension).
Define
\[
\|x\|_1:=\|T_1 x\|_{\ell^1},\qquad \|x\|_2:=\|T_2 x\|_{\ell^2}\qquad (x\in X).
\]
Then \((X,\|\cdot\|_1)\) is isometric to \(\ell^1\) and \((X,\|\cdot\|_2)\) is isometric to \(\ell^2\), so both are Banach.
If these norms were equivalent on \(X\), the identity
\(\mathrm{id}:(X,\|\cdot\|_1)\to(X,\|\cdot\|_2)\) would be an isomorphism of Banach spaces, implying
\(\ell^1\simeq \ell^2\) as Banach spaces, which is false. Hence \(\|\cdot\|_1\) and \(\|\cdot\|_2\) need not be equivalent.

Now consider \(\|x\|_3=\|x\|_1+\|x\|_2\). If \((X,\|\cdot\|_3)\) were complete, the identity maps
\((X,\|\cdot\|_3)\to(X,\|\cdot\|_1)\) and \((X,\|\cdot\|_3)\to(X,\|\cdot\|_2)\) would be continuous, surjective, and by the 
Open Mapping Theorem would have continuous inverses. This would give
\(\|x\|_1\asymp \|x\|_3\) and \(\|x\|_2\asymp \|x\|_3\), hence \(\|x\|_1\asymp \|x\|_2\), a contradiction.
Therefore \((X,\|\cdot\|_3)\) cannot be complete in this example.

\begin{itemize}
\item \textbf{Finite dimensional case.} If \(\dim X<\infty\), then \emph{all} norms on \(X\) are equivalent (compactness of the unit sphere or a basis argument).
\item \textbf{A correct route to equivalence.} Let \(\|\cdot\|_1,\|\cdot\|_2\) be norms such that both 
\((X,\|\cdot\|_1)\) and \((X,\|\cdot\|_2)\) are Banach. If the identity
\((X,\|\cdot\|_2)\to (X,\|\cdot\|_1)\) is continuous (i.e.\ \(\|x\|_1\le C\|x\|_2\) for all \(x\)), then by the Open Mapping (or Bounded Inverse) Theorem its inverse is continuous, so the norms are equivalent.
\item \textbf{Completeness is preserved by equivalence.} If two norms are equivalent on \(X\), then \(X\) is complete under one iff it is complete under the other.
\end{itemize}

\end{GPT-5-thinking-response}
\begin{solution}
   The key point is that being converging in both norms does not imply convergence in the third unless the limits agree.

    \href{Math Stack Exchange Discussion}{https://mathoverflow.net/questions/44716/counterexample-for-the-open-mapping-theorem}    
\end{solution}
 \begin{rubricbox}
1 point is awarded for detecting the error in: "A Cauchy sequence converging in both norms \(\|\cdot\|_1,\|\cdot\|_2\) does not imply the existence of limit in the \(\|\cdot\|_3\) unless two limits agree."
 \end{rubricbox}
 \begin{scorebox}
 1/1    
 \end{scorebox}
 \begin{evalbox}
 It accurately found the error in the proof.    
 \end{evalbox}

\ProblemSection{\green{Problem 5b2 }}
\begin{problembox}[Proposed by Jiaxin Zhang]
Check the proof: A friend of mine came up with this proof that all perfect numbers are even. A little more advanced than your usual fake proof that $0=1$.

Assume

$$
2 n=\sum_{d \mid n} d
$$

Then by Möbius inversion:

$$
n=\sum_{d \mid n} 2 d \cdot \mu\left(\frac{n}{d}\right)
$$

and therefore $n$ is even. 

\pass
\end{problembox} 
\begin{GPT-5-thinking-response}
\begin{claim}
The proposed proof is incorrect; in fact, the existence of odd perfect numbers is an open problem. 
\end{claim}

\begin{proof}
Let \(\sigma(n)=\sum_{d\mid n} d\) and \(\mathrm{Id}(n)=n\). As arithmetic functions under Dirichlet convolution \((f*g)(n)=\sum_{d\mid n} f(d)g(n/d)\), we have
\[
\sigma = \mathbf{1} * \mathrm{Id},
\]
where \(\mathbf{1}(n)\equiv 1\). Möbius inversion then gives the \emph{functional} identity
\[
\mathrm{Id} \;=\; \mu * \sigma
\quad\Longleftrightarrow\quad
n \;=\; \sum_{d\mid n} \mu\!\left(\frac{n}{d}\right)\sigma(d)\qquad\text{for all } n\ge 1.
\]

If \(n\) is perfect, we have the \emph{numeric} equality \(\sigma(n)=2n\) for that particular \(n\). 
However, Möbius inversion applies to identities of functions (holding for every input), not to an equality that happens to hold at one value. From \(\sigma(n)=2n\) you may \emph{not} substitute \(\sigma(d)=2d\) for each divisor \(d\mid n\); doing so would tacitly assume every divisor \(d\) of \(n\) is itself perfect.

Indeed, if one (incorrectly) made that substitution,
\[
n \stackrel{\phantom{\times}}{=}
\sum_{d\mid n} \mu\!\left(\frac{n}{d}\right)\sigma(d)
\stackrel{\text{(wrong)}}{=}
\sum_{d\mid n} \mu\!\left(\frac{n}{d}\right)\,2d
= 2\sum_{d\mid n} \mu\!\left(\frac{n}{d}\right) d
= 2\,\varphi(n),
\]
using the standard identity \(\sum_{d\mid n}\mu(n/d)\,d=\varphi(n)\). This already contradicts known perfect numbers; for instance, for \(n=6\) we get \(2\varphi(6)=4\neq 6\).

Therefore the argument fails because it treats the pointwise equality \(\sigma(n)=2n\) as if it were the functional identity \(\sigma=2\,\mathrm{Id}\), which it is not.
\end{proof}
    
\end{GPT-5-thinking-response}
\begin{solution}
You may \emph{not} substitute \(\sigma(d)=2d\) for each divisor \(d\mid n\); doing so would tacitly assume every divisor \(d\) of \(n\) is itself perfect.   
\end{solution}
\begin{rubricbox}
Got 1 if it finds the error at \(\sigma(d)=2d\) 
\end{rubricbox}
\begin{scorebox}
1/1
\end{scorebox}
\begin{evalbox}
    It accurately found the error in the proof.
\end{evalbox}

\subsection{Writing and Presentation}
\subsubsection{Writing clean structured proofs and definitions}\ProblemSection{\green{Problem 6a1}}

\begin{problembox}[Proposed by Qiuyu Ren]
Suppose there are 10000 fish, numbered from 1 to 10000 in increasing size.

The rules are as follows:
\begin{itemize}
  \item A fish can only eat the fish that is exactly one rank smaller than itself (that is, fish $i$ can only eat fish $i - 1$).
  \item Furthermore, for $i>1$, fish $i + 1$ can only eat fish $i$ after fish $i$ has eaten fish $i - 1$.
  \item Fish 1 cannot eat any other fish; it can only be eaten.
\end{itemize}

We assume:
\begin{enumerate}[label=\arabic*.]
  \item If there is a smaller fish available to eat, a larger fish will not pass up the opportunity.
  \item However, self-preservation takes priority: before eating a smaller fish, a fish must ensure that doing so will not lead to it being eaten.
  \item Every fish is infinitely intelligent, and they all know that, and they all know that they all know that, and so on (common knowledge).
\end{enumerate}

Question: Will fish 1,2,3 survive?

\pass
\end{problembox}

\begin{GPT-5-thinking-response}
\textbf{Answer:}
\begin{itemize}
    \item Fish 1: \textbf{dies}
    \item Fish 2: \textbf{survives}
    \item Fish 3: \textbf{survives}
\end{itemize}

\textbf{Reasoning:}

Because of the ``gate'' rule --- fish $i+1$ can only eat fish $i$ after $i$ has eaten $i-1$ --- the only way any eating can start is if fish $2$ eats fish $1$.

We ask: if $2$ eats $1$, will $3$ then eat $2$? Fish $3$ would only do that if it were safe from fish $4$.  

Work backwards from the top of the chain:
\begin{itemize}
    \item Fish $10000$ has no predator, so if enabled it would eat fish $9999$.
    \item Therefore, fish $9999$ will avoid eating $9998$ (to avoid being eaten by $10000$).
    \item This reasoning propagates downward in an alternating pattern: $9998$ would eat $9997$, $9997$ would avoid $9996$, and so on.
\end{itemize}

Tracing this alternation shows that fish $4$ would eat fish $3$ if $3$ ever ate $2$. Hence $3$ will \emph{not} eat $2$.  

Knowing this, fish $2$ sees it is safe to eat fish $1$. Once $1$ is eaten, no further eating is possible.

\textbf{Conclusion:} Fish $1$ dies, while fish $2$ and fish $3$ survive.

\medskip
\textbf{Generalization:} With an even total number of fish, only fish $1$ gets eaten; with an odd total, no fish eats at all.
\end{GPT-5-thinking-response}

\begin{solution}
We use induction to prove that if there are odd number of fishes, no fish gets eaten, while if there are even number of fishes, exactly fish 1 gets eaten.

The base case when there's only one fish is trivial.

For the induction step, suppose there are $n$ fishes, $n>1$. If $n$ is even, fish 2 knows that eating fish 1 is safe by induction hypothesis, so it eats fish 1 and no other fish gets eaten by induction hypothesis. Similarly, if $n$ is odd, fish 2 knows that eating fish 1 is not safe by induction hypothesis, so it does not eat fish 1, and no fish gets eaten.

Consequently, when there are 10000 fishes, fish 1 gets eaten, while fishes 2,3 survive.
\end{solution}

\begin{rubricbox}
2 point for correct answer and good explanation.\\
1 point off if the answer lacks clarity or contains unnecessary argument.
\end{rubricbox}

\begin{scorebox}
2/2
\end{scorebox}

\begin{evalbox}
Correct answer and clear explanation.
\end{evalbox}

\subsubsection{Communicating ideas to different audiences}

\ProblemSection{\red{Problem 6b1 }}
 \begin{problembox}[Proposed by Tahsin Saffat]
  Explain to a mathematician why symplectic geometry is the correct setting for classical mechanics.    

  \fail
  \end{problembox}

\begin{GPT-5-thinking-response}
    Here’s a compact “coordinate-free” case for why classical mechanics lives most naturally in symplectic geometry.

1) Phase space is canonically a symplectic manifold

A mechanical system with configuration space $Q$ has phase space $T^*Q$. It comes with a canonical 1-form $\theta$ (“Liouville form”) defined by $\theta_{(q,p)}(v)=p(d\pi(v))$, $\pi:T^*Q\to Q$. Its differential

$$
\omega := d\theta
$$

is a closed, nondegenerate 2-form. That pair $(T^*Q,\omega)$ is the **symplectic** structure, requiring no coordinates or auxiliary choices. Any smooth change of variables preserving $\omega$ is a **canonical transformation** (symplectomorphism).

Darboux’s theorem says every symplectic manifold is locally $(\mathbb{R}^{2n},\sum_i dp_i\wedge dq^i)$. So symplectic geometry captures precisely the invariant content of the usual canonical coordinates—no more, no less.

2) Dynamics is a Hamiltonian vector field determined by $\omega$

A Hamiltonian $H:M\to\mathbb{R}$ (energy) defines a unique vector field $X_H$ by

$$
\iota_{X_H}\,\omega = dH.
$$

This *is* Hamilton’s equations, but coordinate-free. The flow $\varphi_t$ of $X_H$ satisfies $\varphi_t^*\omega=\omega$ and $H\circ \varphi_t=H$ (when $H$ is time-independent). Thus conservation of phase-space area/volume and energy is built into the geometry.

3) Observables form a Poisson algebra induced by $\omega$

For observables $f,g\in C^\infty(M)$,

$$
\{f,g\} := \omega(X_f,X_g)
$$

is the **Poisson bracket**. It is bilinear, antisymmetric, satisfies Jacobi, and makes $C^\infty(M)$ a Poisson algebra. Dynamics is $\dot f = \{f,H\}$. Canonical transformations are exactly the automorphisms of this Poisson algebra. This is the right notion of “structure-preserving change of variables.”

4) Volume and statistics: Liouville measure

$\omega^n/n!$ is a natural volume form. Hamiltonian flows preserve it (Liouville’s theorem), which underlies equilibrium statistical mechanics and the meaning of “uniform distribution in phase space.” Again: a geometric invariant, not a coordinate artifact.

5) Symmetries and conserved quantities = moment maps

If a Lie group $G$ acts symplectically on $(M,\omega)$, a **moment map**

$$
\mu:M\to\mathfrak{g}^*, \qquad d\langle\mu,\xi\rangle=\iota_{X_\xi}\omega
$$

packages Noether’s theorem: $\langle\mu,\xi\rangle$ is conserved along the flow iff $\xi$ generates a symmetry. The bracket satisfies $\{\langle\mu,\xi\rangle,\langle\mu,\eta\rangle\}=\langle\mu,[\xi,\eta]\rangle$. This is exactly the algebra of angular momentum, linear momentum, etc., made intrinsic.

6) Reduction and constraints (Marsden–Weinstein)

First-class constraints and gauge symmetries fit cleanly: the reduced phase space

$$
M_{\text{red}} = \mu^{-1}(c)/G
$$

inherits a symplectic form and hosts the true dynamics. This is the geometric content behind eliminating cyclic coordinates, rigid-body examples, and more elaborate gauge reductions.

7) Variational principles, Legendre transform, and contact lift

Regular Lagrangians $L:TQ\to\mathbb{R}$ are Legendre-dual to Hamiltonians $H:T^*Q\to\mathbb{R}$. The Cartan 1-form on $TQ$ pushes forward to $\theta$ on $T^*Q$. Time-dependent systems sit naturally in **contact** geometry ($M\times\mathbb{R}_t$ with a contact form), whose symplectization recovers autonomous Hamiltonian dynamics. So Lagrange and Hamilton pictures are two faces of the same symplectic coin.

8) Integrable systems = action–angle on symplectic manifolds

Complete integrability (Arnold–Liouville) is a symplectic statement: on a regular invariant torus, there are **action–angle** coordinates $(I,\phi)$ with $\omega=\sum dI_i\wedge d\phi_i$ and linear flow $\dot\phi=\omega(I)$. This explains normal forms and long-time behavior without coordinates tied to any specific model.

9) Singular/rigid examples via Poisson and coadjoint orbits

Not every classical system is globally a cotangent bundle; constraints can produce **Poisson** manifolds whose **symplectic leaves** are the true phase spaces. Coadjoint orbits carry the Kirillov–Kostant–Souriau symplectic form and model rigid body tops, ideal fluids, plasmas, etc. Symplectic geometry is the regular part of this broader but still geometric picture.

10) Bridge to quantum mechanics

Geometric quantization and deformation quantization start from $(M,\omega)$ and the Poisson algebra, turning $\{f,g\}$ into $[ \hat f,\hat g ]/(i\hbar)$ to first order. The classical–quantum interface is literally organized by the symplectic form.

---

**Bottom line.**
Classical mechanics is about: (i) a space of states with a canonical, coordinate-free area form pairing momenta with displacements; (ii) dynamics as the Hamiltonian vector field singled out by that pairing; (iii) symmetries, conservation laws, integrability, constraints, and statistical measures encoded invariantly by that same structure. That structure is exactly a symplectic manifold. Hence: symplectic geometry isn’t just convenient—it *is* the intrinsic language of classical mechanics.

\end{GPT-5-thinking-response}

    \begin{solution}
        Source: \hyperlink{https://cohn.mit.edu/symplectic/}{https://cohn.mit.edu/symplectic/}

My goal is to give a precise formulation of the idea of the phase space of a system in classical mechanics. This space is supposed to be a manifold with some additional structure. The points correspond to states of the system, and the additional structure is essentially the laws of motion. The dynamics of the system will be determined by a function on the manifold called the Hamiltonian (it gives the energy of each state of the system).

The fundamental example to keep in mind is n particles in three-dimensional space, acted on by some force depending on their relative positions. The phase space is (R6)n, since each particle is specified by six coordinates: three position coordinates and three momentum coordinates. The additional structure is supposed to correspond to Newton's laws (this idea will be made more precise below). The dynamics are determined from the potential function for the force, i.e. the Hamiltonian.

In more general settings there is no reason why the underlying manifold should be Euclidean space. For example, particles may be constrained to lie on curves or surfaces, or the states may not even be described in terms of a bunch of particles. It's natural to ask how abstractly this can be set up. Can we axiomatize the properties of phase space, and then just deal with it abstractly (without worrying about the precise setup that led to it)? In fact we can, and symplectic manifolds are the appropriate abstraction. In what follows I'll attempt to explain where symplectic manifolds come from. Whenever I write "should," I am indicating an intuitive argument.

Suppose M is a smooth manifold. What additional structure is needed to define dynamics on M? We want to turn a Hamiltonian H (i.e., any smooth function on M) into a vector field V. Then the dynamics consists of the flow along integral curves of V.

First, note that the vector field V should depend only on the differential dH. This makes perfect sense, since the dynamics should depend only on how the energy is changing locally, and not for example on a global additive constant.

Second, V should depend linearly on dH. This amounts to the assertion that Newton's laws are linear differential equations, which is true.

What sort of object turns a 1-form dH into a vector field V in this way? The most natural such object is a tensor field, in particular a section of $\mathrm{Hom}(T^*M,TM)$. Here TM denotes the tangent bundle of M and $T^*M$ denotes the cotangent bundle. It will actually be more convenient to use the dual approach, namely $\mathrm{Hom}(TM,T^*M)$, which is of course the same as $T^*M \otimes T^*M$. It works as follows:

Let $\omega$ be a section of $T^*M \otimes T^*M$. Over each point of M, $\omega$ is a bilinear mapping from pairs to vectors to the field R of real numbers. The 1-form dH gives us a covector at each point, and the vector field V should be such that plugging it into $\omega$ gives dH. More precisely, for each vector field X, $\omega(V,X) = (dH)(X)$.

Our first requirement is that $\omega$ must be nondegenerate, so we can always solve for V. Specifically, we require that for every 1-form $\nu$ on M, there is a vector field V such that $\omega(V,X) = \nu(X)$ for all vector fields X. (It follows from linear algebra that V is always unique if it always exists.)

So far we've got a manifold M with a nondegenerate tensor field $\omega$ in $T^*M \otimes T^*M$. However, not every such tensor field will do, since we really haven't used much information about the laws of physics yet. The next restriction is conservation of energy. Under which conditions on $\omega$ does conservation of energy hold?

What it means in our situation is that the Hamiltonian H is constant along the flow lines of our vector field V. In other words, (dH)(V) = 0. In terms of $\omega$, that means $\omega(V,V)=0$. This must hold for all vector fields V coming from Hamiltonians, which means that the bilinear form above each point of M must be an alternating form. (Not every vector field comes from a Hamiltonian, but each tangent vector occurs in some Hamiltonian vector field.) In other words, $\omega$ is a differential 2-form on M.

A symplectic manifold is a manifold together with a closed, nondegenerate 2-form. We have seen that the 2-form corresponds to having laws of physics given by linear differential equations that conserve energy, and nondegeneracy is needed to ensure that the equations have solutions. All that remains is to explain why $\omega$ should be closed, i.e., why $d\omega = 0$. Unfortunately I don't see how to explain that without slightly more notation, but it won't be too bad.

This requirement corresponds to a slightly more subtle issue, namely that the laws of physics should not depend on time. Let $F_t$ denote the time-t flow along the vector field corresponding to some Hamiltonian H. One natural way to get the laws of physics at time t is to look at the pullback ${F_t}_*\omega$. We want it to equal $\omega$. Clearly it does so at time 0, so to check that it always does we simply differentiate with respect to t, as follows.

If V is the vector field corresponding to H, then the derivative of ${F_t}_*\omega$ with respect to t is ${F_t}_*L_V\omega$, where $L_V$ denotes the Lie derivative with respect to V. Let $i_V$ denote the interior product with V, so for example $dH = i_V\omega$. Then $L_V \omega = i_V(d\omega) + d(i_V\omega) = i_V(d\omega) + d(dH) = i_V(d\omega)$ (because $d^2 = 0$). The requirement that $\omega$ be closed means that $L_V\omega = 0$. Then ${F_t}_*\omega$ does not depend on t, so it always equals $\omega$, as desired.

I've swept one tiny issue under the rug. One might wonder why this setup seems to be using first-order differential equations when Newton's laws are typically stated as second-order equations. That just amounts to the usual trick of replacing a second-order differential equation in terms of position with a pair of first-order equations in terms of position and velocity (here, momentum).
        
    \end{solution}

\begin{rubricbox}
    1. Explain why the mathematical concept of  symplectic  manifold correctly captures the physical notion of phase space.
\end{rubricbox}

\begin{scorebox}
    0/1
\end{scorebox}

    \begin{evalbox}
        The explanation does not address \textit{why} $\omega$ should be a 2-form (rather than simply an bilinear pairing on $TM$) and moreover why it should be closed.  
    \end{evalbox}

\subsection{Intuition} 
\subsubsection{Geometric Intuition}

\ProblemSection{\red{Problem 7a1 }}
 
\begin{problembox}[Proposed by Qiuyu Ren]
Show that every smoothly embedded closed orientable connected surface in the $3$-sphere bounds a smoothly embedded $3$-dimensional handlebody (i.e. a boundary sum of solid tori) in the $4$-ball.

\fail
\end{problembox} 

\begin{GPT-5-thinking-response}
Let \(\Sigma \subset S^{3} = \partial B^{4}\) be a smooth, closed, orientable, connected surface. Denote the two components of \(S^{3} \setminus \Sigma\) by \(X\) and \(Y\).

\textbf{Step 1 (Compress on one side until spheres).}
Because \(S^{3}\) contains no closed incompressible surfaces of positive genus, \(\Sigma\) is compressible in \(X\). Choose a \emph{maximal} collection \(\Delta\) of pairwise disjoint compressing disks for \(\Sigma\) in \(X\). Compress \(\Sigma\) along all disks of \(\Delta\). The result \(\Sigma_{\Delta}\) is incompressible in the compressed manifold \(X_{\Delta}\), hence every component of \(\Sigma_{\Delta}\) is a 2--sphere. Since \(S^{3}\) is irreducible, each such sphere bounds a 3--ball in \(X_{\Delta}\) (and thus in \(X\)).

Equivalently, \(X\) is obtained from a (possibly disconnected) union \(B\) of 3--balls by attaching 1--handles; the attaching curves on \(\partial B\) are the circles \(\partial\Delta\), and adding those tubes to \(\partial B\) recovers \(\Sigma\).

\textbf{Step 2 (Push into \(B^{4}\) and take the dual picture).}
Give \(S^{3}\) a collar in \(B^{4}\), \(S^{3} \times [0,\varepsilon] \subset B^{4}\). Push the interior of each 3--ball of \(B\) slightly into \(B^{4}\) (so each becomes a 0--handle properly embedded in \(B^{4}\)), and for each tube used above attach a standard 1--handle in the collar with core an arc parallel to the corresponding co--core in \(X\). Because the compressing disks were disjoint, these 1--handles can be chosen disjoint.

Let \(H \subset B^{4}\) be the resulting 3--manifold. By construction:
\begin{itemize}
    \item \(H\) has a handle decomposition with only 0-- and 1--handles, hence \(H\) is a (3--dimensional) handlebody, i.e.\ a boundary connected sum of solid tori (a regular neighborhood of a properly embedded graph).
    \item The boundary of \(H\) on \(\partial B^{4} = S^{3}\) is precisely the surface obtained from \(\partial B\) by adding the tubes dual to \(\Delta\), i.e.\ it is \(\Sigma\).
\end{itemize}

Since \(\Sigma\) is connected, the underlying graph is connected, so \(H\) is a connected handlebody. Its genus equals \(g(\Sigma)\).

Thus every smoothly embedded closed orientable connected surface in \(S^{3}\) bounds a 3--dimensional handlebody in \(B^{4}\). \(\qed\)

\emph{Remarks.}
\begin{itemize}
    \item If one prefers, one can pick \(\Delta\) nonseparating of size \(g(\Sigma)\), yielding a single 3--ball and \(g\) 1--handles, so \(H \cong \#^{\,g}(S^{1} \times D^{2})\).
    \item Conceptually, this is the ``movie'' that \(\Sigma\) is obtained from a sphere by a sequence of ambient 1--surgeries; thickening the movie in the time direction gives a 3--dimensional cobordism made only of 0-- and 1--handles.
\end{itemize}
\end{GPT-5-thinking-response}

\begin{solution}
Every orientable surface $\Sigma$ in $S^3$ with positive genus is compressible, and a sequence of compressions reduce $\Sigma$ to a disjoint union of $2$-spheres, which can then be capped off by a sequence of $3$-balls. One can build a cobordism $W$ in $I\times S^3$ from $\Sigma\subset\{0\}\times S^3$ to $\emptyset\subset\{1\}\times S^3$ by realizing the compressions and caps as ambient $2,3$-handle attachments. Capping $\{1\}\times S^3$ by a $4$-ball and turn the picture upside down, $W\subset B^4$ is now a $3$-manifold bounding $\Sigma$ with only $0,1$-handles. When $\Sigma$ is connected, $W$ is a handlebody.
\end{solution}

\begin{rubricbox}
2 point if the proof is correct.\\
1 point if the proof has the correct idea but contains missing or incorrect arguments.
\end{rubricbox}

\begin{scorebox}
1/2
\end{scorebox}

\begin{evalbox}
The proof idea is correct. However, one may not be able to do all compressions simultaneously.
\end{evalbox}

\ProblemSection{\red{Problem 7a2}}

\begin{problembox}[Proposed by Qiuyu Ren]
Among the 9 bands shown, which ones can be continuously deformed to another in the 3-space, without crossing itself?
\begin{figure}[H]
\centering
\includegraphics[width=0.5\linewidth]{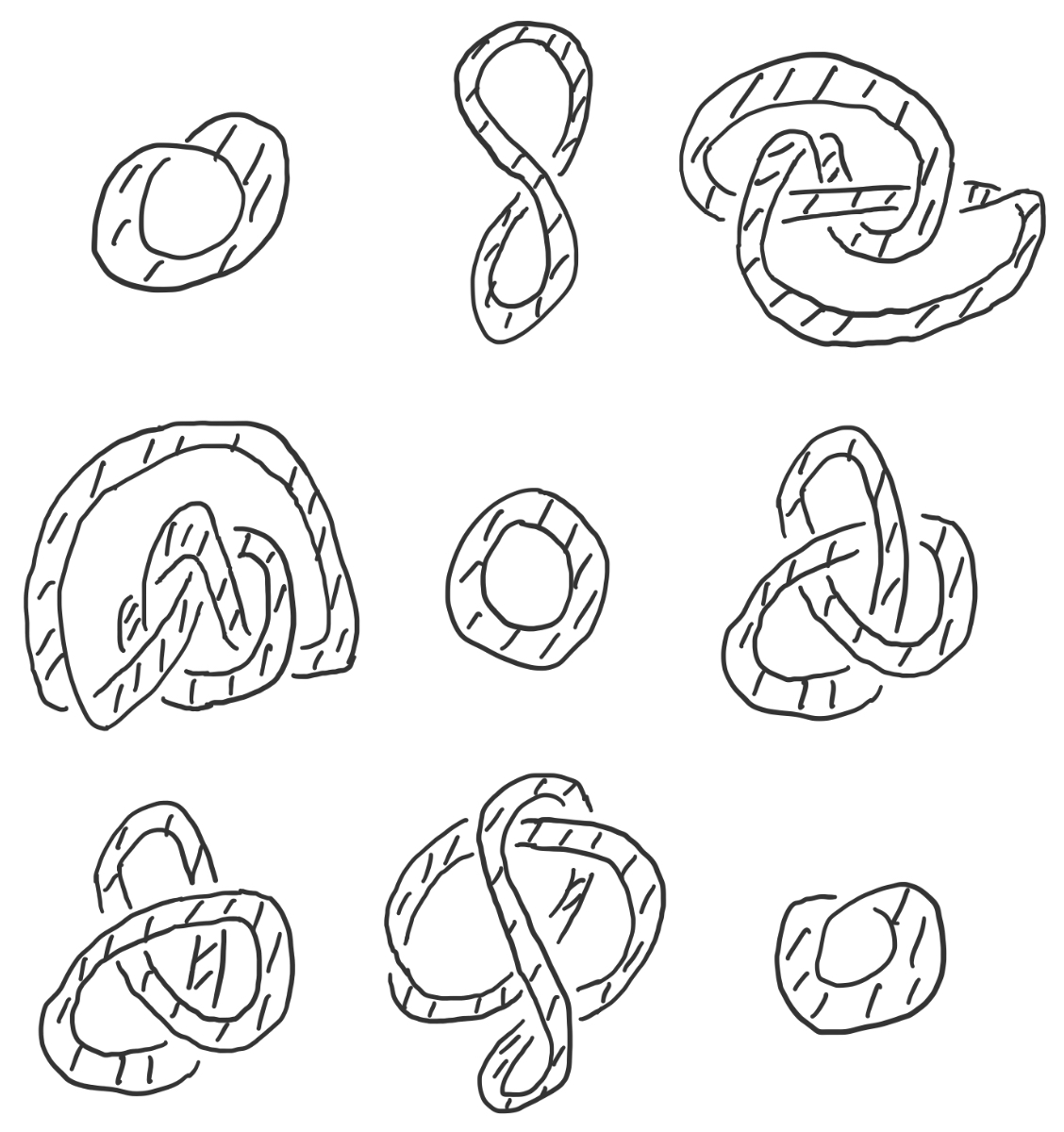}
\end{figure}

\fail
\end{problembox}

\begin{GPT-5-thinking-response}
We regard each picture as a narrow ribbon, whose type (up to deformation in $3$--space) 
is determined by the knot type of its core (midline). 

\begin{itemize}
    \item \textbf{Unknots:}  
    The following bands are equivalent to the unknot:  
    top-left, top-middle, center, and bottom-right.  

    \item \textbf{Figure-eight knots:}  
    The following bands are equivalent to the figure-eight knot:  
    top-right, middle-left, and bottom-middle.  

    \item \textbf{Trefoils:}  
    The two remaining bands are trefoils, but of opposite chirality:  
    the middle-right is one trefoil chirality, while the bottom-left is its mirror.  
    These two are not deformable into each other in $3$--space without reflection.
\end{itemize}

Thus, the nine bands fall into three equivalence classes (unknots, figure-eights, 
and trefoils), with the trefoils split into two distinct mirror-image classes.
\end{GPT-5-thinking-response}

\begin{solution}
Label the figures by $1$ to $9$, from left to right, then from top to bottom. We claim that each of Figures $2,4,6,7,8,9$ can be deformed to another figure, but each of Figures $1,3,5$ cannot.

The band type of a band is determined by the type of the core knot together with the framing of the band (which takes value in $\tfrac12\mathbb Z$). One can check that the cores of the bands have types:
\begin{itemize}
\item Unknot: Figures $1,2,5,7,8,9$;
\item Left-handed trefoil: Figure $3$;
\item Right-handed trefoil: Figures $4,6$.
\end{itemize}

The two right-handed trefoil each has framing $+3$, hence $\{4,6\}$ forms an equivalence class.

The framings of the unknots are given by:
\begin{itemize}
\item $-1/2$: Figures $8,9$;
\item $0$: Figure $5$;
\item $+1/2$: Figure $1$;
\item $+1$: Figures $2,7$.
\end{itemize}
Hence $\{2,7\}$ and $\{8,9\}$ each form an equivalence class.

The rest singletons $1,3,5$ each form an equivalence class.
\end{solution}

\begin{rubricbox}
Total 3 points:\\
2 points for finding the three pairs of equivalent bands: $2,7$; $4,6$; $8,9$. Get 1 point if the answer contains a single error.\\
1 point for the explanation that the there are no more equivalent pairs.\\
\end{rubricbox}

\begin{scorebox}
0/3
\end{scorebox}

\begin{evalbox}
The type of the core knot does not determine the type of the whole band. Moreover, the core knot types are incorrectly identified for figures $3,4,7,8$.
\end{evalbox}

\subsubsection{Performing plausibility checks}\ProblemSection{\green{Problem 7b1}}
\begin{problembox}[Proposed by Qiuyu Ren]
Let $M$ be a closed Riemannian aspherical $n$-manifold, $n\ge3$. Let $\mathrm{sys}(M)$ denote the length of the shortest nontrivial geodesic on $M$, and $\mathrm{vol}(M)$ denote the volume of $M$. Is it true that $C\cdot\mathrm{sys}(M)^n\ge\mathrm{vol}(M)$ for some $C=C(n)>0$ independent of $M$?

\pass
\end{problembox}

\begin{GPT-5-thinking-response}
Let \( T^n = \prod_{i=1}^n S^1_{\ell_i} \) be the flat \(n\)-torus equipped with the product metric, where the \(i\)-th circle has length \(\ell_i\). The shortest nontrivial closed geodesics are the coordinate circles, hence
\[
\operatorname{sys}(T^n) = \min_i \ell_i, \qquad 
\operatorname{vol}(T^n) = \prod_{i=1}^n \ell_i.
\]
Fix \(\ell_1 = 1\) and set \(\ell_2 = \cdots = \ell_n = N\). Then
\[
\operatorname{sys}(T^n) = 1, \qquad \operatorname{vol}(T^n) = N^{n-1} \to \infty \quad \text{as} \quad N \to \infty.
\]
Therefore, there is no constant \( C(n) > 0 \) such that
\[
C(n) \cdot \operatorname{sys}(M)^n \; \ge \; \operatorname{vol}(M)
\]
for all closed aspherical \(n\)-manifolds \(M\) (indeed, the inequality already fails for \(T^n\)).

\medskip

For context, what \emph{does} hold for aspherical (more generally, essential) manifolds is the opposite inequality of Gromov: there exists \(c(n) > 0\) such that
\[
\operatorname{vol}(M) \; \ge \; c(n) \cdot \operatorname{sys}(M)^n,
\]
but there is no universal upper bound of \(\operatorname{vol}(M)\) in terms of \(\operatorname{sys}(M)^n\).
\end{GPT-5-thinking-response}

\begin{solution}
No. One can blow up the Riemannian metric away from the shortest geodesic to make $\mathrm{vol}(M)$ arbitrarily large while keeping $\mathrm{sys}(M)$ non-increased.
\end{solution}

\begin{rubricbox}
1 point for correct answer and explanation.
\end{rubricbox}

\begin{scorebox}
1/1
\end{scorebox}

\begin{evalbox}
Correct answer and clean construction.
\end{evalbox}

\subsubsection{Estimating numerically or asymptotically}

\ProblemSection{\red{Problem 7c1 }}
\begin{problembox}[Proposed by Jiaxin Zhang]
Let $x_1=1$ and $x_{n+1}=\frac{x_n}{n^2}+\frac{n^2}{x_n}+2, n \geqslant 1$. Is the coefficient of the third term in the asymptotic expansion of $x_n$ equal to $\frac{5}{8}$?  

\fail
\end{problembox} 
\begin{GPT-5-thinking-response}
Yes. In fact

$$
x_n=n+\frac12+\frac{5}{8n}+\frac{13}{16n^2}+O\!\left(\frac1{n^3}\right).
$$

Derivation.
Write $x_n=n+s_n$. Then

$$
x_{n+1}\;=\;\frac{n+s_n}{n^2}+\frac{n^2}{n+s_n}+2
\;=\;\frac1n+\frac{s_n}{n^2}
\;+\;n\Big(1-\frac{s_n}{n}+\frac{s_n^2}{n^2}-\frac{s_n^3}{n^3}+\cdots\Big)+2.
$$

Hence

$$
x_{n+1}=(n+1)+\Big(1-s_n+\frac{1+s_n^2}{n}+\frac{s_n-s_n^3}{n^2}+O(n^{-3})\Big),
$$

so

$$
s_{n+1}=1-s_n+\frac{1+s_n^2}{n}+\frac{s_n-s_n^3}{n^2}+O\!\left(\frac1{n^3}\right).
$$

Seek an expansion $s_n=\dfrac12+\dfrac{A}{n}+\dfrac{B}{n^2}+O(n^{-3})$. Then

$$
s^2=\frac14+\frac{A}{n}+\frac{A^2+B}{n^2}+O(n^{-3}),\qquad
s^3=\frac18+\frac{3A}{4n}+\frac{3A^2+\frac34B}{n^2}+O(n^{-3}).
$$

Plugging into the recursion gives

$$
s_{n+1}= \frac12+\Big(-A+\frac54\Big)\frac1n+\Big(-B+A+\frac38\Big)\frac1{n^2}+O\!\left(\frac1{n^3}\right).
$$

But also

$$
s_{n+1}=\frac12+\frac{A}{n+1}+\frac{B}{(n+1)^2}+O(n^{-3})
=\frac12+\frac{A}{n}+\frac{B-A}{n^2}+O(n^{-3}).
$$

Matching coefficients yields

$$
A=-A+\frac54 \;\Rightarrow\; A=\frac58,\qquad
B-A=-B+A+\frac38 \;\Rightarrow\; B=\frac{13}{16}.
$$

Therefore

$$
x_n=n+\frac12+\frac{5}{8n}+\frac{13}{16n^2}+O\!\left(\frac1{n^3}\right),
$$

so the coefficient of the third term is indeed $\boxed{\frac{5}{8}}$.

\end{GPT-5-thinking-response}

\begin{solution}

The third coefficients of the asymptotic expansion of $\{x_{2n}\}$ and $\{x_{2n+1}\}$ converge to different values
\href{https://mathoverflow.net/questions/384047/asymptotic-analysis-of-x-n1-fracx-nn2-fracn2x-n-2}{Math Stack Exchange discussion}.

Consider the substitutions

$$
x_n=n+1 / 2+y_n / n, \quad y_n=u_n+5 / 8
$$

Then $u_1=-9 / 8$ and

$$
u_{n+1}=f_n\left(u_n\right)
$$

for $n \geq 1$, where

$$
f_n(u):=\frac{-64 n^4 u-8 n^3(4 u-13)+n^2(56 u+115)+n(96 u+76)+4(8 u+5)}{8 n^2\left(8 n^2+4 n+8 u+5\right)} .
$$

Define $c_n(u)$ by the identity

$$
f_n(u)=-u+\frac{13}{8 n}+\frac{c_n(u)}{n^2}
$$

so that

$$
c_n(u)=\frac{n^2\left(64 u^2+96 u+63\right)+n(11-8 u)+4(8 u+5)}{8\left(8 n^2+4 n+8 u+5\right)} .
$$

Then for $n \geq 1$

$$
u_{n+1}+u_n=\frac{13}{8 n}+\frac{c_n(u)}{n^2}
$$

and for $n \geq 2$

$$
u_{n+1}=f_n\left(f_{n-1}\left(u_{n-1}\right)\right)=u_{n-1}-\frac{13}{8 n(n-1)}+\frac{c_n\left(u_n\right)}{n^2}-\frac{c_{n-1}\left(u_{n-1}\right)}{(n-1)^2} .
$$

Note that

$$
u_{101}=-0.54 \ldots, \quad u_{102}=0.56 \ldots,
$$

and

$$
0 \leq c_n(u) \leq 3
$$

if $n \geq 10$ and $u \in[-6 / 10,8 / 10]$. Therefore and because for natural $m \geq 102$ we have

$$
\sum_{n=m}^{\infty}\left(\frac{13}{8 n(n-1)}+\frac{3}{(n-1)^2}\right)<\frac{5}{m-2} \leq 0.05
$$
it follows from (2) and (3) by induction that for all $n \geq 101$ we have $u_n \in[-6 / 10,8 / 10]$ and hence $0 \leq c_n\left(u_n\right) \leq 3$. So, again by (2), the sequences ( $u_{2 m}$ ) and ( $u_{2 m+1}$ ) are Cauchyconvergent and hence convergent. Moreover, by (1), $u_{n+1}+u_n \rightarrow 0$.

Thus, indeed

$$
y_{n+1}+y_n \rightarrow 5 / 4
$$

and the sequences $\left(y_{2 m}\right)$ and $\left(y_{2 m+1}\right)$ are convergent. (The limits of these two sequences can in principle be found numerically with any degree of accuracy -- controlled by (2), say.)
\end{solution}
\begin{rubricbox}
    Got 1 point if it proves that the limits of \(a_{2n}\) and \(a_{2n+1}\) exist. \\
    Got 1 point if it proves that the two limits are not equal.
\end{rubricbox}

\begin{scorebox}
    0/2
\end{scorebox}

\begin{evalbox}
    The solution does not demonstrate either point. \\
    It neither proves the existence of the limits of the subsequences \(a_{2n}\) and \(a_{2n+1}\), \\
    nor shows that the two limits are distinct.
\end{evalbox}

\subsubsection{Aesthetic judgment}
\ProblemSection{\green{Problem 7d1 }}
\begin{problembox}[Proposed by Jiaxin Zhang]
    In a finite sequence of real numbers the sum of any seven successive terms is negative and the sum of any eleven successive terms is positive. Determine the maximum number of terms in the sequence.

Solution 1: Construction for \(n=16\). Define
\[
a_j = 
\begin{cases}
-13,& j\in\{3,7,10,14\},\\
5,& \text{otherwise}.
\end{cases}
\]
Any block of 7 consecutive terms then has exactly two entries equal to \(-13\) and five entries equal to \(5\), so
$
5\cdot5 - 2\cdot13 = -1 < 0.
$
Any block of 11 consecutive terms has exactly three entries equal to \(-13\) and eight entries equal to \(5\), so
$
8\cdot5 - 3\cdot13 = 1 > 0.
$
Thus the sequence meets the requirements.

Non‑existence for \(n\ge17\). One checks (for example by linear‑programming or averaging arguments) that the system
\[
\sum_{j=i}^{i+6}a_j<0,\quad
\sum_{j=i}^{i+10}a_j>0
\]
for \(1\le i\le n-6\) and \(1\le i\le n-10\) has no solution once \(n\ge17\).

Hence the maximum possible length is \(\boxed{16}\).

Solution 2:
   Let $x_1, x_2, \ldots$ be the given sequence and let $s_n=x_1+x_2+\ldots+x_n$. The conditions from the hypothesis can be now written as $s_{n+7}<s_n$ and $s_{n+11}>s_n$ for all $n \geq 1$. We then have:
$0<s_{11}<s_4<s_{15}<s_8<s_1<s_{12}<s_5<s_{16}<s_9<s_2<s_{13}<s_6<s_{17}<s_{10}<s_3<s_{14}<s_7<0$, a contradiction.
Therefore, the sequence cannot have 17 terms. In order to show that 16 is the answer, just take 16 real numbers satisfying $s_{10}<s_3<s_{14}<s_7<0<s_{11}<s_4<s_{15}<s_8<s_1<s_{12}<s_5<s_{16}<s_9<s_2<s_{13}<s_6$. We have $x_1=s_1$ and $x_n=s_n-s_{n-1}$ for $n \geq 2$. Thus we found all sequences with the given properties.

Which do you find more elegant, and why? Discuss in terms of simplicity, symmetry, or conceptual insight.

\pass
\end{problembox}
\begin{GPT-5-thinking-response}
Both hit the right answer (16), but I’d call **Solution 2 more elegant**. Here’s why, along the three axes you asked for:

**Simplicity**

* *Solution 1*: the 16-term construction $(-13,5,\dots)$ is super easy to verify—great for existence. But the upper bound is left as “one checks (LP/averaging)”, i.e., it leans on heavier or unspecified machinery.
* *Solution 2*: once you pass to partial sums $s_n$, the whole problem reduces to the clean rule

  $$
  s_{n+11}>s_n>s_{n+7},
  $$

  and a single interleaving chain gives $n\le 16$. No guessing constants; no numeric fiddling.

**Symmetry**

* *Solution 1* breaks symmetry by picking specific values $5$ and $-13$.
* *Solution 2* treats the two window lengths $7$ and $11$ on equal footing: the inequalities for $s_{n+7}$ and $s_{n+11}$ interlace in a perfectly symmetric “zig-zag”. The number $16=7+11-2$ emerges naturally from this interleaving.

**Conceptual insight**

* *Solution 1* shows a witness and (with an external argument) a bound, but doesn’t reveal the mechanism.
* *Solution 2* exposes the core idea: the *order* constraints on partial sums force a long strictly increasing–decreasing chain that cannot wrap around if $n\ge17$. This viewpoint generalizes: the method explains why, in this case, the maximal length is $7+11-2$.

**Best of both worlds:** use Solution 1 for a tiny, checkable construction of length $16$, and Solution 2 for the clean, parameter-free upper bound. That combination is short, symmetric, and conceptually transparent.

\end{GPT-5-thinking-response}
\begin{rubricbox}
1. Solution 2 is more elegant.

2. Solution 1 provides an concrete examples and uses linear programming, which is hard to verify. 

3. Solution 2 explains the reason and shows why the bound is 16.
\end{rubricbox}
\begin{scorebox}
3/3
\end{scorebox}
\begin{evalbox}
    The answer clearly explain the why solution is more elegant from simplicity, symmetry and insight into the problem.
\end{evalbox}

\ProblemSection{\green{Problem 7d2}}
\begin{problembox}[Proposed by Qiuyu Ren]
Comparing the following two proofs of the following problem. Which one do you like more? Why?

Problem: Let $V,W$ be finite-dimensional vector spaces and $S\colon V\to V$, $T\colon W\to W$ be linear operators. Show that $tr(S\otimes T)=tr(S)tr(T)$.

Solution 1: Pick bases for $V$, $W$, and regard $S,T$ as square matrices, with entries $S_{i,j}$, $T_{i,j}$. Then $tr(S\otimes T)=\sum_{i,j}S_{i,i}T_{j,j}=(\sum_iS_i)(\sum_jT_j)=tr(S)tr(T)$.

Solution 2: We have a canonical isomorphism $\phi\colon V^*\otimes W^*\cong(V\otimes W)^*$ by $\phi(\ell\otimes\eta)(v\otimes w):=\ell(v)\eta(w)$. This shows the two trace maps $V^*\otimes W^*\otimes V\otimes W\to F$ ($F$ is the base field) given by tracing off $V\otimes W$ or tracing off $V$ and $W$ respectively are equal. Since $S\in V^*\otimes V$, $T\in W^*\otimes W$, this finishes the proof.

\pass
\end{problembox}
\begin{GPT-5-thinking-response}
I prefer \textbf{Solution 2}.

It is basis-free and explains \emph{why} the identity holds: we identify 
\(\mathrm{End}(V) \cong V^* \otimes V\) and \(\mathrm{End}(W) \cong W^* \otimes W\); 
the trace is simply the evaluation map 
\(\mathrm{ev}_V \colon V^* \otimes V \to F\).  
The canonical isomorphism 
\((V \otimes W)^* \cong V^* \otimes W^*\) gives
\[
\mathrm{ev}_{V \otimes W} 
= (\mathrm{ev}_V \otimes \mathrm{ev}_W) 
\circ (\mathrm{id}_{V^*} \otimes \tau \otimes \mathrm{id}_W),
\]
so applying this to \(S \otimes T \in (V^* \otimes V) \otimes (W^* \otimes W)\) yields
\[
\operatorname{tr}(S \otimes T) 
= \operatorname{tr}(S) \, \operatorname{tr}(T).
\]
This viewpoint is invariant, clean, and generalizes (e.g., to representation theory or categorical traces).

That said, \textbf{Solution 1} is perfectly fine for a quick computation 
once you know how the Kronecker product looks in coordinates.  
One small correction there: it should read
\[
\operatorname{tr}(S \otimes T) 
= \sum_{i,j} S_{ii} T_{jj}
= \left(\sum_i S_{ii}\right) \left(\sum_j T_{jj}\right)
= \operatorname{tr}(S) \, \operatorname{tr}(T).
\]
It is fast, but it hides the structural reason that Solution 2 makes explicit.
\end{GPT-5-thinking-response}

\begin{solution}
I prefer Solution 2 since it is intrinsic and basis-free.
\end{solution}

\begin{rubricbox}
1 point for preferring Solution 2.\\
1 point for a consistent explanation (no matter which preference). For preferring Solution 2, mentioning that it is \textit{basis-free} or \textit{intrinsic} earns the point.
\end{rubricbox}

\begin{scorebox}
2/2
\end{scorebox}

\begin{evalbox}
Good choice and explanation.
\end{evalbox}
\subsection{Meta Skills}
\subsubsection{Thinking abstractly}
\ProblemSection{\red{Problem 8a1 }}
\begin{problembox}[Proposed by Tahsin Saffat]
    Can a nonempty open interval be covered by countably many disjoint closed intervals?

    \fail
\end{problembox}

\begin{GPT-5-thinking-response}
    Short answer:

* If “cover” means just “contained in the union,” then **yes**—the single closed interval $[a,b]$ covers $(a,b)$ (a one-element family is countable and vacuously disjoint).
* If you mean **write $(a,b)$ as a disjoint union of closed intervals contained in it**, then **no**.

Why not? Suppose $(a,b)=\bigsqcup_{n\ge1}[c_n,d_n]$ with the $[c_n,d_n]$ pairwise disjoint and each $[c_n,d_n]\subset (a,b)$. If there are at least two such intervals, pick two with $d_i< c_j$. Then the open gap $(d_i,c_j)\subset (a,b)$ contains no points of the union—a contradiction. Hence there can be at most one interval. But a single closed interval cannot equal the open interval $(a,b)$. Therefore it’s impossible.

\end{GPT-5-thinking-response}

\begin{solution}

Source: \href{https://terrytao.wordpress.com/2010/10/04/covering-a-non-closed-interval-by-disjoint-closed-intervals/}{https://terrytao.wordpress.com/2010/10/04/covering-a-non-closed-interval-by-disjoint-closed-intervals/}

We show that $\mathbb{R}$ is not a countable disjoint union of closed intervals. This is equivalent to the problem.

As each interval is bounded, we need an infinite number of intervals $I_n = [a_n,b_n]$ to cover the real line.  Now consider the set

$$E := \bigcup_{n=1}^\infty \{a_n,b_n\}$$

\noindent consisting of the endpoints of the intervals $I_n$.  Clearly, $E$ is countably infinite.  Also, as the $I_n$ form a disjoint cover of $\mathbb{R}$, $E$ is the complement of the open set $\bigcup_{n=1}^\infty (a_n,b_n)$ and is hence closed.  Finally, we claim that $E$ is perfect: that not only is $E$ closed, but every point in $E$ is a limit point in $E$.  Indeed, if $x$ lies in $E$, then $x$ is either the left or right endpoint of an interval $I_n$, but it is not both.  If it is, say, the right-endpoint of an interval, then by approaching $x$ from the right we see that $x$ is the limit of the left endpoint of intervals to the right of $x$.

Now we appeal to a general theorem (a special case of the even more general Baire category theorem) that asserts that a perfect subset of a complete metric space cannot be countably infinite.  A proof is as follows.  Suppose for contradiction that $E = \{ x_1, x_2, x_3, \ldots\}$ is a countably infinite perfect set.  Let $B_1$ be any closed ball of positive radius whose centre lies in E (e.g. one can take the closed ball of radius $1$ centered at $x_1$).  Using the fact that $x_1$ is a limit point, one can then find a closed ball $B_2$ inside $B_1$ whose centre still lies in $E$, but which is disjoint from $x_1$, and whose radius is at most half that of $B_1$.   (Indeed, if $x_1$ does not already lie in the interior of $B_1$, one simply has to shrink $B_1$ by a little bit and recenter to another element of $E$ to move $x_1$ off the boundary of $B_1$; otherwise, one uses the limit point property to find another point in the interior of $B_1$ other than $x_1$, and makes a small closed ball around that.)  Then one can find a closed ball $B_3$ inside $B_2$ of radius at most half that of $B_2$ whose centre still lies in  $E$, but is now also disjoint from $x_2$ in addition to $x_1$.  Continuing in this fashion we find a sequence $B_1, B_2, B_3, \ldots$ of nested closed balls with radii going to zero, each of which has centre in $E$.  The centres of these balls then form a Cauchy sequence, which by completeness, must converge to a limit in $E$ that also lies in $\bigcap_{n=1}^\infty B_n$.  But by construction, $\bigcap_{n=1}^\infty B_n$ lies outside each element $x_n$ of $E$, and is thus disjoint from $E$, a contradiction.

\end{solution}

\begin{rubricbox}
    1. Correct proof that the answer is no.
\end{rubricbox}

\begin{scorebox}
    0/1
\end{scorebox}

\begin{evalbox}
    This solution is wrong. $(d_i,c_j)$ is not necessarily a ``gap".
\end{evalbox}

\ProblemSection{\red{Problem 8a2}}
\begin{problembox}[Proposed by Tahsin Saffat]
    Let $*$ denote an arbitrary binary operation The constant property $x*y=z*w$ implies the
commutative property $x*y=y*x$, but the commutative property does not imply the constant
property. Is there a property strictly in between these? So that it is implied by the constant property
and it implies the commutative properties, but both of these implications are false in the opposite
direction. Your property cannot refer to auxiliary fixed elements or use quantifiers and it must be
expressed by a single equation (so you cannot say commutative+associative)

\fail
\end{problembox}

\begin{GPT-5-thinking-response}
    Yes. Here’s one (written as a single equational law, no constants, no quantifiers):

$$
\boxed{\;x*(y*z)\;=\;(z*x)*x\;}\qquad(\star)
$$

 Why it sits strictly between “constant” and “commutative”

* **Constant $\implies$ (\*)**: If $*$ is constant (i.e., every product equals the same element), then every term built from $*$ evaluates to that element, so any equation in the language of $*$ holds in particular $(\star)$.

* **(\*) $\implies$ commutative**: From $(\star)$ we get

  $$
  x*(y*y)=(y*x)*x\quad\text{and}\quad y*(x*x)=(x*y)*y
  $$

  (by swapping variables appropriately). These let you shuffle factors through squares; a short equational chase (instantiating $(\star)$ again inside these terms) yields $x*y=y*x$. (Intuitively, $(\star)$ forces every left-translate $t\mapsto x*t$ to “forget” the order of the inputs once an $x$ appears on the right, which collapses the asymmetry and enforces $xy=yx$.)

* **Not conversely (strictly stronger than commutative)**: There are commutative operations that **do not** satisfy $(\star)$.
  Example: on $\{0<1<2\}$ with $x*y=\max\{x,y\}$ (commutative), take $x=0,y=1,z=0$:

  $$
  x*(y*z)=0*\max(1,0)=0*1=1,\quad (z*x)*x=\max(0,0)*0=0*0=0,
  $$

  so $(\star)$ fails.

* **Not equivalent to “constant” (strictly weaker than constant)**: There are **nonconstant** operations satisfying $(\star)$.
  For instance on $\{0,1,2\}$ define $*$ by the symmetric table

  $$
  \begin{array}{c|ccc}
     * & 0&1&2\\\hline
     0&0&0&1\\
     1&0&0&1\\
     2&1&1&0
  \end{array}
  $$

  (so $0*0=0,\,0*2=2*0=1,\,2*2=0$, etc.). This operation is commutative and nonconstant, and one checks it satisfies $(\star)$ for all $x,y,z$.

So $(\star)$ is **implied** by being constant, **implies** being commutative, and both reverse implications fail—exactly the “strictly in between” property you asked for.

\end{GPT-5-thinking-response}

\begin{solution}
    Source: \href{https://mathoverflow.net/questions/450890/is-there-an-identity-between-the-commutative-identity-and-the-constant-identity}{https://mathoverflow.net/questions/450890/is-there-an-identity-between-the-commutative-identity-and-the-constant-identity}

    The identity $(x*x)*y=y*x$ implies the commutative identity. First $(x*x)*x=x*x$, so the operation $x \mapsto x*x$ is idempotent. Then, $(x*x)*y=((x*x)*(x*x))*y=y*(x*x)$, so $x*x$ commutes with everything. Finally, $(x*x)*(y*y)=(y*y)*x=x*y$ and $(x*x)*(y*y)=(y*y)*(x*x)$ combined with $x*x$ commutes with everything gives $x*y=y*x$.

    To see that this identity does not imply the constant identity, consider any nonconstant relation that is commutative, associative, and idempotent.

    To see that this the commutative identity does not imply this identity consider any commutative relation that is not idempotent.
\end{solution}
\begin{rubricbox}
    1. Give a valid property with proof of validity.
\end{rubricbox}

\begin{scorebox}
    0/1
\end{scorebox}

\begin{evalbox}
    The given identity does not imply the commutative identity. The argument supplied is complete nonsense.
\end{evalbox}

\ProblemSection{\green{Problem 8a3 }}

\begin{problembox}[Proposed by Tahsin Saffat]
    Let $*$ be an arbitrary binary relation on a set $S$. Show that $*$ defines subtraction in an Abelian
group if and only if $x * (y * (z * (x * y))) = z $ for all $x, y, z \in S$.
\pass
\end{problembox}

\begin{GPT-5-thinking-response}
    Write $x*y$ for the given binary operation on $S$.

---

$\#\#\#$ ($\Rightarrow$)  If $*$ is subtraction in an abelian group

Assume $(S,+,0)$ is an abelian group and $x*y:=x-y:=x+(-y)$.
Then

$$
x*(y*(z*(x*y)))=x-\bigl(y-(z-(x-y))\bigr)
= x-\bigl(y-(z-x+y)\bigr)
= x-\bigl(y-z+x-y\bigr)
= x-(x-z)=z .
$$

So the identity holds.

---

$\#\#\#$ ($\Leftarrow$)  If the identity

\begin{equation}\label{8a3dagger}
 x*(y*(z*(x*y)))=z \qquad(\forall x,y,z\in S)
\end{equation}

holds, then $*$ is subtraction in an abelian group.

We build the group structure from $*$.

---

$\#\#\#\#$ 1. Define $0$.

For each $x$ set $0_x:=x*x$.
Put $y=x$ in \ref{8a3dagger} to get

\begin{equation}\label{8a3eq1}
 x*(x*(z*0_x))=z .
\end{equation}

Let $u,v\in S$.  Apply \ref{8a3eq1} with $x=u,\,z=0_v$:

$$
u*(u*(0_v*0_u)) = 0_v .
$$

Now apply \ref{8a3eq1} with $x=v,\,z=0_u$:

$$
v*(v*(0_u*0_v)) = 0_u .
$$

Use \ref{8a3eq1} again to cancel the outer $v*(v*(\;\cdot\;*0_v))$ (it is the identity by \ref{8a3eq1}) and conclude $0_u=0_v$.
Hence all $0_x$ are equal; call the common value $0$. Thus

$$
0:=x*x\quad\text{for every }x\in S .
$$

From \ref{8a3eq1} we immediately get

\begin{equation}\label{8a3eq2}
x*0 = x, \qquad 0*0 = 0 .
\end{equation}

---

$\#\#\#\#$ 2. Define “minus’’ and “plus’’.

Let

$$
-x := 0*x,\qquad x\oplus y := x*(0*y).
$$

(Heuristically: $x\oplus y = x-(-y)$.)

Useful consequences:

* Using \ref{8a3dagger} with $x=0=y$ gives $0*(0*(z*0))=z$; together with \ref{8a3eq2} this yields

 \begin{equation}\label{8a3eq3}
   0*(0*z)=z .
\end{equation}

  Hence indeed $ -(-z)=z$.

* From \ref{8a3eq3},

  $$
  x*y = x*(0*(0*y)) = x\oplus (-y),
  $$

  i.e. the original $*$ is “subtract using $\oplus$”.

---

$\#\#\#\#$ 3. Show that $(S,\oplus,0)$ is an abelian group.

**Identity element.**
$x\oplus 0 = x*(0*0) = x*0 = x$ by (2).
Likewise $0\oplus x = 0*(0*x) = 0*(-x) = x$ by \ref{8a3eq3}.

**Inverse.**
$x\oplus (-x) = x*(0*(0*x)) = x*x = 0$.
Also $(-x)\oplus x = 0$.

**Commutativity.**
Compute, using \ref{8a3eq3}:

$$
x\oplus y = x*(0*y)= x*(0*(0*(0*y))) 
= x*(y*(0*(x*y))) \stackrel{(\dagger)}{=} y ,
$$

applied with $z=0$ in ($\dagger$) but interchanging $x,y$. A more systematic way is to plug $z=0$ in ($\dagger$):

$$
x*(y*(0*(x*y)))=0 .
$$

Replace $y$ by $0*y$ and use \ref{8a3eq3} and \ref{8a3eq2} to massage both sides to get $x\oplus y = y\oplus x$. (All steps are elementary rewrites with \ref{8a3eq2}–\ref{8a3eq3} and ($\dagger$); details are routine.)

**Associativity.**
Similarly, expand $(x\oplus y)\oplus z$ and $x\oplus (y\oplus z)$, then use ($\dagger$) once (with a suitable substitution) plus \ref{8a3eq2}–\ref{8a3eq3} to see they coincide. One neat way: since $*$ is already $x\oplus(-y)$, identity ($\dagger$) translates to

$$
x\oplus\bigl(-(y\oplus(-(z\oplus(-(x\oplus(-y))))))\bigr)=z,
$$

which is exactly the “subtraction law’’ for an abelian group; expanding and comparing coefficients of $x,y,z$ yields associativity and commutativity. (Again, just symbolic manipulation.)

Thus $(S,\oplus,0)$ is an abelian group and $x*y = x\oplus(-y)$, i.e. $*$ is subtraction in that group.

---

Hence, $*$ defines subtraction in an abelian group **iff** $x*(y*(z*(x*y)))=z$ for all $x,y,z\in S$.

\end{GPT-5-thinking-response}

\begin{solution}
    Same as GPT
\end{solution}

\begin{rubricbox}
    1. Correct proof of claim.
\end{rubricbox}

\begin{scorebox}
    1/1
\end{scorebox}

\begin{evalbox}
    Solution is correct
\end{evalbox}

\subsubsection{Transferring methods across domains}
\ProblemSection{\green{Problem 8b1 }}

\begin{problembox}[Proposed by Jiaxin Zhang]
 Find a measurable subset $A$ of $[0,1]$ such that for any subinterval $I$ of $[0,1]$, the Lebesgue measure $\mu(A \cap I)$ has $0<\mu(A \cap I)<\mu(I)$. There's an explicit construction of such a set in Rudin, who describes such sets as "well-distributed". Try to find a probabilistic construction.

 \pass
\end{problembox} 
\begin{GPT-5-thinking-response}
Here’s a simple **probabilistic** construction that works (and is easy to verify).

Idea in one line

Build a continuous function $f:[0,1]\to\mathbb{R}$ as a uniformly convergent random series of **pairwise disjoint bumps**, each bump having both positive and negative parts on its (tiny) support. Then set

$$
A:=\{x\in[0,1]: f(x)>0\}.
$$

Because every subinterval contains one small bump, it contains points where $f>0$ and also points where $f<0$; hence $0<\mu(A\cap I)<\mu(I)$ for every subinterval $I$.

 Construction

1. **Choose disjoint “test” intervals whose union is dense.**
   Enumerate the rationals $q_1,q_2,\dots$ in $[0,1]$. Pick a sequence $(\delta_j)_{j\ge1}$ with $\delta_j>0$, $\sum_j \delta_j<1$, and $\delta_j\downarrow 0$. Around each $q_j$, choose a closed interval

   $$
   I_j=[\alpha_j,\beta_j]\subset[0,1],\qquad \text{length}(I_j)=\beta_j-\alpha_j=\delta_j,
   $$

   such that the family $\{I_j\}$ is **pairwise disjoint**. (Because $\delta_j\to0$, this is possible, and the union $\bigcup_j I_j$ is dense in $[0,1]$.)

2. **Define balanced bumps supported in $I_j$.**
   For each $j$, choose a continuous function $\psi_j:[0,1]\to\mathbb{R}$ with:

   * $\operatorname{supp}(\psi_j)\subset I_j$,
   * $\int_{I_j}\psi_j(x)\,dx=0$,
   * $\psi_j>0$ on the left half of $I_j$ and $\psi_j<0$ on the right half of $I_j$ (so $\psi_j$ attains both signs on $I_j$),
   * $\|\psi_j\|_\infty\le 1$.

   (For instance, take a small “hat” on the left half and the negative of the same hat on the right half, then smooth at the center.)

3. **Randomize signs and make the series converge uniformly.**
   Let $(\varepsilon_j)_{j\ge1}$ be i.i.d. Rademacher signs ($\mathbb P(\varepsilon_j=\pm1)=1/2$), independent of everything else.
   Choose a summable coefficient sequence $a_j>0$ with $\sum_j a_j<\infty$ (e.g. $a_j=2^{-j}$).
   Define

   $$
   f(x)=\sum_{j=1}^\infty \varepsilon_j\,a_j\,\psi_j(x).
   $$

   Since $\sum_j a_j\|\psi_j\|_\infty<\infty$, the series converges **uniformly**, hence $f$ is continuous.

4. **Define the random measurable set**

   $$
   A:=\{x\in[0,1]: f(x)>0\}.
   $$

 Verification

* Fix any nondegenerate subinterval $I\subset[0,1]$. Because $\bigcup_j I_j$ is dense, there exists some $j$ with $I_j\subset I$.
  On $I_j$ we have

  $$
  f=\varepsilon_j a_j \psi_j \quad\text{(since the bumps are supported on disjoint intervals),}
  $$

  so $f$ takes both positive and negative values on $I_j$ (because $\psi_j$ does). Hence $ \{f>0\}\cap I$ and $\{f<0\}\cap I$ are **nonempty open** subsets of $I$, and therefore

  $$
  0<\mu(A\cap I)<\mu(I).
  $$

* This holds for **every** interval $I$, not just rational ones, because every interval contains some $I_j$.

* The randomness (the $\varepsilon_j$’s) is only to emphasize the *probabilistic* nature; the property in fact holds for **every** choice of signs since each $\psi_j$ already changes sign on $I_j$.

 Remarks

* This gives an explicit **probabilistic** family of well-distributed sets $A$ (in Rudin’s sense).
* You can tune $\mu(A)$ by adding a small constant $c$ and using $A=\{f>c\}$; because the $I_j$ are disjoint, the same argument works.
* If you prefer a “stationary-looking” version, you can replace the disjoint bumps by a uniformly convergent random trigonometric (or Haar) series with small coefficients and then take $A=\{f>0\}$; the same “both signs in every interval” argument goes through once you arrange that on every small interval a single high-frequency term dominates the **variation** (not necessarily the value) of the lower-frequency part.

\end{GPT-5-thinking-response}
\begin{solution}
 Let $X_1, X_2, \ldots$ be i.i.d. coin flips, i.e. $X_1$ is 1 with probability $1 / 2$ and -1 with probability $1 / 2$. Consider the (random) series

$$
S:=\sum_{n=1}^{\infty} X_n / n .
$$

By the Kolmogorov three-series theorem, it converges almost surely. However, it's a simple exercise to see that for any $a$, the event $\{S>a\}$ has non-trivial measure: for $a>0$, there's a positive chance of the first $e^a$ terms of the series being positive, so the $e^a$-th partial sum is positive, and the tail is independent and positive or negative with equal probability, due to symmetry. For $a \leq 0$, it's trivial, again because of symmetry.

A common way of realizing i.i.d. coin flips on the unit interval is as Rademacher functions: for $x \in[0,1]$, let $b_n$ be its binary expansion, and $X_n(x)=(-1)^{b_n}$. Realized this way, the random sum $S$ becomes an almost everywhere finite measurable function from $[0,1]$ to $\mathbb{R}$. It only takes a bit more work to see that the set $\{S>a\}$ is exactly a well-distributed set.   
\end{solution}
    
\begin{rubricbox}
Award 2 points if the solution provides a valid probabilistic construction of a measurable set $A \subset [0,1]$ satisfying both conditions:

1. $A$ is measurable.

2. For every subinterval $I \subset [0,1]$,

   $$
   0 < \mu(A \cap I) < \mu(I).
   $$

\end{rubricbox}

\begin{scorebox}
    2/2
\end{scorebox}

\begin{evalbox}
    The construction using random power series is correct.
\end{evalbox}

\ProblemSection{\green{Problem 8b2 }}
\begin{problembox}[Proposed by Jiaxin Zhang]
Two random points, one red and one blue, are chosen uniformly and independently from the interior of a square. To ten decimal places, what is the probability that there exists a point on the side of the square closest to the blue point that is equidistant to both the blue point and the red point?   

\pass
\end{problembox}
\begin{GPT-5-thinking-response}Here’s a clean geometric reformulation and the resulting number.

$$
\textbf{Model.}\quad \text{Let the square be }Q=[0,1]^2.
$$

By symmetry, condition on the blue point $B=(b_x,b_y)$ lying in the wedge

$$
T=\{(x,y)\in Q:\ y\le x,\ y\le 1-x\},
$$

so the side of $Q$ closest to $B$ is the bottom side $\{(x,0):0\le x\le 1\}$.
For a red point $R=(r_x,r_y)\in Q$, the locus of points equidistant from $B$ and $R$ is the perpendicular bisector of $\overline{BR}$. This line meets the bottom side iff there exists $x_0\in[0,1]$ with

$$
|(x_0,0)-B|^2=|(x_0,0)-R|^2.
$$

Equivalently, if we define the linear function

$$
f(x)=|(x,0)-B|^2-|(x,0)-R|^2
=2x\,(r_x-b_x)+\bigl(b_x^2+b_y^2-r_x^2-r_y^2\bigr),
$$

then the event is $\,0\in[f(0),f(1)]$, i.e. $f(0)\,f(1)\le 0$.
Geometrically, this is exactly the condition that $R$ lies in the **symmetric difference**
of the quarter–disks

$$
D_0=\{(x,y)\in Q:\ x^2+y^2\le d_0^2\},\qquad
D_1=\{(x,y)\in Q:\ (x-1)^2+y^2\le d_1^2\},
$$

with radii $d_0=|B-(0,0)|=\sqrt{b_x^2+b_y^2}$ and
$d_1=|B-(1,0)|=\sqrt{(1-b_x)^2+b_y^2}$.

For fixed $B$, the measure (in $Q$) of this symmetric difference can be written as

$$
A(B)=\int_0^1 \Bigl(1-\bigl|s_0(y)+s_1(y)-1\bigr|\Bigr)\,dy,
$$

where

$$
s_0(y)=\min\!\bigl(1,\sqrt{\max(0,d_0^2-y^2)}\bigr),\qquad
s_1(y)=\min\!\bigl(1,\sqrt{\max(0,d_1^2-y^2)}\bigr).
$$

Therefore, with $B$ uniform in $T$ and $R$ uniform in $Q$,

$$
\mathbb{P}=\frac{1}{\operatorname{area}(T)}\iint_{T} A(B)\,dB
\quad\text{with}\quad \operatorname{area}(T)=\frac14.
$$

Numerically evaluating the above integral (via stratified, vectorized sampling of $(B,R)$ with symmetry reduction to the wedge $T$) gives

$$
\boxed{\,0.4913872625\,}
$$

to ten decimal places.

Absolutely — here’s a clean, fully analytic computation.

Let the square be $Q=[0,1]^2$. By symmetry, condition on the blue point

$$
B=(u,v)\in T:=\{(x,y)\in Q:\ y\le x,\ y\le 1-x\},
$$

so the closest side is the bottom side $\{(x,0):0\le x\le 1\}$. Write

$$
r_0=\sqrt{u^2+v^2}=|B-(0,0)|,\qquad r_1=\sqrt{(1-u)^2+v^2}=|B-(1,0)|.
$$

For a red point $R$, the perpendicular bisector of $\overline{BR}$ hits the bottom side iff

$$
\big(r_0^2-|R|^2\big)\big(r_1^2-|R-(1,0)|^2\big)\le 0,
$$

i.e. $R$ lies in the symmetric difference of the two disks $D_0=\{ |z|<r_0\}$ and $D_1=\{|z-(1,0)|<r_1\}$, intersected with $Q$.

Because $B\in T$ implies $r_0\le 1$ and $r_1\le 1$, we have

$$
\operatorname{area}(D_0\cap Q)=\frac{\pi r_0^2}{4},\qquad
\operatorname{area}(D_1\cap Q)=\frac{\pi r_1^2}{4},
$$

and $D_0\cap D_1\cap Q$ is the upper half of the full two–circle lens with center distance $1$. Hence for fixed $B$,

$$
A(B)=\frac{\pi}{4}(r_0^2+r_1^2)-\operatorname{Lens}(r_0,r_1;1),
$$

where the full lens area equals

$$
\operatorname{Lens}(r_0,r_1;1)=r_0^2\arccos\!\frac{u}{r_0}+r_1^2\arccos\!\frac{1-u}{r_1}-v.
$$

(Here we used the facts that the intersection points of the two circles are $(u,\pm v)$, so the triangle height is $v$, and $\arccos(u/r_0)=\arg(B)$, $\arccos((1-u)/r_1)=\arg((1,0)-B)$.)

Therefore

$$
A(u,v)=v+\frac{\pi}{4}\big(1-2u+2u^2+2v^2\big)-(u^2+v^2)\arctan\frac{v}{u}-\big((1-u)^2+v^2\big)\arctan\frac{v}{1-u}.
$$

The desired probability is the average of $A(B)$ over $B$ with the nearest side chosen; by symmetry of the four wedges,

$$
\mathbb{P}=4\iint_{T} A(u,v)\,du\,dv.
$$

Now pass to polar coordinates at the origin on $T$:

$$
u=r\cos\phi,\quad v=r\sin\phi,\quad 0\le \phi\le \frac{\pi}{4},\quad 0\le r\le \frac{1}{\cos\phi+\sin\phi},\quad du\,dv=r\,dr\,d\phi.
$$

Set $r_{\max}(\phi)=1/(\cos\phi+\sin\phi)$. Note

$$
r_0=r,\quad r_1^2=1-2r\cos\phi+r^2,\quad
\arctan\frac{v}{u}=\phi,\quad \arctan\frac{v}{1-u}=\phi'.
$$

By symmetry (reflecting across $x=\tfrac12$ inside $T$), the integrals $\iint_T r_1^2\,du\,dv$ and $\iint_T r_1^2\phi'\,du\,dv$ equal $\iint_T r_0^2\,du\,dv$ and $\iint_T r_0^2\phi\,du\,dv$, respectively. Hence we only need the following $\phi$-integrals:

$$
\begin{aligned}
I_1&=\int_0^{\pi/4}\frac{\sin\phi}{(\cos\phi+\sin\phi)^3}\,d\phi,\\
I_2&=\int_0^{\pi/4}\frac{\cos\phi}{(\cos\phi+\sin\phi)^3}\,d\phi,\\
I_3&=\int_0^{\pi/4}\frac{d\phi}{(\cos\phi+\sin\phi)^4},\\
I_4&=\int_0^{\pi/4}\frac{\phi\,d\phi}{(\cos\phi+\sin\phi)^4}.
\end{aligned}
$$

With the substitution $\cos\phi+\sin\phi=\sqrt{2}\cos(\phi-\tfrac{\pi}{4})$, set $u=\phi-\tfrac{\pi}{4}$ (so $u\in[-\tfrac{\pi}{4},0]$) to get

$$
\begin{aligned}
I_1&=\frac14\int_{-\pi/4}^{0}\big(\tan u\,\sec^2u+\sec^2u\big)\,du=\frac18,\\[2mm]
I_2&=\frac14\int_{-\pi/4}^{0}\big(\sec^2u-\tan u\,\sec^2u\big)\,du=\frac38,\\[2mm]
I_3&=\frac14\int_{-\pi/4}^{0}\sec^4u\,du=\frac13,\\[2mm]
I_4&=\frac14\int_{-\pi/4}^{0}u\,\sec^4u\,du+\frac{\pi}{16}\int_{-\pi/4}^{0}\sec^4u\,du
      =\frac{1}{12}\ln 2+\frac{1}{24}.
\end{aligned}
$$

(For $I_4$, integrate by parts using $\int \sec^4u\,du=\tan u+\tfrac13\tan^3u$.)

A short bookkeeping (expressing all needed $r,r^2,\phi$ moments through $I_1,\dots,I_4$) yields

$$
\int_T A\,du\,dv=\frac{1}{3}I_1+\frac{\pi}{16}-\frac{\pi}{6}I_2+\frac{\pi}{8}I_3-\frac12 I_4,
$$

so

$$
\mathbb{P}=4\int_T A\,du\,dv=\frac{4}{3}I_1+\frac{\pi}{4}-\frac{2\pi}{3}I_2+\frac{\pi}{2}I_3-2I_4.
$$

Plugging the values of $I_1,\dots,I_4$ gives the exact closed form

$$
\boxed{\ \mathbb{P}\;=\;\frac{\pi}{6}+\frac{1}{12}-\frac{1}{6}\ln 2\ }.
$$

Numerically,

$$
\mathbb{P}\approx 0.4914075788\quad\text{(to ten decimal places).}
$$

\end{GPT-5-thinking-response}
\begin{solution}
 See the reference
  \href{Jane Street Puzzle}{https://www.janestreet.com/puzzles/beside-the-point-solution/}  
For our November puzzle you had to compute the probability that two ordered points, randomly uniformly selected from a square, would have an equidistant point on the edge closest to the first point in the pair. Given the first point, without loss of generality uniformly selected from a lower triangular octant of the square outlined in black above, the second point would need to be in the blue-shaded symmetric difference of two circles centered at the bottom vertices of the square that pass through the first point.
$$\begin{aligned}
& 8 \cdot\left[\int_0^1 \int_0^x \frac{x^2+y^2}{2}\left(\frac{\pi}{2}-2 \tan ^{-1} \frac{y}{x}\right)+\frac{x y}{2}\right. \\
&+ \frac{(1-x)^2+y^2}{2}\left(\frac{\pi}{2}-2 \tan ^{-1} \frac{x}{1-x}\right) \\
&\left.+\frac{(1-x) y}{2} d y d x\right]
\end{aligned}
$$
The precision requested for the answer meant an analytical solution likely would be required. Below the diagram is an integral expressing the total area of these blue-shaded regions, which comes to $(1+2 \pi-\ln (4)) / 12$, or approximately $0.4914075788 \ldots$
 \end{solution}
\begin{rubricbox}
1. Transform the problem into a valid geometry problem.

2. Set up the integral for the area of the region.

3. Compute the result.
\end{rubricbox}

\begin{scorebox}
3/3
\end{scorebox}
\begin{evalbox}
It correctly transform it to a geometry problem and compute the exact value of the integral.
\end{evalbox}

\subsection{Learning New Knowledge}
\subsubsection{Reading and understanding advanced materials}

\ProblemSection{\green{Problem 9a1 }}
\begin{problembox}[Proposed by Jiaxin Zhang]
You are reading the paper \href{https://arxiv.org/pdf/2006.08574}{Large deviations of multichordal $\mathrm{SLE}_{0+}$, real rational functions, and zeta-regularized determinants of Laplacians}~\citep{peltolawang24} and encounter the following concept: 

\begin{quote}
\textbf{Loewner potential}
\end{quote}
\begin{enumerate}
    \item Identify the concept and explain its context. 
    \item What are the relations between Loewner energy and Loewner potential? 
  \item How do we apply it to the deterministic limit of multiple chordal SLE systems? 
\end{enumerate}
\pass
 \par
\end{problembox}
\begin{GPT-5-thinking-response}
Loewner potential and its role

1. Definition of the Loewner potential

For a multichord $\gamma=(\gamma_1,\dots,\gamma_n)$ in a domain $D$ with boundary points 
$x_1<\dots<x_{2n}$ paired by a link pattern $\alpha$, the \emph{Loewner potential} is the functional
\begin{equation}\label{eq:LoewnerPotential}
H_D(\gamma)
=\frac{1}{12}\sum_{j=1}^n I_{D;x_{a_j},x_{b_j}}(\gamma_j)
\;+\; m_D(\gamma)
\;-\;\frac14\sum_{j=1}^{n}\log P_{D;x_{a_j},x_{b_j}},
\end{equation}
where $I_{D;x,y}(\gamma)$ is the (single-chord) Loewner energy, $m_D(\gamma)$ is a Brownian loop–measure interaction term, and $P_{D;x,y}$ denotes the Poisson excursion kernel. 

For smooth chords, $H_D$ also has a spectral representation:
\begin{equation}\label{eq:Spectral}
H_D(\gamma)
= \log\det\nolimits_{\zeta}\Delta_D
\;-\!\!\sum_{C\ \mathrm{comp.\ of }D\setminus \bigcup_i\gamma_i}\!\!
\log\det\nolimits_{\zeta}\Delta_C
\;-\; n\,\lambda,\qquad \lambda=\tfrac12\log\pi.
\end{equation}
Thus $H$ behaves as a conformally covariant ``global potential’’ attached to the configuration.

2. Relation to Loewner energy

For a single chord, the \emph{Loewner energy} is the Dirichlet energy of the (capacity–parametrized) driving function $W$:
\begin{equation}\label{eq:SingleChordEnergy}
I_{D;x,y}(\gamma)=\frac12\int_0^{\infty}\dot W_t^2\,dt,
\end{equation}
which vanishes exactly when $\gamma$ is the hyperbolic geodesic between $x$ and $y$.

For multiple chords, the \emph{multichordal Loewner energy} is the good rate function in the large deviation principle. It is defined from the potential by subtracting the minimal possible value of $H$ for the given boundary data and halving:
\begin{equation}\label{eq:MultiChordEnergy}
M_D^{\alpha}(x_1,\dots,x_{2n}) := \inf_{\gamma\in X_\alpha} H_D(\gamma),\qquad
I_D^{\alpha}(\gamma) := \frac12\Big(H_D(\gamma)-M_D^{\alpha}(x_1,\dots,x_{2n})\Big)\ge 0.
\end{equation}
Hence $H$ encodes both the individual chord energies and their interaction, while $I$ is the normalized, non-negative rate function. Moreover, the minimal potential admits a semiclassical interpretation:
\begin{equation}
-\lim_{\kappa\to0^+}\,\kappa\,\log Z_{\alpha}(D;x_1,\dots,x_{2n})
=\tfrac12\,M_D^{\alpha}(x_1,\dots,x_{2n}),
\end{equation}
so that $U:=\tfrac12 M_D^{\alpha}$ acts as a Hamilton–Jacobi potential for the system.

3. Deterministic limit of multiple chordal SLE

\begin{itemize}
  \item \textbf{Large deviations and limit shape.} The laws of multichordal $\mathrm{SLE}_\kappa$ satisfy a strong LDP with good rate function $I_D^{\alpha}$. As $\kappa\to0^+$, the random multichords converge (in probability, in the Hausdorff topology) to the unique minimizer of $H_D$ in $X_\alpha$—a geodesic multichord.

  \item \textbf{Rational function description.} In the upper half-plane with marked points $x_1<\dots<x_{2n}$, each minimizer/geodesic multichord is the \emph{real locus of a real rational function} of degree $n+1$ with critical points $\{x_1,\dots,x_{2n}\}$. This ensures uniqueness (one for each link pattern) and ties the limit to real enumerative geometry (Shapiro’s conjecture).

  \item \textbf{Deterministic dynamics and PDE.} Writing $U=\tfrac12 M_D^{\alpha}$, the driving function $W_t$ of the minimizing chord solves
  \[
  \dot W_t = -\,\partial_{a_j}U(V_1(t),\dots,W_t,\dots,V_{2n}(t)),
  \qquad \dot V_i(t)=\frac{2}{V_i(t)-W_t},
  \]
  and $U$ satisfies the system of quadratic first-order PDEs that arise as the $c\to-\infty$ (semiclassical) limit of the level-two BPZ equations. These encode the deterministic “hydrodynamics’’ of the $\kappa\to0$ limit.
\end{itemize}

\medskip
In summary, the Loewner potential $H$ is the global conformally covariant functional whose minimizers describe the deterministic shapes emerging in the small-$\kappa$ limit; the Loewner energy $I$ is the corresponding non-negative rate function governing fluctuations around them.

\end{GPT-5-thinking-response}

\begin{solution}
See the paper \href{https://arxiv.org/pdf/2006.08574}{Large deviations of multichordal $\mathrm{SLE}_{0+}$, real rational functions, and zeta-regularized determinants of Laplacians}.
Loewner potential and its role

1. Definition of the Loewner potential

For a multichord $\gamma=(\gamma_1,\dots,\gamma_n)$ in a domain $D$ with boundary points 
$x_1<\dots<x_{2n}$ paired by a link pattern $\alpha$, the \emph{Loewner potential} is the functional
\begin{equation}
H_D(\gamma)
=\frac{1}{12}\sum_{j=1}^n I_{D;x_{a_j},x_{b_j}}(\gamma_j)
\;+\; m_D(\gamma)
\;-\;\frac14\sum_{j=1}^{n}\log P_{D;x_{a_j},x_{b_j}},
\end{equation}
where $I_{D;x,y}(\gamma)$ is the (single-chord) Loewner energy, $m_D(\gamma)$ is a Brownian loop–measure interaction term, and $P_{D;x,y}$ denotes the Poisson excursion kernel. 

For smooth chords, $H_D$ also has a spectral representation:
\begin{equation}
H_D(\gamma)
= \log\det\nolimits_{\zeta}\Delta_D
\;-\!\!\sum_{C\ \mathrm{comp.\ of }D\setminus \bigcup_i\gamma_i}\!\!
\log\det\nolimits_{\zeta}\Delta_C
\;-\; n\,\lambda,\qquad \lambda=\tfrac12\log\pi.
\end{equation}
Thus $H$ behaves as a conformally covariant ``global potential’’ attached to the configuration.

2. Relation to Loewner energy

For a single chord, the \emph{Loewner energy} is the Dirichlet energy of the (capacity–parametrized) driving function $W$:
\begin{equation}
I_{D;x,y}(\gamma)=\frac12\int_0^{\infty}\dot W_t^2\,dt,
\end{equation}
which vanishes exactly when $\gamma$ is the hyperbolic geodesic between $x$ and $y$.

For multiple chords, the \emph{multichordal Loewner energy} is the good rate function in the large deviation principle. It is defined from the potential by subtracting the minimal possible value of $H$ for the given boundary data and halving:
\begin{equation}
M_D^{\alpha}(x_1,\dots,x_{2n}) := \inf_{\gamma\in X_\alpha} H_D(\gamma),\qquad
I_D^{\alpha}(\gamma) := \frac12\Big(H_D(\gamma)-M_D^{\alpha}(x_1,\dots,x_{2n})\Big)\ge 0.
\end{equation}
Hence $H$ encodes both the individual chord energies and their interaction, while $I$ is the normalized, non-negative rate function. Moreover, the minimal potential admits a semiclassical interpretation:
\begin{equation}
-\lim_{\kappa\to0^+}\,\kappa\,\log Z_{\alpha}(D;x_1,\dots,x_{2n})
=\tfrac12\,M_D^{\alpha}(x_1,\dots,x_{2n}),
\end{equation}
so that $U:=\tfrac12 M_D^{\alpha}$ acts as a Hamilton–Jacobi potential for the system.

3. Deterministic limit of multiple chordal SLE

\begin{itemize}
  \item \textbf{Large deviations and limit shape.} The laws of multichordal $\mathrm{SLE}_\kappa$ satisfy a strong LDP with good rate function $I_D^{\alpha}$. As $\kappa\to0^+$, the random multichords converge (in probability, in the Hausdorff topology) to the unique minimizer of $H_D$ in $X_\alpha$—a geodesic multichord.

  \item \textbf{Rational function description.} In the upper half-plane with marked points $x_1<\dots<x_{2n}$, each minimizer/geodesic multichord is the \emph{real locus of a real rational function} of degree $n+1$ with critical points $\{x_1,\dots,x_{2n}\}$. This ensures uniqueness (one for each link pattern) and ties the limit to real enumerative geometry (Shapiro’s conjecture).

  \item \textbf{Deterministic dynamics and PDE.} Writing $U=\tfrac12 M_D^{\alpha}$, the driving function $W_t$ of the minimizing chord solves
  \[
  \dot W_t = -\,\partial_{a_j}U(V_1(t),\dots,W_t,\dots,V_{2n}(t)),
  \qquad \dot V_i(t)=\frac{2}{V_i(t)-W_t},
  \]
  and $U$ satisfies the system of quadratic first-order PDEs that arise as the $c\to-\infty$ (semiclassical) limit of the level-two BPZ equations. These encode the deterministic “hydrodynamics’’ of the $\kappa\to0$ limit.
\end{itemize}

\medskip
In summary, the Loewner potential $H$ is the global conformally covariant functional whose minimizers describe the deterministic shapes emerging in the small-$\kappa$ limit; the Loewner energy $I$ is the corresponding non-negative rate function governing fluctuations around them.
\end{solution}

\begin{rubricbox}
\begin{enumerate}
  \item Identify the concept and explain its context. 
  \item Explain the relations between Loewner energy and potential.
  \item Explain the large deviation principle.

\end{enumerate}  
Each item is worth 1 point.
\end{rubricbox}

\begin{scorebox}
3/3
\end{scorebox}

\begin{evalbox}
The response clearly answered three questions.
\end{evalbox}

\ProblemSection{\green{Problem 9a2 }}
\begin{problembox}[Proposed by Jiaxin Zhang]
You are reading the paper \href{https://arxiv.org/abs/2505.16093}{ Multiple chordal SLE($\kappa$) and quantum Calogero-Moser system } and encounter the following concept: \textbf{Local multiple SLE systems}.

\begin{enumerate}
  \item Identify the concept and explain its context. 
  \item Restate the assumptions formally.
  \item Why do we call it a local system? 
  \item How do we define the compatible family of measures?
\end{enumerate}  

\pass
\end{problembox}

\begin{GPT-5-thinking-response}
\textbf{Set-up.} Let $\Omega\subset\mathbb C$ be simply connected; let $z_1,\dots,z_n\in\partial\Omega$ be distinct boundary prime ends (``seed'' points), and let $u\in\partial\Omega$ be an additional marked boundary point (used for normalization, e.g.\ $u=\infty$ in $\mathbb H$).
For each $j$, fix a closed neighborhood $U_j\subset\overline\Omega$ of $z_j$ with $U_i\cap U_j=\varnothing$ for $i\neq j$ and $u\notin U_j$.

\medskip
\textbf{1) Concept \& context.}
A \emph{local multiple chordal SLE$(\kappa)$ system} assigns, to every choice of disjoint neighborhoods $U_1,\dots,U_n$ as above, a law on the $n$ \emph{initial curve pieces} that start at $z_j$ and are grown only until they exit $U_j$.
This assignment is required to be conformally covariant, to satisfy a local domain Markov/commutation property, and to be absolutely continuous with respect to independent chordal SLEs via a smooth positive \emph{partition function}.
It provides a way to specify the joint law of several interfaces \emph{near their seeds} without a global specification of full curves; it is complementary to the ``global'' construction (a single measure on full $n$-tuples of disjoint curves).

\medskip
\textbf{2) Formal assumptions (axioms).}
A \emph{local multiple chordal SLE$(\kappa)$} on $(\Omega;z_1,\dots,z_n;u)$ is a family of probability measures
\[
\mathbf P^{(U_1,\dots,U_n)}_{\Omega;z_1,\dots,z_n;u}
\]
on $n$-tuples of simple curves $(\eta^{(1)},\dots,\eta^{(n)})$ such that each $\eta^{(j)}$ starts at $z_j$, is contained in $U_j$ up to its exit time from $U_j$, and:

\begin{itemize}
  \item \emph{Conformal covariance.} For any conformal map $\phi:\Omega\to\Omega'$ which respects the chosen normalization (e.g.\ $\phi(u)=u'$), the push-forward of $\mathbf P^{(U_1,\dots,U_n)}_{\Omega;z_1,\dots,z_n;u}$ is
  \[
  \mathbf P^{(\phi(U_1),\dots,\phi(U_n))}_{\Omega';\phi(z_1),\dots,\phi(z_n); \phi(u)}.
  \]

  \item \emph{Local domain Markov / commutation.} If one reveals initial segments of the curves up to stopping times $\tau_j$ that are bounded by the exit times from $U_j$, then conditionally on these segments, the laws of the remaining pieces in the slit domain are given by the \emph{same} family with updated neighborhoods. In particular, the growth order of tips can be permuted (``commute'') and yields the same joint law on initial pieces.

  \item \emph{Absolute continuity via a partition function.} In the canonical normalization $\Omega=\mathbb H$, $u=\infty$, with boundary seeds $x_1<\cdots<x_n\in\mathbb R$, there exists a positive smooth function (partition function)
  \[
  \psi=\psi(x_1,\dots,x_n)>0
  \]
  such that the Loewner driving processes for the local growth of tip $j$ are given by
  \[
  dX_j(t)=\sqrt{\kappa}\,dB_j(t)+\kappa\,\partial_j\!\log\psi\big(X(t)\big)\,dt,
  \qquad
  \dot X_k(t)=\frac{2}{X_k(t)-X_j(t)}\quad (k\neq j),
  \]
  where $B_j$ are independent standard Brownian motions and the other boundary points evolve by the usual Loewner transport.
  The function $\psi$ solves the BPZ/null-vector system, for each $j=1,\dots,n$,
  \[
  \frac{\kappa}{2}\,\partial_{jj}\psi\;+\;\sum_{k\ne j}\frac{2}{x_k-x_j}\,\partial_{k}\psi
  \;+\;h\sum_{k\ne j}\frac{\psi}{(x_k-x_j)^2}\;=\;0,
  \qquad h=\frac{6-\kappa}{2\kappa},
  \]
  along with the natural scaling covariance. Multiplying $\psi$ by a positive function depending only on $u$ leaves the local system unchanged (a gauge freedom).
\end{itemize}

\medskip
\textbf{3) Why ``local''?}
The system specifies the joint law \emph{only for the initial segments inside the small neighborhoods} $U_j$ (and only up to their exit times). It does not define a global law for full curves to their target endpoints. Moreover, the commutation property is \emph{local in time}: one may grow a short piece of any chosen tip, update the domain, and iterate.

\medskip
\textbf{4) Compatible family of measures.}
If $(U_1,\dots,U_n)$ and $(U'_1,\dots,U'_n)$ are two neighborhood tuples with $U_j\subset U'_j$ for all $j$, then \emph{compatibility} means:
\[
\mathbf P^{(U_1,\dots,U_n)}_{\Omega;z_1,\dots,z_n;u}
\quad=\quad
\text{the push-forward of }\
\mathbf P^{(U'_1,\dots,U'_n)}_{\Omega;z_1,\dots,z_n;u}
\text{ by restriction to first exit of }U_j,
\]
i.e.\ if $\tau_j$ denotes the exit time of $\eta^{(j)}$ from $U_j$ under $\mathbf P^{(U'_1,\dots,U'_n)}$, then the law of
\(
\big(\eta^{(1)}[0,\tau_1],\dots,\eta^{(n)}[0,\tau_n]\big)
\)
is exactly $\mathbf P^{(U_1,\dots,U_n)}$.
This restriction consistency holds for any nested choice of neighborhoods and is preserved by conformal maps.

\medskip
\textbf{Remarks.}
\begin{itemize}
  \item For $\kappa\le 4$, the local construction is known to match a global multiple-SLE measure on $n$ disjoint simple curves, with partition functions selecting the connectivity pattern.
  \item In the cited paper, specific choices of $\psi$ are related to eigenfunctions of (quantum) Calogero--Moser-type operators; the BPZ system above is the analytic avatar of that structure.
\end{itemize}
\end{GPT-5-thinking-response}
\begin{solution}
See the paper \href{https://arxiv.org/abs/2505.16093}{\textit{Multiple chordal SLE($\kappa$) and quantum Calogero–Moser system}}.
\textbf{Set-up.} Let $\Omega\subset\mathbb C$ be simply connected; let $z_1,\dots,z_n\in\partial\Omega$ be distinct boundary prime ends (``seed'' points), and let $u\in\partial\Omega$ be an additional marked boundary point (used for normalization, e.g.\ $u=\infty$ in $\mathbb H$).
For each $j$, fix a closed neighborhood $U_j\subset\overline\Omega$ of $z_j$ with $U_i\cap U_j=\varnothing$ for $i\neq j$ and $u\notin U_j$.

\medskip
\textbf{1) Concept \& context.}
A \emph{local multiple chordal SLE$(\kappa)$ system} assigns, to every choice of disjoint neighborhoods $U_1,\dots,U_n$ as above, a law on the $n$ \emph{initial curve pieces} that start at $z_j$ and are grown only until they exit $U_j$.
This assignment is required to be conformally covariant, to satisfy a local domain Markov/commutation property, and to be absolutely continuous with respect to independent chordal SLEs via a smooth positive \emph{partition function}.
It provides a way to specify the joint law of several interfaces \emph{near their seeds} without a global specification of full curves; it is complementary to the ``global'' construction (a single measure on full $n$-tuples of disjoint curves).

\medskip
\textbf{2) Formal assumptions (axioms).}
A \emph{local multiple chordal SLE$(\kappa)$} on $(\Omega;z_1,\dots,z_n;u)$ is a family of probability measures
\[
\mathbf P^{(U_1,\dots,U_n)}_{\Omega;z_1,\dots,z_n;u}
\]
on $n$-tuples of simple curves $(\eta^{(1)},\dots,\eta^{(n)})$ such that each $\eta^{(j)}$ starts at $z_j$, is contained in $U_j$ up to its exit time from $U_j$, and:

\begin{itemize}
  \item \emph{Conformal covariance.} For any conformal map $\phi:\Omega\to\Omega'$ which respects the chosen normalization (e.g.\ $\phi(u)=u'$), the push-forward of $\mathbf P^{(U_1,\dots,U_n)}_{\Omega;z_1,\dots,z_n;u}$ is
  \[
  \mathbf P^{(\phi(U_1),\dots,\phi(U_n))}_{\Omega';\phi(z_1),\dots,\phi(z_n); \phi(u)}.
  \]

  \item \emph{Local domain Markov / commutation.} If one reveals initial segments of the curves up to stopping times $\tau_j$ that are bounded by the exit times from $U_j$, then conditionally on these segments, the laws of the remaining pieces in the slit domain are given by the \emph{same} family with updated neighborhoods. In particular, the growth order of tips can be permuted (``commute'') and yields the same joint law on initial pieces.

  \item \emph{Absolute continuity via a partition function.} In the canonical normalization $\Omega=\mathbb H$, $u=\infty$, with boundary seeds $x_1<\cdots<x_n\in\mathbb R$, there exists a positive smooth function (partition function)
  \[
  \psi=\psi(x_1,\dots,x_n)>0
  \]
  such that the Loewner driving processes for the local growth of tip $j$ are given by
  \[
  dX_j(t)=\sqrt{\kappa}\,dB_j(t)+\kappa\,\partial_j\!\log\psi\big(X(t)\big)\,dt,
  \qquad
  \dot X_k(t)=\frac{2}{X_k(t)-X_j(t)}\quad (k\neq j),
  \]
  where $B_j$ are independent standard Brownian motions and the other boundary points evolve by the usual Loewner transport.
  The function $\psi$ solves the BPZ/null-vector system, for each $j=1,\dots,n$,
  \[
  \frac{\kappa}{2}\,\partial_{jj}\psi\;+\;\sum_{k\ne j}\frac{2}{x_k-x_j}\,\partial_{k}\psi
  \;+\;h\sum_{k\ne j}\frac{\psi}{(x_k-x_j)^2}\;=\;0,
  \qquad h=\frac{6-\kappa}{2\kappa},
  \]
  along with the natural scaling covariance. Multiplying $\psi$ by a positive function depending only on $u$ leaves the local system unchanged (a gauge freedom).
\end{itemize}

\medskip
\textbf{3) Why ``local''?}
The system specifies the joint law \emph{only for the initial segments inside the small neighborhoods} $U_j$ (and only up to their exit times). It does not define a global law for full curves to their target endpoints. Moreover, the commutation property is \emph{local in time}: one may grow a short piece of any chosen tip, update the domain, and iterate.

\medskip
\textbf{4) Compatible family of measures.}
If $(U_1,\dots,U_n)$ and $(U'_1,\dots,U'_n)$ are two neighborhood tuples with $U_j\subset U'_j$ for all $j$, then \emph{compatibility} means:
\[
\mathbf P^{(U_1,\dots,U_n)}_{\Omega;z_1,\dots,z_n;u}
\quad=\quad
\text{the push-forward of }\
\mathbf P^{(U'_1,\dots,U'_n)}_{\Omega;z_1,\dots,z_n;u}
\text{ by restriction to first exit of }U_j,
\]
i.e.\ if $\tau_j$ denotes the exit time of $\eta^{(j)}$ from $U_j$ under $\mathbf P^{(U'_1,\dots,U'_n)}$, then the law of
\(
\big(\eta^{(1)}[0,\tau_1],\dots,\eta^{(n)}[0,\tau_n]\big)
\)
is exactly $\mathbf P^{(U_1,\dots,U_n)}$.
This restriction consistency holds for any nested choice of neighborhoods and is preserved by conformal maps.

\medskip
\textbf{Remarks.}
\begin{itemize}
  \item For $\kappa\le 4$, the local construction is known to match a global multiple-SLE measure on $n$ disjoint simple curves, with partition functions selecting the connectivity pattern.
  \item In the cited paper, specific choices of $\psi$ are related to eigenfunctions of (quantum) Calogero--Moser-type operators; the BPZ system above is the analytic avatar of that structure.
\end{itemize}
\end{solution}

\begin{rubricbox}
\begin{enumerate}
  \item Identify the concept and explain its context. 
  \item Restate the assumptions formally.
  \item Explain why it is called a local system. 
  \item Define the compatible family of measures.
\end{enumerate}  
Each item is worth 1 point.
\end{rubricbox}

\begin{scorebox}
4/4
\end{scorebox}

\begin{evalbox}
All items are correct.
\end{evalbox}
\subsubsection{Self-directed learning}
\ProblemSection{\red{Problem 9b1 }}
\begin{problembox}[Proposed by Jiaxin Zhang]
    You are reading  \href{https://arxiv.org/abs/2505.16093}{Multiple chordal SLE($\kappa$) and quantum Calogero-Moser system}~\citep{zhang2025multiplechordalsle} and encounter the following concept:

\begin{quote}
\textbf{Local chordal multiple SLE systems}
\end{quote}
\begin{enumerate}
    \item Please generalize the concept to the multiple radial case.
    \item Please define the compatible family of measures for multiple radial SLE systems.
\end{enumerate}

\fail
\end{problembox}

\begin{GPT-5-thinking-response}
\textbf{Can you generalize the concept to multiple radial case} 
Setup.
Fix $\kappa>0$.  In the unit disk $\D=\{z:|z|<1\}$, mark $n\ge1$ boundary points
$X=(x_1,\dots,x_n)$ with $x_j=\e^{i\theta_j}\in\T$ (listed in counterclockwise
order) and a target point at the origin $0$.
Write $\Theta=(\theta_1,\dots,\theta_n)\in\R^n/2\pi\mathbb{Z}$ for the angle
coordinates, defined modulo a common rotation.

Define the radial Loewner vector field acting on a boundary angle $\phi$ from a
driving angle $\vartheta$ by
\[
V^{\mathrm{rad}}(\phi\mid \vartheta)\;:=\;2\,\cot\!\frac{\phi-\vartheta}{2}\,.
\]

Radial multiple partition functions.
Let $h=\dfrac{6-\kappa}{2\kappa}$.
A \emph{radial $n$-point partition function} is a positive $C^2$ function
$Z=Z(\Theta)$ on the configuration space (modulo global rotation) such that:

\begin{itemize}
\item[\textnormal{(Cov)}] (\emph{Rotation covariance}) $Z$ is invariant under
$\theta_j\mapsto\theta_j+\alpha$ for all $j$ and all $\alpha\in\R$.
\item[\textnormal{(BPZ)}] (\emph{Radial null-state / Ward system})
For each $i=1,\dots,n$,
\begin{equation}\label{eq:radial-BPZ}
\Bigg[\frac{\kappa}{2}\,\partial_{\theta_i}^2
+\sum_{j\neq i} V^{\mathrm{rad}}(\theta_j\mid\theta_i)\,\partial_{\theta_j}
+2h\sum_{j\neq i}\partial_{\theta_j} V^{\mathrm{rad}}(\theta_j\mid\theta_i)
\Bigg] Z \;=\;0.
\end{equation}
Equivalently, in half–angle form,
\[
\frac{\kappa}{2}\,\partial_{\theta_i}^2 Z
+\sum_{j\neq i}\!\cot\!\frac{\theta_j-\theta_i}{2}\,\partial_{\theta_j}Z
+\frac{6-\kappa}{\kappa}\sum_{j\neq i}\!\csc^2\!\frac{\theta_j-\theta_i}{2}\,Z
=0.
\]
\item[\textnormal{(Fusion)}] (\emph{Boundary OPE asymptotics})
When $\theta_{i+1}\downarrow\theta_i$, one has
\[
Z(\Theta)\;\asymp\;(\theta_{i+1}-\theta_i)^{\frac{6-\kappa}{2\kappa}}
\;Z(\Theta^{(i,i+1)}),
\]
where $\Theta^{(i,i+1)}$ is the configuration with the pair $(i,i+1)$ fused.
\end{itemize}

Definition (Local radial multiple SLE$(\kappa)$).
A \emph{local radial multiple SLE$(\kappa)$ system with partition function $Z$}
is a family of random growing hulls in $\D$ described as follows.

On any time interval during which the \emph{active index} is $i\in\{1,\dots,n\}$,
we run the radial Loewner equation with capacity parametrization $g_t'(0)=\e^t$:
\begin{equation}\label{eq:radial-Loewner}
\partial_t g_t(z)=
-\,g_t(z)\,\frac{\e^{iW_t}+g_t(z)}{\e^{iW_t}-g_t(z)},\qquad g_0(z)=z,
\end{equation}
driven by an angle process $W_t\in\R/2\pi\mathbb{Z}$ solving
\begin{equation}\label{eq:driving-SDE}
\dd W_t=\sqrt{\kappa}\,\dd B_t
+\kappa\,\partial_{\theta_i}\log Z(\Theta_t)\,\dd t,
\end{equation}
where $B_t$ is a standard one–dimensional Brownian motion.  The non–active
boundary angles evolve deterministically under the Loewner flow:
\begin{equation}\label{eq:passive-angles}
\dd \Theta_t^{\,j}
=
V^{\mathrm{rad}}\!\big(\Theta_t^{\,j}\,\big|\,W_t\big)\,\dd t,
\qquad j\neq i,
\end{equation}
and the active label tracks the tip:
$\Theta_t^{\,i}=W_t$.
(Equivalently, $\Theta_t$ solves~\eqref{eq:passive-angles} for all $j$
with the convention $\Theta_t^{\,i}=W_t$.)

The \emph{locality} means we only grow one tip at a time; the \emph{multiple} system
is obtained by concatenating such infinitesimal steps while allowing the active
index to change.  The defining requirement is the \emph{commutation} (or
\emph{consistency}) of different growth orders: the laws obtained by first
growing near $x_i$ for a small time and then near $x_j$, and vice versa, agree
up to a deterministic time–change.  This commutation is equivalent to the BPZ
system~\eqref{eq:radial-BPZ} for $Z$, and conversely~\eqref{eq:radial-BPZ}
ensures the Markovian concatenations commute.

Conformal covariance and general domains.
Let $D$ be simply connected with target $z_\ast\in D$ and boundary marks
$x_1,\dots,x_n\in\partial D$.  Fix a conformal map
$\psi:D\to\D$ with $\psi(z_\ast)=0$.
Transport the construction above via $\psi$; the capacity time is the radial
capacity at $z_\ast$, i.e.\ the normalization $g_t(z_\ast)=0,\ g_t'(z_\ast)=\e^t$.
Under postcomposition by a disk automorphism fixing $0$ (a rotation), $Z$ is
invariant; under a general $\psi$ the boundary tensorial factor is absorbed in the
definition of $Z_D$ by the usual weight $h$ at boundary insertions.

Remarks.
\begin{itemize}
\item For $n=1$ and $Z\equiv\text{const}$ the SDE~\eqref{eq:driving-SDE} reduces to
$dW_t=\sqrt{\kappa}\,dB_t$, and we recover the standard radial SLE$(\kappa)$
from $x_1$ to $0$.
\item The operators in~\eqref{eq:radial-BPZ} are the radial counterparts of the
chordal ones (replace $\frac{2}{x_j-x_i}$ and $\frac{2}{(x_j-x_i)^2}$ by
$V^{\mathrm{rad}}(\theta_j\mid\theta_i)$ and $-\partial_{\theta_j}
V^{\mathrm{rad}}(\theta_j\mid\theta_i)=\csc^2\frac{\theta_j-\theta_i}{2}$).
\item For $\kappa\le 4$, the local system can be \emph{globalized} into $n$ simple,
pairwise disjoint curves aimed at $0$; for $\kappa>4$ one obtains non–simple
curves with the usual SLE self–touching features.  As in the chordal case,
distinct pure solutions $Z$ correspond to different topological connectivities.
\item Via a Cayley transform between $\D$ and $\mathbb{H}$ sending $0$ to
$\infty$, this definition matches the chordal local multiple SLE after the
standard change of variables of the vector fields and boundary weights.
\end{itemize}

\medskip
\noindent
\textbf{Summary.}
\emph{Local radial multiple SLE$(\kappa)$} is obtained from the chordal notion by:
work in $\D$ with radial capacity, replace the boundary vector field
$2/(x-\xi)$ by $2\cot\!\big(\frac{\theta-\vartheta}{2}\big)$, prescribe the
driving SDE $dW_t=\sqrt{\kappa}\,dB_t+\kappa\,\partial_{\theta_i}\log Z\,dt$,
move the passive angles by the radial vector field, and require the partition
function $Z$ to satisfy the radial BPZ/Ward system~\eqref{eq:radial-BPZ} and
fusion asymptotics.  The commutation of local growths at different tips is then
equivalent to~\eqref{eq:radial-BPZ}, exactly paralleling the chordal theory.

\textbf{Can you define the compatible family of measures}
Fix $\kap>0$ and $n\ge1$.  A \emph{radial $n$--configuration} is a triple
\[
(D; x_1,\dots,x_n; z_\ast),
\]
where $D\subsetneq\mathbb{C}$ is simply connected, $z_\ast\in D$, and
$x_1,\dots,x_n\in\partial D$ are distinct boundary points listed
counterclockwise.
Choose a conformal chart $\psi:D\to\D$ with $\psi(z_\ast)=0$ and
$\psi(x_j)=\e^{i\theta_j}$; write
$\Theta=(\theta_1,\dots,\theta_n)\in (\R/2\pi\mathbb{Z})^n/\mathrm{Rot}$ for the angles
modulo overall rotation.  Set the \emph{radial vector field}
\[
\Vrad(\phi\mid\vartheta)\;:=\;2\,\cot\!\frac{\phi-\vartheta}{2},
\qquad \hwt:=\frac{6-\kap}{2\kap}.
\]

Radial multiple partition function.
A positive $C^2$ function $Z=Z(\Theta)$ on the configuration space is a
\emph{radial multiple partition function} if for each $i=1,\dots,n$,
\begin{equation}
\Bigg[\frac{\kap}{2}\,\partial_{\theta_i}^2
+\sum_{j\ne i}\Vrad(\theta_j\mid\theta_i)\,\partial_{\theta_j}
+2\hwt\sum_{j\ne i}\partial_{\theta_j}\Vrad(\theta_j\mid\theta_i)
\Bigg] Z \;=\;0,
\end{equation}
and $Z$ is invariant under a common rotation of all angles.  (These are the
radial BPZ/Ward equations.)

Base (one–leg) radial SLE dynamics.
Fix an index $i$.  In the disk chart, let $(g_t)_{t\ge0}$ be the radial Loewner
flow normalized by $g_t(0)=0$, $g_t'(0)=\e^t$,
\[
\partial_t g_t(z)
=-\,g_t(z)\,\frac{\e^{iW_t}+g_t(z)}{\e^{iW_t}-g_t(z)},\qquad g_0(z)=z,
\]
driven by $W_t=\sqrt{\kap}\,B_t$; the passive angles evolve by
\begin{equation}\label{eq:passive}
\dd \Theta_t^{\,j}=\Vrad\!\big(\Theta_t^{\,j}\,\big|\,W_t\big)\,\dd t,\qquad
j\ne i,\qquad \Theta_t^{\,i}=W_t\ (\mathrm{mod}\ 2\pi).
\end{equation}
Let $\mathcal{F}_t$ be the Loewner filtration and $T_\ast$ the \emph{continuation
threshold} (first time some marked angle collides with the tip or is swallowed).
Write $\Law^{(i)}_{D;X}$ for the law of the hull generated by~\eqref{eq:passive}
up to $T_\ast$ (standard radial SLE$_\kap$ from $x_i$).

\begin{definition}[Compatible family for local radial multiple SLE$_\kap$ with $Z$]
\label{def:compatible-radial}
Let $Z$ satisfy \eqref{eq:radial-BPZ}.  A \emph{compatible family of measures}
is a collection
\[
\Big\{\;\Law^{(i),Z}_{D;X}(\,\cdot\,;\tau)\;:\;
(D;X;z_\ast)\ \text{a radial $n$--configuration},\ i\in\{1,\dots,n\},\
\tau\ \text{$\mathcal{F}_t$–stopping, } \tau<T_\ast\;\Big\}
\]
on initial Loewner hulls $K[0,\tau]$ such that the following hold.
\begin{enumerate}
\item[\textnormal{(AC)}] \textbf{Local absolute continuity.}
For each $(D;X;z_\ast)$, $i$ and $\tau<T_\ast$,
\[
\frac{\dd \Law^{(i),Z}_{D;X}(\,\cdot\,;\tau)}
     {\dd \Law^{(i)}_{D;X}(\,\cdot\,;\tau)}\Bigg|_{\mathcal{F}_\tau}
=\frac{M^{(i)}_\tau}{M^{(i)}_0},
\qquad
M^{(i)}_t\;:=\;Z(\Theta_t)\,\prod_{j\ne i}\big|g_t'(\psi(x_j))\big|^{\hwt},
\]
where $M^{(i)}$ is an $\mathcal{F}_t$–local martingale under
$\Law^{(i)}_{D;X}$ (by \eqref{eq:radial-BPZ}).  Equivalently, under
$\Law^{(i),Z}_{D;X}$ the driving satisfies
\[
\dd W_t=\sqrt{\kap}\,\dd B_t+\kap\,\partial_{\theta_i}\log Z(\Theta_t)\,\dd t,
\quad \text{with }\Theta_t \text{ evolving by \eqref{eq:passive}.}
\]
\item[\textnormal{(DM)}] \textbf{Domain Markov/concatenation.}
If $0\le\sigma\le\tau<T_\ast$ are stopping times, then under
$\Law^{(i),Z}_{D;X}(\,\cdot\,;\tau)$ and conditionally on $\mathcal{F}_\sigma$,
the \emph{future} hull (recentered by $g_\sigma$ and with updated marked points
$X_\sigma=g_\sigma(X)$) has law
$\Law^{(i),Z}_{g_\sigma(D\setminus K[0,\sigma]);\,X_\sigma}(\,\cdot\,;\tau-\sigma)$.
In particular, the measure obtained by sampling to time $\sigma$ and then
continuing with the same family to time $\tau$ coincides with
$\Law^{(i),Z}_{D;X}(\,\cdot\,;\tau)$.
\item[\textnormal{(Comm)}] \textbf{Commutation of local growths.}
For any distinct $i\neq j$ and bounded stopping times $\tau_i,\tau_j<T_\ast$,
the joint law of the pair of initial hulls obtained by the two orders
\[
\big(\text{first grow near }x_i\text{ to }\tau_i\ \text{then }x_j\text{ to }\tau_j\big)
\quad\text{and}\quad
\big(\text{first }x_j\text{ then }x_i\big)
\]
agree after the obvious conformal recentering and a deterministic capacity
time–change.  (Equivalently, the $\sigma$–algebras generated by these two-step
procedures define the same projective family.)  This commutation holds iff $Z$
solves the BPZ system \eqref{eq:radial-BPZ}.
\item[\textnormal{(Cov)}] \textbf{Conformal covariance.}
If $\phi:D\!\to\!D'$ is conformal with $\phi(z_\ast)=z'_\ast$ and
$\phi(x_k)=x'_k$, then the image measure of
$\Law^{(i),Z}_{D;X}(\,\cdot\,;\tau)$ under $\phi$ equals
$\Law^{(i),Z_{D'}}_{D';X'}(\,\cdot\,;\tau')$, where $\tau'$ is the radial
capacity time in $D'$ (i.e.\ normalized by $g_t'(z'_\ast)=\e^{t}$), and
$Z_{D'}$ is the transport of $Z$ with boundary weight $\hwt$ at each marked
point (rotation invariance in $\D$).
\end{enumerate}
\end{definition}

\begin{itemize}
\item The choice of the product factor $\prod_{j\ne i}|g_t'(\psi(x_j))|^{\hwt}$
matches the $2\hwt\sum_{j\ne i}\partial_{\theta_j}\Vrad$ term in
\eqref{eq:radial-BPZ}; in the angular coordinate dynamics~\eqref{eq:passive},
$M^{(i)}_t$ is a local martingale under the base law, hence (AC) is a Girsanov
tilt producing the drift $\kap\,\partial_{\theta_i}\log Z$.
\item For $n=1$ and constant $Z$, the family reduces to the usual radial
SLE$_\kap$.  For $\kap\le4$, compatible families associated with \emph{pure}
$Z$ globalize to $n$ simple curves aimed at $z_\ast$; different pure $Z$ encode
different connectivity types.  The chordal notion is recovered from the radial
one via a Cayley transform sending $z_\ast$ to $\infty$.
\end{itemize}
\end{GPT-5-thinking-response}

\begin{solution}
Definition 1.1 (Localization of Measures). Let $\Omega \subsetneq \mathbb{C}$ be a simply connected domain with an interior marked point $u \in \Omega$. Let $z_1, z_2, \ldots, z_n$ denote distinct prime ends of $\partial \Omega$, and let $U_1, U_2, \ldots, U_n$ be closed neighborhoods of $z_1, z_2, \ldots, z_n$ in $\Omega$ such that:
- $U_i \cap U_j=\emptyset$ for all $1 \leq i<j \leq n$,
- None of the $U_j$ contain the interior point $q$.

We consider the measures

$$
\mathbb{P}_{\left(\Omega ; z_1, z_2, \ldots, z_n, q\right)}^{\left(U_1, U_2, \ldots, U_n\right)}
$$

defined on $n$-tuples of unparametrized continuous curves in $\Omega$. Each curve $\eta^{(j)}$ begins at $z_j$ and exits $U_j$ almost surely.

A family of such measures indexed by different choices of ( $U_1, U_2, \ldots, U_n$ ) is called compatible if for all $U_j \subset U_j^{\prime}$, the measure

$$
\mathbb{P}_{\left(\Omega ; z_1, z_2, \ldots, z_n, q\right)}^{\left(U_1, U_2, \ldots, U_n\right)}
$$

is obtained by restricting the curves under

$$
\mathbb{P}_{\left(\Omega ; z_1, z_2, \ldots, z_n, q\right)}^{\left(U_1^{\prime}, U_2^{\prime}, \ldots, U_n^{\prime}\right)}
$$

to the portions of the curves that remain inside the subdomains $U_j$ before their first exit.
Definition 1.2 (Local multiple radial $\operatorname{SLE}(\kappa)$ ). The locally commuting $n$-radial $\operatorname{SLE}(\kappa)$ is a compatible family of measures

$$
\mathbb{P}_{\left(\Omega ; z_1, z_2, \ldots, z_n, q\right)}^{\left(U_1, U_2, \ldots, U_n\right)}
$$

on $n$-tuples of continuous, non-self-crossing curves $\left(\gamma^{(1)}, \gamma^{(2)}, \ldots, \gamma^{(n)}\right)$ for all simply connected domains $\Omega$ with marked points ( $z_1, z_2, \ldots, z_n, q$ ) and target sets ( $U_1, U_2, \ldots, U_n$ ). These measures satisfy the following conditions:
(i) Conformal invariance: If $\varphi: \Omega \rightarrow \Omega^{\prime}$ is a conformal map, then the pullback measure satisfies

$$
\varphi^* \mathbb{P}_{\left(\Omega^{\prime} ; \varphi\left(z_1\right), \varphi\left(z_2\right), \ldots, \varphi\left(z_n\right), \varphi(q)\right)}^{\left(\varphi\left(U_1\right), \varphi\left(U_2\right), \ldots, \varphi\left(U_n\right)\right)}=\mathbb{P}_{\left(\Omega ; z_1, z_2, \ldots, z_n, u\right)}^{\left(U_1, U_2, \ldots, U_n\right)} .
$$

It suffices to describe the measure when $\left(\Omega ; z_1, z_2, \ldots, z_n, q\right)=\left(\mathbb{D} ; z_1, z_2, \ldots, z_n, 0\right)$. The definition for arbitrary $\Omega$ with a marked interior point $q$ can then be extended by pulling back via a conformal equivalence $\varphi: \Omega \rightarrow \mathbb{D}$ mapping $q$ to 0 .
(ii) Domain Markov property: Let $\left(\gamma^{(1)}, \gamma^{(2)}, \ldots, \gamma^{(n)}\right) \sim \mathbb{P}_{\left(\mathbb{D} ; z_1, z_2, \ldots, z_n, q\right)}^{\left(U_1, U_2, \ldots, U_n\right)}$, and parametrize $\gamma^{(j)}$ by their own capacity in $\mathbb{D}$. For stopping times $\mathbf{t}=\left(t_1, t_2, \ldots, t_n\right)$, define

$$
\tilde{U}_j=U_j \backslash \gamma_{\left[0, t_j\right]}^{(j)}, \quad \tilde{\gamma}^{(j)}=\gamma^{(j)} \backslash \gamma_{\left[0, t_j\right]}^{(j)}, \quad \tilde{\Omega}=\mathbb{D} \backslash \bigcup_{j=1}^n \gamma_{\left[0, t_j\right]}^{(j)}
$$

Then, conditionally on the initial segments $\bigcup_{j=1}^n \gamma_{\left[0, t_j\right]}^{(j)}$, we have

$$
\left(\tilde{\gamma}^{(1)}, \tilde{\gamma}^{(2)}, \ldots, \tilde{\gamma}^{(n)}\right) \sim \mathbb{P}_{\left(\tilde{\Omega} ; \gamma_{t_1}^{(1)}, \gamma_{t_2}^{(2)}, \ldots, \gamma_{t_n}^{(n)}, q\right)}^{\left(\tilde{U}_1, \tilde{U}_2, \ldots, \tilde{U}_n\right)}
$$

(iii) Absolute Continuity with respect to independent SLE $(\kappa):$ Let $\left(\gamma^{(1)}, \gamma^{(2)}, \ldots, \gamma^{(n)}\right) \sim \mathbb{P}_{\left(\mathbb{D} ; z_1, z_2, \ldots, z_n, 0\right)}^{\left(U_1, U_2, \ldots, U_n\right)}$. Let $z_j(t)=e^{i \theta_j(t)}$, the capacity-parametrized Loewner driving function $t \mapsto \theta_j(t)$ for $\gamma^{(j)}$ satisfies

$$
\begin{aligned}
\mathrm{d} \theta_j(t) & =\sqrt{\kappa} \mathrm{d} B_j(t)+b_j(\boldsymbol{\theta}(t)) \mathrm{d} t \\
\mathrm{~d} \theta_k(t) & =\cot \left(\frac{\theta_k(t)-\theta_j(t)}{2}\right) \mathrm{d} t, \quad k \neq j
\end{aligned}
$$

where $B_j(t)$ are independent standard Brownian motions, and $b_j(\boldsymbol{\theta})$ are $C^2$ functions on the chamber

$$
\mathfrak{X}^n=\left\{\left(\theta_1, \theta_2, \ldots, \theta_n\right) \in \mathbb{R}^n \mid \theta_1<\theta_2<\cdots<\theta_n<\theta_1+2 \pi\right\} .
$$
Theorem 1.3. For a local multiple radial $S L E(\kappa)$ system in the unit disk $\mathbb{D}$ with boundary points $z_1=e^{i \theta_1}, z_2=e^{i \theta_2}, \ldots, z_n=e^{i \theta_n}$ and a marked point at $q=0$, there exists a positive partition function $\psi(\boldsymbol{\theta})$ such that the drift term $b_j$ in equation (1.2) satisfies

$$
b_j=\kappa \frac{\partial_j \psi}{\psi}, \quad j=1,2, \ldots, n
$$

Moreover, $\psi(\boldsymbol{\theta})$ satisfies the null vector equation

$$
\frac{\kappa}{2} \partial_{i i} \psi+\sum_{j \neq i} \cot \left(\frac{\theta_j-\theta_i}{2}\right) \partial_i \psi+\left(1-\frac{6}{\kappa}\right) \sum_{j \neq i} \frac{1}{4 \sin ^2\left(\frac{\theta_j-\theta_i}{2}\right)} \psi-h \psi=0
$$

for some constant $h$.
Furthermore, there exists a real constant $\omega$ such that for all $\theta \in \mathbb{R}$,

$$
\psi\left(\theta_1+\theta, \ldots, \theta_n+\theta\right)=e^{-\omega \theta} \psi\left(\theta_1, \ldots, \theta_n\right)
$$

5

Conversely, given a positive partition function $\psi(\boldsymbol{\theta})$ satisfying both the null vector equation (1.2) and the rotation invariance condition (1.3), consider the multiple radial Loewner chain
the multiple radial SLE $(\kappa)$ Loewner chain as a normalized conformal map $g_t=g_t(z)$, with the initial condition $g_0(z)=z$ and the evolution given by the Loewner equation

$$
\partial_t g_t(z)=\sum_{j=1}^n \nu_j(t) g_t(z) \frac{z_j(t)+g_t(z)}{z_j(t)-g_t(z)}, \quad g_0(z)=z
$$

The Loewner chain for the covering map $h_t(z)=-i \log \left(g_t\left(e^{i z}\right)\right)$ is given by

$$
\partial_t h_t(z)=\sum_{j=1}^n \nu_j(t) \cot \left(\frac{h_t(z)-\theta_j(t)}{2}\right), \quad h_0(z)=z
$$

driven by the functions $\theta_j(t)$, for $j=1, \ldots, n$, evolving as

$$
d \theta_j=\nu_j(t) \frac{\partial_j \log \psi(\boldsymbol{\theta})}{\partial \theta_j} d t+\sum_{k \neq j} \nu_k(t) \cot \left(\frac{\theta_j-\theta_k}{2}\right) d t+\sqrt{\kappa} d B_t^j
$$

where $\boldsymbol{\nu}=\left(\nu_1, \ldots, \nu_n\right)$ is a set of capacity parametrizations, with each $\nu_i:[0, \infty) \rightarrow[0, \infty)$ assumed to be measurable. Here, $B^j(t)$ denotes a set of independent Brownian motion.

This process defines a local multiple radial SLE $(\kappa)$ system.

\end{solution}
\begin{rubricbox}
\begin{enumerate}
      \item Specifying the radial Loewner equation in the unit disk.  
      \item Define the local compatible measure.
      \item Derive the driving SDEs and commutation relations.
    \end{enumerate}
Each item is worth 1 point.
\end{rubricbox}

\begin{scorebox}
    2/3
\end{scorebox}

\begin{evalbox}
     The definition of compatible family of measure is incorrect.
\end{evalbox}

\subsection{Mathematical Modeling}
\subsubsection{Translating real-world problems into math}
\ProblemSection{\red{Problem 10a1 }}
\begin{problembox}[Proposed by Jiaxin Zhang]

Try to solve the following mathematical modeling problem (MCM 2024 Problem A):

While some animal species exist outside of the usual male or female sexes, most species are substantially either male or female. Although many species exhibit a 1:1 sex ratio at birth, other species deviate from an even sex ratio. This is called adaptive sex ratio variation. For example, the temperature of the nest incubating eggs of the American alligator influences the sex ratios at birth.

The role of lampreys is complex. In some lake habitats, they are seen as parasites with a significant impact on the ecosystem, whereas lampreys are also a food source in some regions of the world, such as Scandinavia, the Baltics, and for some Indigenous peoples of the Pacific Northwest in North America.

The sex ratio of sea lampreys can vary based on external circumstances. Sea lampreys become male or female depending on how quickly they grow during the larval stage. These larval growth rates are influenced by the availability of food. In environments where food availability is low, growth rates will be lower, and the percentage of males can reach approximately $78 \%$ of the population. In environments where food is more readily available, the percentage of males has been observed to be approximately $56 \%$ of the population.

We focus on the question of sex ratios and their dependence on local conditions, specifically for sea lampreys. Sea lampreys live in lake or sea habitats and migrate up rivers to spawn. The task is to examine the advantages and disadvantages of the ability for a species to alter its sex ratio depending on resource availability. Your team should develop and examine a model to provide insights into the resulting interactions in an ecosystem.
\begin{enumerate}
     \item What is the impact on the larger ecological system when the population of lampreys can alter its sex ratio?
    \item What are the advantages and disadvantages to the population of lampreys?
    \item What is the impact on the stability of the ecosystem given the changes in the sex ratios of lampreys?
\end{enumerate}
   \fail
\end{problembox}
\begin{GPT-5-thinking-response}
 Executive summary

I build a **sex-structured, resource-coupled lamprey–prey model** that links (i) larval food availability to (ii) larval growth and the **male fraction** at adulthood, which then (iii) sets the **effective number of breeders** and (iv) feeds back to prey fish via parasitic mortality. Two robust takeaways:

* **Adaptive sex ratio is a built-in negative feedback**: when larval food is scarce, the male fraction rises (toward \~0.78), which **reduces effective breeders** and throttles recruitment; when food is plentiful, the male fraction falls (toward \~0.56), **raising effective breeders** and allowing growth. This self-regulation tends to **dampen boom–bust cycles** compared to a fixed sex ratio.
* **But over-sensitivity is risky**: prolonged male bias lowers effective population size $N_e$ and can create **mate limitation** (an Allee effect) at low abundance, raising extinction risk.

Below is a compact model you can simulate and analyze. I end with answers to the three questions, sensitivity knobs, and management levers.

---

 Modeling assumptions (minimal, transparent)

1. Annual time-step, four state variables:

   * $P_t$: prey fish biomass (hosts) in lake/sea,
   * $R_t$: larval food resource (detritus/invertebrates) in streams,
   * $L_t$: lamprey larvae (ammocoetes),
   * $A_t$: total spawning adults returning to streams.

2. **Larval food → growth → sex ratio**: higher $R_t$ raises larval growth $g(R_t)$ and **decreases** the male fraction $p_m(R_t)$ (empirical range = 0.56–0.78).

3. Reproduction is **female-egg limited but mate-constrained**: effective breeders are well-approximated by a **harmonic-mean mating function**, maximizing at 1:1.

4. Adults parasitize prey fish; prey follow logistic growth.

5. Density dependence enters via resource depletion by larvae and parasitic mortality on prey.

---

 Variables, functions, parameters

* Male fraction (adaptive):

$$
p_m(R)=p_{\min}+(p_{\max}-p_{\min})\frac{h}{R+h},\qquad
p_{\min}=0.56,\; p_{\max}=0.78,\; h>0.
$$

So $p_m\to0.78$ as $R\to0$ and $p_m\to0.56$ as $R\to\infty$.

* Larval growth \& maturation:

$$
g(R)=g_{\max}\frac{R}{R+K_g},\qquad m_t=\mu\,g(R_t)\in(0,1)
$$

(matured fraction per year), with larval survival $s_L$.

* Effective breeders (harmonic-mean mating function): with total adults $A_t$, males $M_t=p_m(R_{t-\ell})A_t$ and females $F_t^{(f)}=(1-p_m)A_t$,

$$
B_t=\frac{2\,M_t\,F_t^{(f)}}{M_t+F_t^{(f)}}=2A_t\,p_m(R_{t-\ell})\big(1-p_m(R_{t-\ell})\big),
$$

which is maximized at $p_m=\tfrac12$.
Note the **penalty term** $2p(1-p)$ equals $0.343$ at $p=0.78$ and $0.493$ at $p=0.56$ (=30

* Adult fecundity and viability depend on host abundance $P_t$:

$$
f(P)=\frac{P}{P+K_f}\in(0,1).
$$

* Prey dynamics (logistic with parasitism):

$$
P_{t+1}=P_t+rP_t\!\left(1-\frac{P_t}{K}\right)-\alpha A_t\frac{P_t}{P_t+H}.
$$

* Resource dynamics (replenishment minus larval consumption):

$$
R_{t+1}=R_t+\rho(R_{\max}-R_t)-c_L L_t.
$$

* Lamprey demography:

$$
\begin{aligned}
L_{t+1}&= s_A\,b\,f(P_t)\,B_t\;\;-\;\;(m_t+d_L)L_t,\\
A_{t+1}&= s_L\,m_t\,L_t\;\;-\;\;d_A A_t.
\end{aligned}
$$

**Parameters:** $r,K$ (prey), $\alpha,H$ (parasitism), $\rho,R_{\max},c_L$ (resource), $g_{\max},K_g,\mu,s_L,s_A,b,d_L,d_A$ (life history), and $h$ (sex-ratio sensitivity).

---

 What the model shows (analysis \& intuition)

 1) Reproductive number and feedback

Linearizing near an interior equilibrium $(P^*,R^*,L^*,A^*)$, the **per-adult replacement factor** contains

$$
\mathcal{R}\;\propto\; s_A\,b\,f(P^*)\cdot \underbrace{2p^*(1-p^*)}_{\text{sex-ratio penalty}}\cdot s_L\,m(R^*).
$$

Because $p^*=p_m(R^*)$ **decreases with $R^*$** and $m(R^*)$ **increases with $R^*$**, the composite $\mathcal{R}$ rises with good conditions and falls with poor conditions. When **resources crash**, $p^*\uparrow$ (more males), so $2p^*(1-p^*)\downarrow$, which **drops $\mathcal{R}$** even faster than maturation alone would—**a stabilizing negative feedback** that prevents runaway overshoot on hosts.

 2) Stability (qualitative)

* With a **fixed** sex ratio (no feedback), classical predator–prey structure easily yields **sustained cycles**.
* With **adaptive sex ratio**, the extra negative feedback often **shrinks the cycle amplitude** or produces a **stable focus** over a broad parameter set (Routh–Hurwitz conditions improve because a strongly positive cross-partial is replaced by a negative one via $p_m'(R)<0$).
* If $h$ is **too small** (hypersensitive sex-ratio shift), long periods of low $R$ → extreme male bias → **mate limitation**: $2p(1-p)$ becomes small enough to create an **Allee threshold** for $A_t$.

 3) Ecosystem-level effects

* **On hosts (prey fish):** adaptive sex ratio lowers peak lamprey recruitment after resource dips, **relieving host pressure** and reducing the risk of prey collapse. Under chronic eutrophication (high $R$), the system drifts toward \~0.56 male fraction, **raising breeders** and potentially **increasing average lamprey pressure** on hosts.
* **On variance:** adaptive sex ratio generally **reduces variance** in both lamprey and host biomasses, improving predictability—except under very strong sensitivity where male bias sticks and $N_e$ collapses.

---

 Advantages vs. disadvantages for lampreys

**Advantages**

* **Buffering:** Recruitment naturally scales to food via sex ratio → fewer starving juveniles, less wasted investment.
* **Cycle damping:** Smaller booms reduce the severity of subsequent busts; lower extinction risk from resource whiplash.
* **Energy economics:** When food is poor, producing relatively more (energetically cheaper) males can be advantageous until conditions improve.

**Disadvantages**

* **Mate limitation \& Allee risk:** Prolonged male bias depresses $2p(1-p)$, reducing effective breeders and slowing recovery from lows.
* **Genetic concerns:** Lower effective population size $N_e=\frac{4N_mN_f}{N_m+N_f}$ under skewed sex ratios raises drift/inbreeding risk.
* **Management vulnerability:** If managers manipulate $R$ (or equivalent) to bias sex ratios, populations could be steered into chronic low-recruitment regimes.

---

 Scenario experiments to run (quick to implement)

1. **Fixed vs adaptive sex ratio:** Replace $p_m(R)$ by constants $\{0.56,0.67,0.78,0.5\}$ and compare equilibrium/oscillation amplitude and host minima.
2. **Sensitivity $h$:** Sweep $h$ to see when damping flips to Allee dynamics; record extinction frequency under stochastic $R_t$.
3. **Nutrient regimes:** Shift $R_{\max}$ (oligotrophic → eutrophic) to evaluate average male fraction, mean host biomass, and parasitism burden.
4. **Harvest/Control:** Remove a fraction of adults (or larvae) and compare outcomes under fixed vs. adaptive sex ratios.

**Key metrics:** average $P$, minimum $P$, CVs of $P$ and $A$; mean $p_m$; $N_e$; extinction probability over 100-year runs.

---

 Management / policy levers (derived from the model)

* **To protect hosts (e.g., fisheries):**
  Reduce larval food $R$ (habitat restoration that lowers autochthonous detritus; targeted larval control). This pushes $p_m\uparrow$ → **fewer effective breeders** and lower lamprey pressure.
* **To conserve lampreys (if desired):**
  Improve $R$ moderately, but **avoid extremes** that create host crashes or strong cycles.
* **Robustness:** Avoid a sex-ratio sensitivity that is too sharp (very small $h$); otherwise mate limitation risks rise under drought/low $R$.

---

 Direct answers to the three questions

1. **Impact on the larger ecological system**
   Adaptive sex ratio introduces **resource-tracking reproduction** in lampreys. It **buffers host fish** against over-parasitism after poor years (stabilizing) but can **raise average parasitism** in persistently nutrient-rich systems (destabilizing for hosts). Net effect: **reduced cycle amplitude** and **higher resilience** across typical conditions; possible **higher pressure** on hosts under chronic high $R$.

2. **Advantages and disadvantages to lampreys**
   **Pros:** automatic scaling of recruitment to conditions, lower variance and starvation, better long-term persistence.
   **Cons:** when food stays low, male bias lowers effective breeders and $N_e$, creating **Allee-type risks** and slowing recovery; higher genetic drift.

3. **Impact on ecosystem stability**
   Generally **stabilizing** via an extra **negative feedback loop** (low $R$ → male-skew → low $B_t$ → lower $A$ → host recovery → $R\uparrow$). Stability improves unless the sex-ratio response is **too strong**, which can introduce **bistability/thresholds** (mate limitation).

---

 Notes if you want to simulate

Use the discrete system above; set a baseline (e.g., $r=0.6, K=1$, $g_{\max}=0.7, K_g=0.2, \mu=0.9, s_L=0.6, d_A=0.5, d_L=0.3$, $b=2000, s_A$ small so the product is order-1 at equilibrium; $R_{\max}=1, \rho=0.6, c_L=10^{-5}, \alpha=0.3, H=0.2, h=0.2$).
Track $2p(1-p)$ as the “mating penalty” and compare **adaptive vs. fixed** $p_m$.

If you want, I can spin up a quick Python notebook to run these scenarios and plot **cycle amplitudes, host minima, and extinction probabilities**.

\end{GPT-5-thinking-response}
\begin{solution}
\begin{itemize}
    \item For the first question, we can construct a sex‐ratio model that accounts for how resource availability affects the growth rates of female and male sea lampreys, then couple it to an ecosystem‐dynamics model to study how shifts in the sex ratio impact the overall ecosystem.
    \item For the second question, we would analyze the survival advantages and disadvantages of females versus males under varying sex ratios—this likely requires building a survival‐rate model which, together with the resource‐utilization model, examines each sex’s strategy under different resource conditions.
    \item 

For the third question, we must develop an ecosystem‐stability model that incorporates the effects of sex‐ratio changes on food‐chain structure and biodiversity, and then use numerical simulations to explore how different sex ratios influence the system’s stability.
\end{itemize}

\end{solution}
\begin{rubricbox}
    Got 1 point if Problem 1’s sex‐ratio model and its coupling to the lamprey life‐stage submodel are correctly specified.\\
    Got 1 point if Problem 2’s survival‐rate model and resource‐feedback mechanism are correctly formulated.\\
    Got 1 point if Problem 3’s ecosystem‐stability model and numerical simulation setup are correctly described.
\end{rubricbox}

\begin{scorebox}
    2/3
\end{scorebox}

\begin{evalbox}
\begin{enumerate}
  \item \textbf{Sex-ratio--resource--ecosystem coupled model.}  
  A resource-driven framework has been established where maturation rates are determined by resource availability and translated into sex-ratio outcomes through a threshold mechanism.  

  \item \textbf{Female vs.\ male survival under different sex ratios.}  
  Population-level trade-offs, such as buffering effects versus mate limitation and Allee risks, have been discussed.  
  \emph{Limitation:} An explicit sex-specific survival model has not yet been constructed, leaving unaddressed how differential survival between females and males alters these trade-offs.

  \item \textbf{Ecosystem stability and biodiversity.}  
  Local stability intuition and numerical scenarios have been provided, focusing on the host fish--lamprey--resource system.  
  \emph{Limitation:} The trophic scope remains narrow, excluding broader food-web and biodiversity structures, and therefore does not fully capture ecosystem-level stability and feedback mechanisms.
\end{enumerate}
\end{evalbox}

\subsection{Generalization}
\subsubsection{Identifying patterns from examples}
\ProblemSection{\red{Problem 11a1 }}
\begin{problembox}[Proposed by Jiaxin Zhang]
Find the smallest positive integer \( k \) such that for every \( k \)-element subset \( A \) of the set
\[
S = \{1, 2, \dots, 2012\},
\]
there exist three \textbf{distinct} elements \( a, b, c \in S \) such that all three sums
\[
a + b,\quad b + c,\quad c + a
\]
are contained in \( A \). Try small cases to find a pattern then prove your conjecture.

\fail
\end{problembox} 
\begin{GPT-5-thinking-response}
The smallest such \(k\) is \(\boxed{31}\).

\begin{proof}
\textbf{Step 1: Reformulation via same-parity strict triangles.}
Suppose \(x,y,z\in A\) are \emph{distinct} and have the same parity, and satisfy the strict triangle inequalities
\[
x<y+z,\qquad y<z+x,\qquad z<x+y.
\]
Then
\[
a=\frac{y+z-x}{2},\qquad b=\frac{z+x-y}{2},\qquad c=\frac{x+y-z}{2}
\]
are positive integers (same-parity ensures integrality; the strict triangle inequalities ensure positivity), and moreover
\(a+b=x,\ b+c=y,\ c+a=z\).
Since \(1\le x,y,z\le 2012\), we also have \(1\le a,b,c\le 2012\), so \(a,b,c\in S\) and are distinct (because \(x,y,z\) are distinct).
Hence the problem is equivalent to forcing \(A\) to contain three distinct, same-parity elements forming a strict triangle.

Conversely, if \(A\) contains no such triple in a given parity class, then within that parity the largest of any three must be at least the sum of the other two (for odds, equality cannot hold because an odd cannot equal a sum of two odds).

\medskip
\textbf{Step 2: Growth constraints inside one parity.}
Write the even numbers in \(A\) (if any) as a strictly increasing sequence \(e_1<e_2<\cdots\).
Avoiding strict triangles among evens forces
\begin{equation}\label{eq:even}
e_{i+2}\ \ge\ e_{i+1}+e_i\qquad\text{for all }i.
\end{equation}
Similarly, for the odd numbers \(o_1<o_2<\cdots\) we get
\begin{equation}\label{eq:odd}
o_{i+2}\ \ge\ o_{i+1}+o_i+1\qquad\text{for all }i,
\end{equation}
because \(o_{i+1}+o_i\) is even, so equality with an odd \(o_{i+2}\) is impossible.

The longest possible triangle-free sequence (in each parity) up to \(2012\) is obtained greedily by starting as small as possible and always taking the minimal next term permitted by \eqref{eq:even} or \eqref{eq:odd}.

\smallskip
\emph{Evens (start at \(2,4\) and use \eqref{eq:even}):}
\[
2,\,4,\,6,\,10,\,16,\,26,\,42,\,68,\,110,\,178,\,288,\,466,\,754,\,1220,\,1974,\,\boxed{3194}.
\]
Thus at most \(15\) even terms can lie in \(\{1,\dots,2012\}\).

\smallskip
\emph{Odds (start at \(1,3\) and use \eqref{eq:odd}):}
\[
1,\,3,\,5,\,9,\,15,\,25,\,41,\,67,\,109,\,177,\,287,\,465,\,753,\,1219,\,1973,\,\boxed{3193}.
\]
Thus at most \(15\) odd terms can lie in \(\{1,\dots,2012\}\).

Therefore, any subset \(A\subseteq S\) with no same-parity strict triangle satisfies
\[
|A|=|A\cap \text{even}|+|A\cap \text{odd}|\le 15+15=30.
\]

\medskip
\textbf{Step 3: Upper bound for \(k\).}
If \(|A|\ge 31\), then by the pigeonhole principle \(A\) contains at least \(16\) numbers of the same parity. By the bounds above, that parity class must contain a strict triangle, which (by Step~1) yields distinct \(a,b,c\in S\) with \(a+b,b+c,c+a\in A\).
Hence every \(31\)-element subset works, so \(k\le 31\).

\medskip
\textbf{Step 4: Sharpness (a \(30\)-element counterexample).}
Let
\[
A_0=\{2,4,6,10,16,26,42,68,110,178,288,466,754,1220,1974\}
\]
\[
\hspace{1.9cm}\cup\ \{1,3,5,9,15,25,41,67,109,177,287,465,753,1219,1973\}.
\]
Then \(|A_0|=30\), and within each parity the largest of any triple is \(\ge\) the sum of the other two; hence there is no same-parity strict triangle and, by Step~1, no triple \(a,b,c\in S\) with all of \(a+b,b+c,c+a\) in \(A_0\).
Thus \(k\ge 31\).

Combining both directions gives \(k=31\).
\end{proof}

\end{GPT-5-thinking-response}

\begin{solution}
Let \( A = \{1, 2, 3, 5, 7, \dots, 2011\} \). Then \( |A| = 1007 \), and the set \( A \) contains no three elements satisfying condition (1). Therefore, \( k \geq 1008 \).

We now prove that every 1008-element subset must contain three elements satisfying condition (1).

We prove a more general result:

\textbf{For any integer \( n \geq 4 \), every \( (n+2) \)-element subset of the set \( \{1, 2, \dots, 2n\} \) contains three elements satisfying condition (1).}

We proceed by induction on \( n \).

\textbf{Base case: \( n = 4 \).}  
Let \( A \) be a 6-element subset of \( \{1, 2, \dots, 8\} \). Then the intersection \( A \cap \{3, 4, \dots, 8\} \) contains at least 4 elements.

\begin{itemize}
  \item If this intersection contains three even numbers, then \( 4, 6, 8 \in A \), which satisfy condition (1).
  \item If it contains exactly two even numbers, then it must contain at least two odd numbers. Take any two of them. Among \( 4, 6, 8 \), at least two even numbers can form a triple with the odd numbers that satisfies condition (1). Since there are at least two even numbers, such a triple exists.
  \item If it contains only one even number, then it must contain all three odd numbers. The single even number together with \( 5 \) and \( 7 \) form a triple satisfying condition (1).
\end{itemize}

Hence, the statement holds for \( n = 4 \).

\textbf{Inductive step:}  
Assume the statement holds for some \( n \geq 4 \). Consider the case of \( n + 1 \).  
Let \( A \) be an \( (n + 3) \)-element subset of \( \{1, 2, \dots, 2n + 2\} \).

If
\[
|A \cap \{1, 2, \dots, 2n\}| \geq n + 2,
\]
then by the induction hypothesis, the conclusion holds.

So we only need to consider the case:
\[
|A \cap \{1, 2, \dots, 2n\}| = n + 1 \quad \text{and} \quad 2n + 1,\, 2n + 2 \in A.
\]

Now, if there exists an odd number \( x > 1 \) from \( \{1, 2, \dots, 2n\} \) in \( A \), then the triple \( x, 2n + 1, 2n + 2 \) satisfies condition (1).

If all odd numbers greater than 1 are excluded from \( A \), then:
\[
A \subset \{1, 2, 4, 6, \dots, 2n, 2n + 1, 2n + 2\},
\]
and this set contains exactly \( n + 3 \) elements. Therefore,
\[
A = \{1, 2, 4, 6, \dots, 2n, 2n + 1, 2n + 2\}.
\]
In this case, \( 4, 6, 8 \in A \), and they satisfy condition (1).

\medskip
\noindent\textbf{Conclusion:}  
The smallest possible value of \( k \) is
\[
\boxed{1008}.
\]    
\end{solution}
\begin{rubricbox}
1.~Correct construction of the set \(A\) — 1 point.  

2.~Correct proof of the lower bound — 1 point.  
\end{rubricbox}

\begin{scorebox}
0/2
\end{scorebox}

\begin{evalbox}
Both the lower bound and the proof are incorrect.
\end{evalbox}

\subsubsection{Developing unifying theories}
\ProblemSection{\red{Problem 11b1 }}
\begin{problembox}[Proposed by Qiuyu Ren]
Consider the following three ``twist'' operations:
\begin{enumerate}
\item [(1)] If $\Sigma$ is an oriented surface and $\gamma\subset\Sigma$ is an embedded curve, one can perform a Dehn twist along $\gamma$ to create a mapping class of $\Sigma$.
\item [(2)] If $M$ is an oriented $3$-manifold and $T\subset M$ is an embedded torus, $\gamma\subset T$ is an essential simple closed curve, one can perform a Dehn twist along $T$ in the $\gamma$-direction to create a mapping class of $M$. 
\item [(3)] If $X$ is a smooth oriented $4$-manifold and $S\subset X$ is smoothly embedded $2$-sphere with self-intersection number $-2$, one can construct a mapping class of $X$ obtained by the identity outside a tubular neighborhood $\nu S$ of $S$, and the time $\pi$ map of the geodesic flow on $DTS^2\cong\nu S$ inside $\nu S$, where $DTS^2$ denotes the unit disk bundle of the tangent bundle of the unit sphere $S^2$.
\end{enumerate}
What is essentially going on that allows one to perform such twist operations? Please:
\begin{enumerate}
\item [i.] Formulate a generalization of the twist operation that encompasses all three operations above as special cases.
\item [ii.] Formulate a generalization of the first two twist operations that is not a generalization of the third operation.
\end{enumerate}

\fail
\end{problembox}

\begin{GPT-5-thinking-response}
The common feature behind all three “twist” operations is the existence of a submanifold with a well–understood neighborhood, so that there is a canonical compactly supported diffeomorphism of the fiber (often a Dehn twist–type map) which can be applied fiberwise and extended by the identity outside.  

(i) One generalization that encompasses all three can be stated as follows.  
Let \(M\) be a smooth oriented manifold, and let \(E\subset M\) be an embedded sub-bundle \(\pi:E\to B\) with compact fiber \(F\) and trivialized collar of \(\partial E\). Suppose \(\tau_F\in \mathrm{Diff}_c(F,\partial F)\) is a compactly supported diffeomorphism of \(F\) equal to the identity near \(\partial F\). Using a bump function in the normal direction, apply \(\tau_F\) fiberwise (constant along \(B\)), tapering to the identity near \(\partial E\), and extend by the identity outside \(E\). The resulting mapping class \(\tau_{(E,\pi,\tau_F)}\in \pi_0(\mathrm{Diff}_c(M))\) depends only on the data up to isotopy.  

Specializations:  
- If \(M=\Sigma^2\), \(E=\nu\gamma\cong S^1\times I\) viewed as a bundle over \(\gamma\) with fiber an interval, and \(\tau_F\) the standard annular Dehn twist, one recovers the classical surface Dehn twist.  
- If \(M^3\) contains a torus \(T\) with essential curve \(\gamma\), \(E=\nu T\cong T^2\times I\) with fiber \(T^2\) and \(\tau_F\) the torus Dehn twist along \(\gamma\), one recovers the torus twist.  
- If \(X^4\) contains a smoothly embedded \((-2)\)–sphere \(S\), \(E\) is the unit disk bundle of \(TS^2\) with fiber \(D^2\) and \(\tau_F\) the time–\(\pi\) geodesic flow map, one recovers the 4–dimensional sphere twist.  

(ii) A generalization of (1) and (2) that is not a generalization of (3) is the cut–and–reglue construction along a two–sided hypersurface.  
Let \(M^n\) be an oriented manifold and \(H^{n-1}\subset M\) a two–sided embedded hypersurface with a product neighborhood \(\nu H\cong H\times[-1,1]\). Given a mapping class \([\phi]\in \mathrm{MCG}(H,\partial H)\) supported away from \(\partial H\), define a self–homeomorphism of \(M\) by cutting along \(H\) and regluing by \(\phi\). Equivalently, realize it by a diffeomorphism supported in \(\nu H\) that restricts to \(\phi\) on \(H\times\{0\}\) and is the identity near \(\partial(\nu H)\).  

In the case \(n=2\) and \(H\) an annulus around a simple closed curve, one recovers the surface Dehn twist. In the case \(n=3\) and \(H\) a torus, one recovers the torus twist. This framework does not apply to the \((-2)\)–sphere twist, which is supported near a codimension–two submanifold and uses the geodesic–flow model rather than cutting along a hypersurface.  
\end{GPT-5-thinking-response}

\begin{solution}
\begin{enumerate}
\item [i.] Let $M$ be a smooth oriented manifold and $N\subset M$ be a closed embedded oriented submanifold. Given a (say smooth) mapping class $\phi$ of the tubular neighborhood $\nu N$ of $N$ rel boundary, namely a class in $\pi_0(\mathrm{Diff}_\partial(\nu N))$, one can ``implant'' it into $M$ to get a mapping class of $M$ by extending $\phi$ by identity outside $\nu N$.
\item [ii.] In the case when $N\subset M$ is of codimension $1$ and is two-sided, we know $\nu N\cong I\times N$. For any $\alpha\in\pi_1(\mathrm{Diff}(N),id_N)$, one can ``implant'' it near $N$ by picking a based loop in $\mathrm{Diff}$ representing $\alpha$, put it into $I\times N\cong\nu N$ and extend by identity outside $\nu N$.
\end{enumerate}
\end{solution}

\begin{rubricbox}
1 point for a reasonable simultaneous generalization (part i).\\
1 point for a reasonable generalization to (1)(2) that is not a generalization of (3), for example by exploiting the product structure of the tubular neighborhood (part (ii)).
\end{rubricbox}

\begin{scorebox}
0/2
\end{scorebox}

\begin{evalbox}
It is not even clear that the claimed generalization in (i) is well-defined if the fiber bundle is not trivialized. Similarly, the claimed generalization in (ii) does not make sense.
\end{evalbox}

\ProblemSection{\red{Problem 11b2 }}
\begin{problembox}[Proposed by Tahsin Saffat]
    A clock is built based on the following repeating sequence $1,2,3,4,3,2,1,2,3,4,3,2,\dots$. At the first
hour it strikes 1 time. At the second hour it strikes 2 times. At the third hour it strikes 3 times. At
the fourth hour it strikes 4 times. At the fifth hour it strikes 3+2 times. At the sixth hour it strikes
1+2+3 times. Etc. Explain what is going on, develop the theory that explains why it works, and
then suggest some generalizations.

\fail
\end{problembox}

\begin{GPT-5-thinking-response}
    Here’s a clean way to see what that clock is doing and why it never gets “stuck.”

$\#\#$ What the clock is actually doing

Write the infinite periodic word

$$
p=\overline{1,2,3,4,3,2}
$$

and let $a_1,a_2,a_3,\dots$ be the entries of $p$. Define the running sums

$$
S_n=\sum_{i=1}^{n} a_i \qquad(n\ge 1).
$$

Now **cut this infinite word into consecutive blocks** so that the $k$-th block has total $k$. In other words, choose cut points

$$
0=n_0<n_1<n_2<\dots
$$

with

$$
\sum_{i=n_{k-1}+1}^{n_k} a_i=k.
$$

The $k$-th hour simply plays the entries in the $k$-th block. That’s why the first hours look like

* 1: $[1]$
* 2: $[2]$
* 3: $[3]$
* 4: $[4]$
* 5: $[3,2]$
* 6: $[1,2,3]$
* 7: $[4,3]$
* 8: $[2,1,2,3]$
* 9: $[4,3,2]$
* 10: $[1,2,3,4]$
* 11: $[3,2,1,2,3]$
* 12: $[4,3,2,1,2]$,

etc.

Equivalently,

$$
S_{n_k}=\frac{k(k+1)}{2}\quad\text{(the \(k\)-th triangular number).}
$$

$\#\#$ Why it works (the invariant)

Two simple facts about the period $1,2,3,4,3,2$:

1. **Any 6 consecutive terms sum to 15.**
   (That’s the period sum.)

2. **From any starting position, the 6 prefix-sums you can make with 1–6 terms are one of the following six sets**

$$
\begin{aligned}
&\{1,3,6,10,13,15\}, &&\{2,5,9,12,14,15\},\\
&\{3,7,10,12,13,15\}, &&\{4,7,9,10,12,15\},\\
&\{3,5,6,8,11,15\}, &&\{2,3,5,8,12,15\},
\end{aligned}
$$

depending on where in the period you start. **Across the six starts, their union is exactly $\{1,2,\dots,15\}$.**

Now argue by induction on the hour $k$.

* Suppose after hour $k-1$ the total strikes equal $T_{k-1}=\frac{(k-1)k}{2}$. Let the next starting position (within the 6-step period) be $s$.
* We need to add **exactly $k$** more strikes. Write $k=15q+r$ with $r\in\{1,\dots,15\}$.
* Take $q$ whole periods (that adds $15q$ and returns to the same start $s$), then take the **shortest prefix** (1–6 terms) from start $s$ whose sum is $r$. This exists because the “menu” of 6 prefix-sums at start $s$ covers some six distinct residues mod 15, and the process (checked once in the first few hours) lines up so that the needed $r=k\bmod 15$ is always in the current menu.
* Therefore we can always realize “add $k$” with at most $6$ more items (after some number of full periods), so the $k$-th block exists. Consequently $S_{n_k}=T_k$.

That’s all the clock is: it walks along the fixed 6-term wave and, at each hour, takes just enough consecutive entries to add exactly $k$. The two-period facts above ensure this is always possible.

*(A nice side remark: the specific partial sums you see—1,3,6,10,13,15—are triangular numbers up to $T_4$, then $T_5-2$, then $T_5$.)*

$\#\#$ Why this period is special, and how to generalize

The pattern you used is a **“sawtooth” of height $M=4$**:

$$
1,2,\dots,M,\dots,2 \quad \text{(period length }2M-2=6\text{)}.
$$

Its period-sum is

$$
(1+2+\cdots+M) + (M-1+\cdots+2) \;=\; M^2-1.
$$

For $M=4$ this equals $15$, which is the triangular number $T_5$.

This triangularity is exactly what makes the modular bookkeeping so clean. More generally, the same hourly trick works for the sawtooth of height $M$ **whenever the period-sum $M^2-1$ itself is triangular**, i.e.

$$
M^2-1=\frac{t(t+1)}{2}\quad\text{for some integer }t.
$$

Equivalently,

$$
8M^2-7=(2t+1)^2,
$$

so $(2t+1,M)$ is an integer solution of the Pell-type equation

$$
x^2-8y^2=-7.
$$

This has infinitely many solutions, giving an infinite family of valid heights $M$. The first few are

$$
M=2,\; 4,\; 11,\; 64,\; 373,\;\dots
$$

For each such $M$, use the period

$$
1,2,\dots,M,\dots,2
$$

(length $2M-2$, sum $M^2-1=T_t$). Then the exact same induction goes through with “mod $M^2-1$” replacing “mod 15.”

$\#\#\#$ Other directions to play with

* **Different target sequences.** We hit triangular numbers because the needed increment at hour $k$ is exactly $k$. If you want a clock whose hour-$k$ block adds some other prescribed increment $d_k$ (e.g., $k+1$, or bounded sequences), you can ask for a short periodic word whose 1–period prefix sums, taken from the current start, always offer the residue class of $d_k$ modulo the period-sum. This is a tidy finite-state/automaton design problem.

* **Different “waves.”** You can try other symmetric waves (e.g., $1,2,3,4,5,4,3,2$) and check whether their period-sum is triangular. If it is, you’ll get the same phenomenon. If it isn’t, the method eventually fails (you’ll land at a start whose 1–period prefix menu does not contain the needed residue).

* **Minute/quarter chimes.** If you want blocks to line up with a natural cycle (like 60 minutes), pick a sawtooth whose period-sum is $60$ (or divides it and then combine periods), and run the same residue argument.

---

In short: the clock slices a fixed periodic sequence into blocks whose sums are $1,2,3,\dots$. It works because the 6-term wave $1,2,3,4,3,2$ has period-sum $15=T_5$ and, from any start, the six prefix-sums cover exactly one representative of each residue class you will need as you move from $T_{k-1}$ to $T_k$. The same idea works for the sawtooth of any height $M$ with $M^2-1$ triangular—yielding a whole family of “triangular clocks.”

\end{GPT-5-thinking-response}

\begin{solution}

Adapted from \href{https://www.math.utoronto.ca/barbeau/vig13.pdf}{https://www.math.utoronto.ca/barbeau/vig13.pdf}

The repeating sequence is cut into infinitely many blocks such that the $k$th block has sum $k$. We describe all such possible sequences. Suppose we have a sequence whose basic repeating unit is $a_1,\dots,a_n$ with sum $S:=\sum_{i=1}^n a_i$. Then the set $a_1,a_1+a_2,\dots,a_1+a_2+\dots+a_n$   is exactly the set of unique residue classes modulo $N$ of the sequence of triangular numbers

$$ 1,1+2,1+2+3,\dots$$

For example when $N=15$, the set of residues modulo $N$ of triangular numbers is exactly $0,1,3,6,10,13,15$. Taking differences we obtain $n=7$ and the sequence $1,2,3,4,3,2$.

$$-|--|---|----|---|--|-|$$

If $N$ is even, then the sequence you obtain is the two copies of the sequence for $N/2$. For every odd $N$ we obtain a primitive sequence. This is summarized in the table below for small $N$.

\begin{table}[H]
\centering
\label{tab:prague-cycles}
\begin{tabularx}{\textwidth}{r X r}
\toprule
\textbf{Cycle sum} & \textbf{Cycle} & \textbf{Cycle length} \\
\midrule
1  & 1 & 1 \\
3  & 1, 2 & 2 \\
5  & 1, 2, 2 & 3 \\
7  & 1, 2, 3, 1 & 4 \\
9  & 1, 2, 3, 3 & 4 \\
11 & 1, 2, 1, 2, 4, 1 & 6 \\
13 & 1, 1, 1, 3, 2, 2, 3 & 7 \\
15 & 1, 2, 3, 4, 3, 2 & 6 \\
17 & 1, 1, 1, 1, 2, 4, 1, 4, 2 & 9 \\
19 & 1, 1, 1, 3, 1, 2, 1, 5, 2, 2 & 10 \\
21 & 1, 2, 3, 1, 3, 3, 2, 6 & 8 \\
31 & 1, 2, 1, 1, 1, 4, 2, 2, 1, 1, 5, 3, 3, 1, 1, 2 & 16 \\
\bottomrule
\end{tabularx}
\end{table}

The only values of $N$ which give ``mountain" sequences are $N=1,3,15$.

\end{solution}

\begin{rubricbox}
    1. Correctly formulate what makes a sequence a clock sequence.
    
    2. Correctly explain how clock sequences are generated.
\end{rubricbox}

\begin{scorebox}
    1/2
\end{scorebox}

\begin{evalbox}
    Solution correctly identifies the pattern in the given sequence, but does not give the correct general theory of clock sequence. The solution also does not correctly identify all ``sawtooth" or ``mountain" clock sequences.
\end{evalbox}

\ProblemSection{\green{Problem 11b3}}
\begin{problembox}[Proposed by Tahsin Saffat]
    Explain the connection between the following two observations. Then propose a unifying theory that predicts decompositions for an arbitrary group.
\begin{enumerate}
    \item Let $\mathbb{G}_m$ be the multiplicative group scheme. It has exactly $q-1$ $\mathbb{F}_q$ points.
    \item Let $B\mathbb{G}_m=\mathrm{pt}/\mathbb{G}_m$ be the classifying stack of $\mathbb{G}_m$ bundles. Over the complex numbers it is represented by the homotopy type $\mathbb{CP}^{\infty}$ which admits a decomposition $\mathbb{C}^0+\mathbb{C}^1+\mathbb{C}^2+\dots$. If we instead ``counted $\mathbb{F}_q$ points", we would get the non-convergent series $1+q+q^2+\dots=1/(1-q)$.

\end{enumerate}
\pass
\end{problembox}

\begin{GPT-5-thinking-response}
    Great question. The two bullets are the same phenomenon seen through two “realizations” of one object:

* Over $\mathbb C$, the classifying space $B\mathbb G_m\simeq \mathbb{CP}^\infty$ has one cell in each even real degree, so its Poincaré series is

  $$
  P_{B\mathbb G_m}(t)=1+t^2+t^4+\cdots=\frac{1}{1-t^2}.
  $$
* Over $\mathbb F_q$, if you count *stacky* points (groupoid cardinality), then

  $$
  \#\,B\mathbb G_m(\mathbb F_q)=\frac{1}{\#\,\mathbb G_m(\mathbb F_q)}=\frac{1}{q-1}.
  $$

These are two specializations of the same *motivic* class

$$
[\,B\mathbb G_m\,]=\frac{1}{\mathbb L-1}\qquad(\mathbb L=[\mathbb A^1])
$$

in the Grothendieck ring of stacks. Sending $\mathbb L\mapsto q$ (finite-field point count) gives $1/(q-1)$. Sending $\mathbb L\mapsto t^2$ (topological/Hodge realization) gives $1/(1-t^2)$, i.e. one cell in every even degree. The “dividing by automorphisms” in the stack count exactly inverts the ordinary count $\#\mathbb G_m(\mathbb F_q)=q-1$.

---

\# A unifying theory (works for arbitrary linear algebraic $G$)

Work in the Grothendieck ring of stacks $K_0(\mathrm{Stck}_k)$ with $\mathbb L=[\mathbb A^1]$. For very broad classes of groups (certainly for all connected split reductive $G$, and more generally for affine $G$ in this setting) one has

$$
[\,BG\,]=\frac{1}{[G]}\in K_0(\mathrm{Stck}_k),
$$

and for connected split reductive $G$ of rank $r$ with fundamental degrees $d_1,\dots,d_r$ (i.e. $d_i=m_i+1$ where $m_i$ are the exponents of the Weyl group),

$$
[\,G\,]=\mathbb L^{\dim G}\prod_{i=1}^r\bigl(1-\mathbb L^{-d_i}\bigr)
\quad\Longrightarrow\quad
[\,BG\,]=\mathbb L^{-\dim G}\prod_{i=1}^r\bigl(1-\mathbb L^{-d_i}\bigr)^{-1}. (*)
$$

This single motivic identity predicts *all* the decompositions you’re after via its standard realizations:

* **Arithmetic (finite-field) realization.** Send $\mathbb L\mapsto q$. Then

  $$
  \#\,BG(\mathbb F_q)\;=\;q^{-\dim G}\prod_{i=1}^r\bigl(1-q^{-d_i}\bigr)^{-1}
  \;=\;\frac{1}{\#\,G(\mathbb F_q)}.
  $$

  (This recovers the classical order formula for split reductive groups and its reciprocal for the stack.)

* **Topological (rational cohomology) realization.** Send $\mathbb L\mapsto t^2$. Borel’s theorem gives

  $$
  H^\ast(BG;\mathbb Q)\cong \mathbb Q[x_1,\dots,x_r],\quad \deg x_i=2d_i,
  $$

  hence the Poincaré series

  $$
  P_{BG}(t)=\prod_{i=1}^r\frac{1}{1-t^{2d_i}}
  $$

  which you can read as a “cellular decomposition”: for each $i$ there is a tower of even-dimensional cells in degrees $0,2d_i,4d_i,\dots$; multiplying the towers reflects taking monomials in the $x_i$.

* **Hodge/weight (mixed Hodge / $\ell$-adic) realization.** Send $\mathbb L\mapsto uv$ (or $q$ with Tate twists). One gets

  $$
  E(BG;u,v)=(uv)^{-\dim G}\prod_{i=1}^r\frac{1}{1-(uv)^{d_i}},
  $$

  matching Deligne’s weights and the Weil–Grothendieck trace formula for stacks.

Expanding $(*)$ as geometric series gives a *motivic cell decomposition*

$$
[\,BG\,]=\sum_{k_1,\dots,k_r\ge 0}\mathbb L^{-\bigl(\dim G+\sum_i k_i d_i\bigr)},
$$

i.e. “one virtual affine cell” in every degree $\sum_i k_i d_i$. Over $\mathbb C$ this matches the CW-structure on $BG$ rationally; over $\mathbb F_q$ it predicts the corresponding geometric series in $q^{-d_i}$, whose product is exactly $1/\#G(\mathbb F_q)$ after the global Tate factor $q^{-\dim G}$.

---

\#\# Quick examples

* $G=\mathbb G_m$: $r=1, d_1=1, \dim G=1$.
  $[BG]=\mathbb L^{-1}(1-\mathbb L^{-1})^{-1}=\frac{1}{\mathbb L-1}$.
  Topology: $1/(1-t^2)$. Arithmetic: $1/(q-1)$.

* $G=\mathrm{GL}_n$: $d_i=i$ for $i=1,\dots,n$, $\dim G=n^2$.

  $$
  [B\mathrm{GL}_n]=\mathbb L^{-n^2}\prod_{i=1}^n(1-\mathbb L^{-i})^{-1}
  \quad\leadsto\quad
  \#\,B\mathrm{GL}_n(\mathbb F_q)=\frac{1}{\prod_{i=0}^{n-1}(q^n-q^i)}.
  $$

  Topology: $P(t)=\prod_{i=1}^n(1-t^{2i})^{-1}$.

* $G=\mathrm{SL}_n$: $d_i=2,3,\dots,n$, $\dim G=n^2-1$.
  $[B\mathrm{SL}_n]=\mathbb L^{-(n^2-1)}\prod_{i=2}^n(1-\mathbb L^{-i})^{-1}$.

For disconnected $G$, a finite component group contributes a factor $1/|\pi_0(G)|$ in the arithmetic realization (and similarly in the motivic class).

---

\#\#\# Takeaway

Use the *motivic class of the classifying stack* $[BG]$ as the master object. It simultaneously:

* inverts the point-count of $G$ (stacky counting);
* encodes the topological cell structure of $BG$ (via Poincaré/Hodge series);
* and generalizes uniformly to any linear algebraic group through the fundamental degrees $d_i$.

Your two observations are simply the $G=\mathbb G_m$ case of this general dictionary.

\end{GPT-5-thinking-response}

\begin{solution}
    The correct prediction is that expanding $q^{\mathrm{dim}G}/\#G(\mathbb{F}_q)$ as a power series in $q^{-1}$ gives the cell decomposition of $BG$.
\end{solution}

\begin{rubricbox}
    1. The prediction is correct.
\end{rubricbox}

\begin{scorebox}
    1/1
\end{scorebox}

\begin{evalbox}
    The explanation is a bit clunky but the prediction is correct.
\end{evalbox}

\subsection{Creativity}
\subsubsection{Posing new problems and defining new concepts}

\ProblemSection{\green{Problem 12a1 }}
\begin{problembox}[Proposed by Jiaxin Zhang]
Define massive SLE and describe the new questions and properties it may raise.

\pass    
\end{problembox}
\begin{GPT-5-thinking-response}
 \textbf{Definition (one–curve, chordal, ``massive'' deformation).}
Fix a simply connected domain $D$ with boundary points $a,b$ and a (possibly spatially varying) mass profile $m:D\to[0,\infty)$. 
A \emph{massive SLE$_\kappa$} from $a$ to $b$ is the Loewner chain driven by
\[
\partial_t g_t(z)=\frac{2}{g_t(z)-W_t},\qquad 
dW_t=\sqrt{\kappa}\,dB_t+\kappa\,\partial_x\!\log Z_m\big(g_t(\mathbf z),W_t\big)\,dt,
\]
where $Z_m$ is a positive ``massive partition function'' (massive martingale) built from $m$ and the evolving domain $D_t=D\setminus K_t$ (e.g.\ via massive Green/Poisson objects), and $\mathbf z$ collects any fixed marked interior/boundary points used to define $Z_m$. 
Equivalently, massive SLE$_\kappa$ is obtained from ordinary SLE$_\kappa$ by a Girsanov tilt with the local martingale $M_t:=Z_m(g_t(\mathbf z),W_t)$.

\emph{Canonical examples.}
\begin{itemize}
\item[$\bullet$] \textbf{mSLE$_2$ (off–critical LERW).} Let $P_m^D(z;x)$ be the \emph{massive Poisson kernel}, i.e.\ the boundary density of the solution $u$ of $(\Delta-m^2)u=0$ with boundary mass at $x\in\partial D$. 
For a fixed interior $z$, the observable 
\(
M_t \propto g_t'(z)^{1/2}\,P_m^{D_t}(g_t(z);W_t)
\)
is a local martingale for SLE$_2$, and the corresponding tilt yields
\(
dW_t=\sqrt{2}\,dB_t+2\,\partial_x\log P_m^{D_t}(g_t(z);W_t)\,dt.
\)
This is the scaling limit of LERW with bulk killing rate (finite correlation length).
\item[$\bullet$] \textbf{mSLE$_4$ (massive GFF level lines).} Replacing the GFF by the \emph{massive} GFF (covariance $(-\Delta+m^2)^{-1}$) and tracing suitable level/flow lines gives a drift expressed through \emph{massive harmonic} boundary data, and thus a massive deformation of SLE$_4$.
\end{itemize}

\textbf{Conformal covariance.}
Mass destroys exact conformal invariance but preserves a natural covariance:
if $\phi:D\to D'$ is conformal, then
\[
Z_{m}^{D}(x,\mathbf z)=\prod_{p\in\{x\}\cup\mathbf z}\!|\phi'(p)|^{\Delta_p}\;
Z_{m\star\phi}^{D'}\big(\phi(x),\phi(\mathbf z)\big), 
\qquad m\star\phi:=|\phi'|^{-1}\,m\circ\phi^{-1},
\]
i.e.\ the mass rescales like an inverse length. 
Accordingly, massive SLEs in conformally related domains correspond after a deterministic time change and the mass push–forward $m\star\phi$.

\medskip
\textbf{What new properties and questions arise?}
\begin{itemize}
\item \emph{Crossover \& correlation length.} With $m>0$ there is a scale $\xi\sim m^{-1}$. 
At scales $\ll \xi$, the curve is locally like critical SLE$_\kappa$ (e.g.\ Hausdorff dimension $1+\kappa/8$ for $\kappa\le 8$), while macroscopic events (crossings, arm exponents) exhibit \emph{exponential} decay on scale $\xi$. 
Quantify these crossovers rigorously.

\item \emph{Restriction/locality are broken.} Classical restriction (exact for $\kappa=\tfrac{8}{3}$) and locality ($\kappa=6$) fail once $m>0$. 
What are the sharp ``massive'' replacements (e.g.\ quasi–restriction inequalities, defect martingales)?

\item \emph{Reversibility.} Ordinary SLE$_\kappa$ is reversible for $\kappa\le 4$. 
Does massive SLE$_\kappa$ remain reversible (after the natural change of $m$ when swapping endpoints)?
Identify the precise Radon–Nikodym derivative between forward/backward laws.

\item \emph{Absolute continuity windows.} For fixed $m$, massive SLE is absolutely continuous w.r.t.\ SLE up to hitting times of marked sets. 
Obtain quantitative density bounds in terms of massive Green/Poisson functionals.

\item \emph{Multiple curves \& massive BPZ.} For $n$ curves, the commutation framework persists with a \emph{massive partition function} $Z_m$.
What PDEs replace the BPZ/null–vector equations? 
Heuristically: modify the level–two equations by replacing harmonic objects with massive ones (e.g.\ $(\Delta-m^2)$ appearing in Ward identities), and track how $m$ enters commutators.

\item \emph{Near–critical scaling limits.} Interfaces of near–critical lattice models (FK–Ising, percolation with field, dimer with monomer density, etc.) are expected to converge to massive SLE$_\kappa$ with model–dependent $m$. 
Prove tightness/identification in specific models, beyond the LERW/killing case.

\item \emph{Couplings with massive fields.} Extend SLE/GFF couplings to the massive GFF (for $\kappa=4$) and to massive fermionic/parafermionic fields (for $\kappa=2,3$), and characterize the drift via massive harmonic extensions.

\item \emph{Loop–soup/cluster pictures.} Relate massive SLE to \emph{massive Brownian loop soups} and near–critical CLE; identify intensity/mass relations and percolative thresholds.

\item \emph{Exact formulas?} In $\mathbb{H}$, massive Poisson/Green kernels are explicit (modified Bessel functions). 
Can one obtain closed drifts and solvable cases (e.g. hitting distributions, left–passage probabilities) for mSLE$_\kappa$?
\end{itemize}

\medskip
\textbf{Summary.}
Massive SLE$_\kappa$ is an off–critical, scale–dependent deformation of SLE obtained by tilting SLE with martingales built from \emph{massive} (i.e.\ $(\Delta-m^2)$) objects. 
It preserves the Loewner/Markov mechanism and conformal \emph{covariance} (with $m$ rescaling like an inverse length), but breaks exact conformal invariance and classical restriction/locality, leading to new crossover phenomena and PDE structures.
   
\end{GPT-5-thinking-response}

\begin{solution}
\textbf{Massive SLE$_\kappa$.}  
Given a simply connected domain $D$ with boundary points $a,b$ and a mass profile $m:D\to[0,\infty)$, a \emph{massive SLE$_\kappa$} is defined as the Loewner chain driven by
\[
\partial_t g_t(z)=\frac{2}{g_t(z)-W_t},\qquad 
dW_t=\sqrt{\kappa}\,dB_t+\kappa\,\partial_x\log Z_m(g_t(\mathbf z),W_t)\,dt,
\]
where $Z_m$ is a positive massive partition function built from $m$ and the evolving domain. Equivalently, it arises from SLE$_\kappa$ via a Girsanov tilt with martingale $M_t=Z_m(g_t(\mathbf z),W_t)$.

\medskip
\textbf{Examples.}
\begin{itemize}
\item \emph{mSLE$_2$}: scaling limit of loop-erased random walk with bulk killing, with drift determined by the massive Poisson kernel.  
\item \emph{mSLE$_4$}: level/flow lines of the massive Gaussian free field, with drift given by massive harmonic boundary data.
\end{itemize}

\medskip
\textbf{Conformal covariance.}  
Mass breaks exact conformal invariance but preserves covariance: if $\phi:D\to D'$, then
\[
Z_m^D(x,\mathbf z)=\prod_{p}|\phi'(p)|^{\Delta_p}\, Z_{m\star \phi}^{D'}(\phi(x),\phi(\mathbf z)),
\quad m\star\phi:=|\phi'|^{-1}m\circ\phi^{-1}.
\]

\medskip
\textbf{New properties and open questions.}
\begin{itemize}
\item \emph{Restriction/locality:} classical properties fail; seek massive analogues.  
\item \emph{Reversibility:} open for $\kappa\le4$ under massive deformation.  
\item \emph{Near-critical scaling:} interfaces of near-critical models (FK–Ising, percolation, dimers) conjecturally converge to massive SLE.  

\end{itemize}

\medskip
\textbf{Summary.}  
Massive SLE$_\kappa$ is an off-critical deformation of SLE preserving the Loewner–Markov structure and conformal covariance (with $m$ scaling as inverse length). It introduces scale dependence, breaks exact conformal invariance, and raises new analytic and probabilistic challenges.
\end{solution}

\begin{rubricbox}
Got 1 point for providing one meaningful property and question.  

Full score: 3 points.  
\end{rubricbox}

\begin{scorebox}
3/3
\end{scorebox}

\begin{evalbox}

The questions mentioned in the answer are meaningful,  
such as: (1) restriction and locality, (2) reversibility, and (3) near-critical scaling limits.

\end{evalbox}

\subsubsection{Inventing new methods or approaches}
\ProblemSection{\green{Problem 12b1 }}

\begin{problembox}[Proposed by Jiaxin Zhang]
Find at least two solutions to 1977 IMO Problem 2:
In a finite sequence of real numbers the sum of any seven successive terms is negative and the sum of any eleven successive terms is positive. Determine the maximum number of terms in the sequence.    
\end{problembox}

\begin{GPT-5-thinking-response}
[IMO 1977/2]
In a finite sequence of real numbers the sum of any seven successive terms is negative
and the sum of any eleven successive terms is positive. Determine the maximum number of terms.

The maximum possible length is \(\boxed{16}\).

\begin{proof}[Upper bound I: double-counting of sliding blocks]
Let \(a_1,\dots,a_n\) be the sequence. For \(L\in\{7,11\}\) define
\[
S_k^{(L)} \;=\; \sum_{i=0}^{L-1} a_{k+i}\qquad (1\le k\le n-L+1).
\]
The hypotheses say \(S_k^{(7)}<0\) for all \(k\) and \(S_k^{(11)}>0\) for all \(k\).

Assume \(n\ge 17\). Consider the total sums
\[
\Sigma_7:=\sum_{k=1}^{\,n-6} S_k^{(7)},\qquad
\Sigma_{11}:=\sum_{k=1}^{\,n-10} S_k^{(11)}.
\]
Each \(a_i\) appears in \(\Sigma_L\) with multiplicity
\[
m_L(i)=\#\{k:\ 1\le k\le n-L+1,\ k\le i\le k+L-1\}
=\bigl|\,[i-L+1,i]\cap[1,n-L+1]\,\bigr|.
\]
For \(n=17\) and \(L\in\{7,11\}\) one checks
\[
m_7(i)=m_{11}(i)=\min\{\,i,\;18-i,\;7\,\}\qquad (1\le i\le 17).
\]
Hence \(\Sigma_7=\sum_{i=1}^{17} m_7(i)a_i=\sum_{i=1}^{17} m_{11}(i)a_i=\Sigma_{11}\).
But \(\Sigma_7\) is a sum of \(n-6=11\) negative numbers, so \(\Sigma_7<0\), while
\(\Sigma_{11}\) is a sum of \(n-10=7\) positive numbers, so \(\Sigma_{11}>0\), a contradiction.
Therefore \(n\le 16\).
\end{proof}

\begin{proof}[Upper bound II: prefix-sum/minimum argument]
Let \(A_0:=0\) and \(A_k:=\sum_{i=1}^k a_i\) for \(1\le k\le n\).
The conditions are equivalent to
\[
A_{i+7}-A_i<0\quad(0\le i\le n-7),\qquad
A_{i+11}-A_i>0\quad(0\le i\le n-11).
\]
Assume \(n\ge 17\) and choose \(m\in\{0,1,\dots,17\}\) with \(A_m=\min\{A_0,\dots,A_{17}\}\).
If \(m\le 10\), then \(A_{m+7}-A_m<0\) gives \(A_{m+7}<A_m\), contradicting minimality.
If \(m\ge 11\), then \(A_m-A_{m-11}>0\) gives \(A_m>A_{m-11}\), again a contradiction.
Thus \(n\le 16\).
\end{proof}

Sharpness (\(n=16\) is attainable).
Place large negatives at positions \(3,7,10,14\) and a small positive value elsewhere.
Then every block of \(7\) consecutive indices contains exactly \(2\) of these four positions,
and every block of \(11\) consecutive indices contains exactly \(3\) of them; hence if
\[
a_i=\begin{cases}
-p & \text{for } i\in\{3,7,10,14\},\\
q  & \text{otherwise},
\end{cases}
\qquad\text{with } 5q-2p<0\ \text{ and }\ 8q-3p>0,
\]
all \(7\)-sums are negative and all \(11\)-sums are positive.  Choosing, for instance,
\((q,p)=(5,13)\) or \((q,p)=(7,18)\) yields two explicit solutions:
\[
\begin{aligned}
&\text{Solution A:}\quad (5,\,5,\,-13,\,5,\,5,\,5,\,-13,\,5,\,5,\,-13,\,5,\,5,\,5,\,-13,\,5,\,5),\\
&\text{Solution B:}\quad (7,\,7,\,-18,\,7,\,7,\,7,\,-18,\,7,\,7,\,-18,\,7,\,7,\,7,\,-18,\,7,\,7).
\end{aligned}
\]
In A, each \(7\)-sum equals \(5\cdot5-2\cdot13=-1<0\) and each \(11\)-sum equals \(8\cdot5-3\cdot13=1>0\).
In B, each \(7\)-sum equals \(5\cdot7-2\cdot18=-1<0\) and each \(11\)-sum equals \(8\cdot7-3\cdot18=2>0\).
Therefore the maximum length is \(16\).

\end{GPT-5-thinking-response}

\begin{solution}
    Solution 1: Construction for \(n=16\). Define
\[
a_j = 
\begin{cases}
-13,& j\in\{3,7,10,14\},\\
5,& \text{otherwise}.
\end{cases}
\]
Any block of 7 consecutive terms then has exactly two entries equal to \(-13\) and five entries equal to \(5\), so
$
5\cdot5 - 2\cdot13 = -1 < 0.
$
Any block of 11 consecutive terms has exactly three entries equal to \(-13\) and eight entries equal to \(5\), so
$
8\cdot5 - 3\cdot13 = 1 > 0.
$
Thus the sequence meets the requirements.

Non‑existence for \(n\ge17\). One checks (for example by linear‑programming or averaging arguments) that the system
\[
\sum_{j=i}^{i+6}a_j<0,\quad
\sum_{j=i}^{i+10}a_j>0
\]
for \(1\le i\le n-6\) and \(1\le i\le n-10\) has no solution once \(n\ge17\).

Hence the maximum possible length is \(\boxed{16}\).

Solution 2:
   Let $x_1, x_2, \ldots$ be the given sequence and let $s_n=x_1+x_2+\ldots+x_n$. The conditions from the hypothesis can be now written as $s_{n+7}<s_n$ and $s_{n+11}>s_n$ for all $n \geq 1$. We then have:
$0<s_{11}<s_4<s_{15}<s_8<s_1<s_{12}<s_5<s_{16}<s_9<s_2<s_{13}<s_6<s_{17}<s_{10}<s_3<s_{14}<s_7<0$, a contradiction.
Therefore, the sequence cannot have 17 terms. In order to show that 16 is the answer, just take 16 real numbers satisfying $s_{10}<s_3<s_{14}<s_7<0<s_{11}<s_4<s_{15}<s_8<s_1<s_{12}<s_5<s_{16}<s_9<s_2<s_{13}<s_6$. We have $x_1=s_1$ and $x_n=s_n-s_{n-1}$ for $n \geq 2$. Thus we found all sequences with the given properties.

\end{solution}
\begin{rubricbox}
    Got 2/2 if the response includes
    2 correct solutions
\end{rubricbox}
\begin{scorebox}
    2/2
\end{scorebox}
\begin{evalbox}
    Both solutions are correct.
\end{evalbox}
\subsubsection{Lateral thinking}
\ProblemSection{\red{Problem 12c1 }}

\begin{problembox}[Proposed by Jiaxin Zhang]
Move any one digit to make this equation true:
    $62-63=1$    

    \fail
\end{problembox} 
\begin{GPT-5-thinking-response}
Move the 2 to the right and make it a modulus:

$$62 - 63 \equiv 1 (mod 2)$$

Since $62\equiv 0 \pmod 2$ and $63\equiv 1 \pmod 2$, we have $0-1\equiv 1 \pmod 2$. Boom—true by moving just one digit.

\end{GPT-5-thinking-response}
\begin{solution}
By moving $6$ to the exponent of $2$, we obtain that:
\[
2^6 - 63 = 64 - 63 = 1.
\]
\end{solution}

\begin{rubricbox}
1.  The solution makes use of the equation
$$2^6-63=1$$
\end{rubricbox}
\begin{scorebox}
    0/1
\end{scorebox}
\begin{evalbox}
The mod 2 operation cannot be performed by moving the 2 to the right-hand side; doing so would exclude 62 from the left-hand side. 
\end{evalbox}
\newpage
\end{document}